%% file: main.tex
\documentclass[accepted]{uai2023}

\usepackage[utf8]{inputenc} 
\usepackage[T1]{fontenc}    
\usepackage{hyperref}       
\usepackage{url}            
\usepackage{booktabs}       
\usepackage{amsfonts}       
\usepackage{nicefrac}       
\usepackage{microtype}      
\usepackage{xcolor}         
\usepackage{multirow}
\usepackage{algorithm}
\usepackage{algorithmic}
\usepackage{bbm}
\usepackage{hyperref}
\usepackage{graphicx}
\usepackage{comment}

\usepackage[american]{babel}

\usepackage{natbib} 
\usepackage{mathtools} 
\usepackage{booktabs} 
\usepackage{tikz} 

\usepackage{amsmath,amsthm,amssymb,multirow,paralist,mathrsfs,amsfonts,dsfont}
\newtheorem{theorem}{Theorem}

\newtheorem{proposition}{Proposition}

\newtheorem{lemma}{Lemma}
\newtheorem{corollary}[theorem]{Corollary}
\newtheorem{definition}{Definition}

\newtheorem{remark}{Remark}
\let\oldremark\remark
\renewcommand{\remark}{\oldremark\normalfont}

\def\indicator{\mathrm I}
\def\indicator{\mathds I}
\def\indicator{\mathbb I}

\def\iPR{\text{iPR}}
\def\aPR{\text{aPR}}
\def\rob{\text{rob}}

\def\AR{\text{AR}}
\def\cal{\text{cal}}
\def\tr{\text{tr}}
\def\test{\text{test}}

\def\and{\mathrm{and}}

\def\calX{\mathcal X}
\def\calE{\mathcal E}

\def\calC{\mathcal C}
\def\calY{\mathcal Y}

\def\calP{\mathcal P}

\def\calD{\mathcal D}

\def\E{\mathbb E}
\def\P{\mathbb P}
\def\R{\mathbb R}

\usepackage{newfloat}
\usepackage{listings}

\begin{document}

\newcommand{\swap}[3][-]{#3#1#2} 

\title{Probabilistically Robust Conformal Prediction}

\author[1=]{Subhankar Ghosh}
\author[1=]{Yuanjie Shi}
\author[1]{Taha Belkhouja}
\author[1]{Yan Yan}
\author[1]{Janardhan Rao Doppa}
\author[2]{Brian Jones}
\affil[1]{%
    School of Electrical Engineering and Computer Science\\
    Washington State University
}
\affil[2]{%
    Proofpoint Inc.   
}  

\maketitle

\begin{abstract}
Conformal prediction (CP) is a framework to quantify uncertainty of machine learning classifiers including deep neural networks. Given a testing example and a trained classifier, CP produces a prediction set of candidate labels with a user-specified  coverage (i.e., true class label is contained with high probability). Almost all the existing work on CP assumes clean testing data and there is not much known about the robustness of CP algorithms w.r.t natural/adversarial perturbations to testing examples. This paper studies the problem of probabilistically robust conformal prediction (PRCP) which ensures robustness to most perturbations around clean input examples. PRCP generalizes the standard CP (cannot handle perturbations) and adversarially robust CP (ensures robustness w.r.t worst-case perturbations) to achieve better trade-offs between nominal performance and robustness.  We propose a novel adaptive PRCP (aPRCP) algorithm to achieve probabilistically robust coverage. The key idea 
behind aPRCP is to determine two parallel thresholds, one for data samples and another one for the perturbations on data (aka ``{\em quantile-of-quantile}'' design). We provide theoretical analysis to show that aPRCP algorithm achieves robust coverage. Our experiments on CIFAR-10, CIFAR-100, and ImageNet datasets using deep neural networks demonstrate that aPRCP achieves better trade-offs than state-of-the-art CP and adversarially robust CP algorithms.

\end{abstract}

\section{Introduction}

\begin{figure*}[t]
\centering
\includegraphics[width=.45\linewidth]{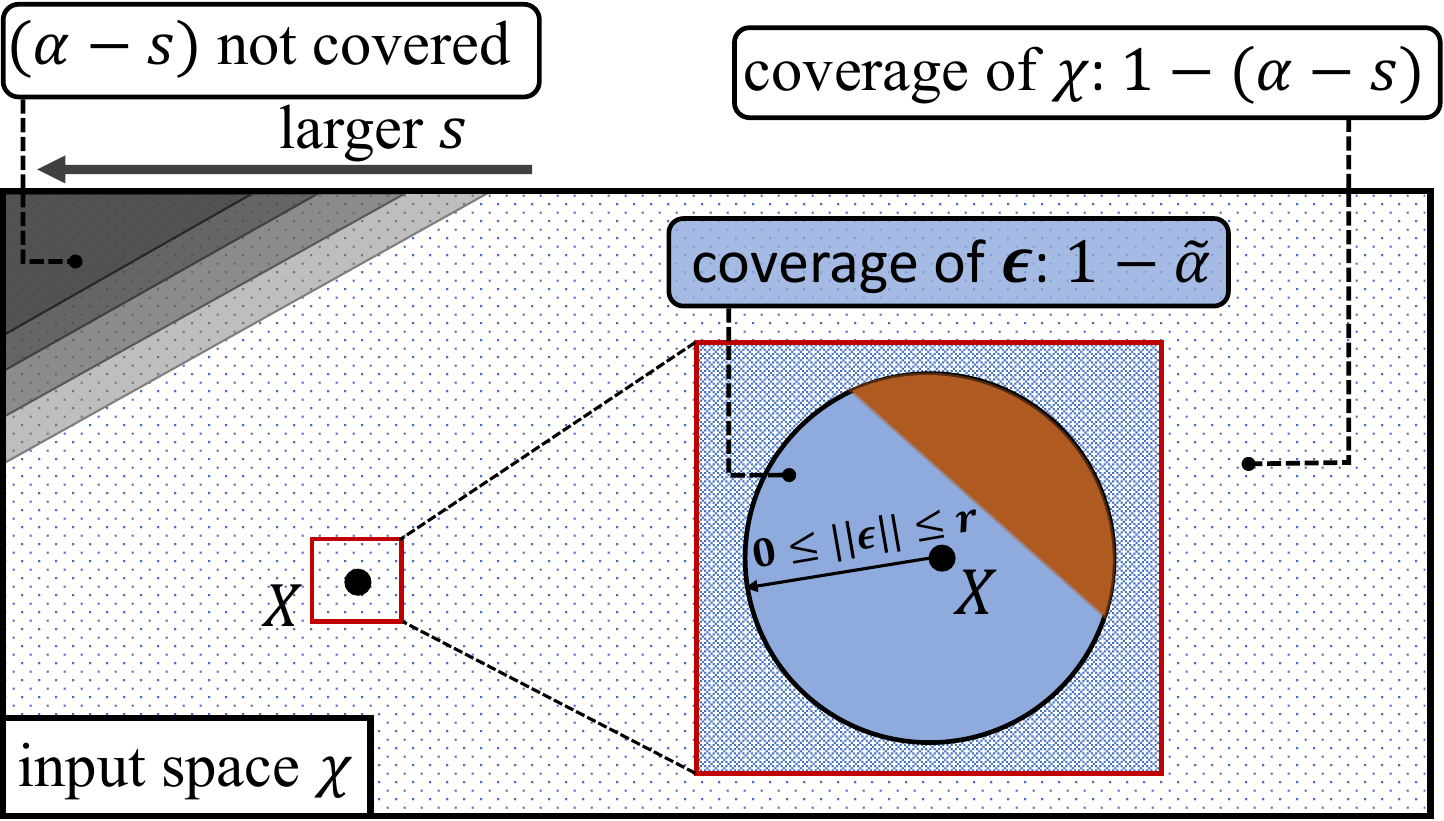}
\caption{Conceptual illustration of the adaptive PRCP setting. 
The goal is to improve the robustness of the CP framework to handle perturbations $\epsilon$ bounded by $r$ for every input $X\in\mathcal{X}$. 
The robust quantile corresponding to 1-$\tilde{\alpha}$ region (blue circle around $X$) is computed by accounting for most of the perturbed data $X+\epsilon$ (see (\ref{eq:robust_quantile_x})). 
$s$ is a conservativeness parameter for the robust quantile that can be varied to achieve the target marginal coverage $1-\alpha + s$ (see (\ref{eq:aPRCP_threshold})). 
Adaptive PRCP can find a trade-off between the marginal coverage on feature space $(X, Y)$ and the robustness for perturbation $\epsilon$ by changing the value of $\tilde \alpha$ and $s$ to achieve probabilistically robust coverage (See Definition \ref{definition:prob_robust_coverage}).
}
\label{fig:illustration}
\end{figure*}

Deep learning has shown significant success in diverse real-world applications.
However, to deploy these deep models in safety-critical applications (e.g, autonomous driving and medical diagnosis), we need uncertainty quantification (UQ) tools to capture the deviation of the prediction from the ground-truth output. For example, producing a subset of candidate labels referred to as {\em prediction set} for classification tasks. Conformal prediction (CP) \citep{vovk1999machine,vovk2005algorithmic,shafer2008tutorial} is a framework for UQ that provides formal guarantees for a user-specified {\em coverage}: ground-truth output is contained in the prediction set with a high probability $1 - \alpha$ (e.g., 90\%).
There are two key steps in CP. First, in the prediction step, we use a black-box classifier (e.g., deep neural network) to compute 
{\it (non-)conformity} scores which measure similarity between calibration examples and a testing input. Second, in the calibration step, we use the conformity scores on a set of calibration examples to find a threshold to construct prediction set which meets the coverage constraint (e.g., $1 - \alpha$=90\%). The {\em efficiency} of CP \citep{sadinle2019least} is measured in terms of size of the prediction set (the smaller the better) which is important for human-ML collaborative systems \citep{rastogi2022unifying}. 

In spite of the recent successes of CP \citep{vovk2005algorithmic}, there is little known about the robustness of CP to adversarial perturbations of clean inputs. Most CP methods \citep{cauchois2020robust,gibbs2021adaptive,tibshirani2019conformal,podkopaev2021distribution,guan2022prediction} are brittle as they assume clean input examples and cannot handle {\em any} perturbations. 
The recent work on adversarially robust CP \citep{gendler2022adversarially} ensures robustness to {\em all} perturbations bounded by a norm ball with radius $r$. 
However, this conservative approach of dealing with {\em worst-case} perturbations can degrade the nominal performance (evaluation on only clean inputs) of the CP method. 
For example, the prediction set size can be large even for clean and easy-to-classify inputs, which increases the burden of human expert in human-ML collaborative systems \citep{cai2019human,rastogi2022unifying}. The main research question of this paper is: {\em how can we develop probably correct CP
algorithms for ensuring robustness to most perturbations
for (pre-trained) deep classifiers?} \footnote{$= $Equal contribution by first two authors}

To answer this question, we present a general notion for probabilistically robust coverage that balances the standard conformal coverage and the adversarial (worst-case) coverage as the fundamental setting.
To address this challenge, we develop the adaptive PRCP algorithm (aPRCP) which is based on the principle of "{\em quantile-of-quantile}" design: consists of two parallel quantiles as illustrated in Figure \ref{fig:illustration}: one defined in the perturbed noise space (see (\ref{eq:robust_quantile_x})), the other one in the data space (\ref{eq:aPRCP_threshold}). 
Our analysis fixes one quantile probability as a given hyper-parameter, and finds the other one to achieve the target probabilistically robust coverage.
We provide theoretical analysis for probabilistic correctness of aPRCP at the population level and the approximation error of empirical quantiles as a function of the number of samples. As a result, aPRCP achieves improved trade-offs between nominal performance (evaluation on clean inputs) and robust performance (evaluation on perturbation inputs) for both probabilistic and worst-case settings as illustrated in Figure \ref{All_ARCP_PRCP_mainPaper1}, which is analogous to the recent work on probabilistically robust learning \cite{robey2022probabilistically}.

\noindent {\bf Contributions.} The key contribution of this paper is the development, theoretical analysis, and empirical evaluation of the aPRCP algorithm. Our specific contributions include:

\begin{itemize}
\item A general notion of probabilistically robust coverage for conformal prediction against perturbations of clean input examples.
\item Development of the adaptive PRCP algorithm based on the principle of "\textit{quantile-of-quantile}" design.
\item Theory to show that aPRCP algorithm achieves probabilistically robust coverage for adversarial examples.
\item Experimental evaluation of aPRCP method on classification benchmarks using deep models to demonstrate its efficacy over prior CP methods on CIFAR-10, CIFAR-100, and ImageNet.
\end{itemize}

\begin{figure*}[!h]
    \centering
    \begin{minipage}{.98\linewidth}
        \begin{minipage}{\linewidth}
            \centering
            \includegraphics[width=0.6\linewidth]{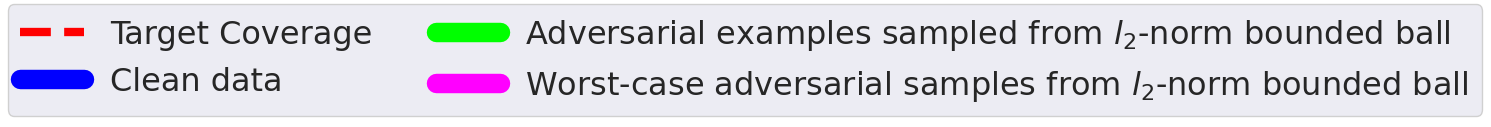}
        \end{minipage}     
        \begin{minipage}{.24\linewidth}
            \centering
            (a)
        \end{minipage}
        \hfill
        \begin{minipage}{.24\linewidth}
            \centering
            (b)
        \end{minipage} 
        \begin{minipage}{.24\linewidth}
            \centering
            (c)
        \end{minipage}
        \hfill
        \begin{minipage}{.24\linewidth}
            \centering
            (d)
        \end{minipage}
        \hfill
        \begin{minipage}{.24\linewidth}
            \includegraphics[width=\linewidth]{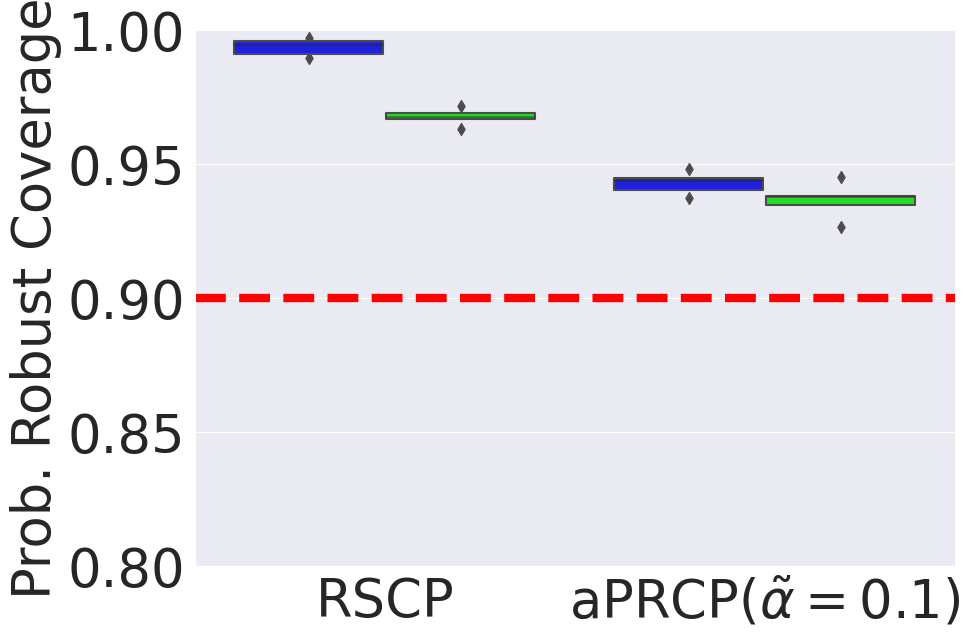}
        \end{minipage}
        \hfill
        \begin{minipage}{.24\linewidth}
            \centering
            \includegraphics[width=\linewidth]{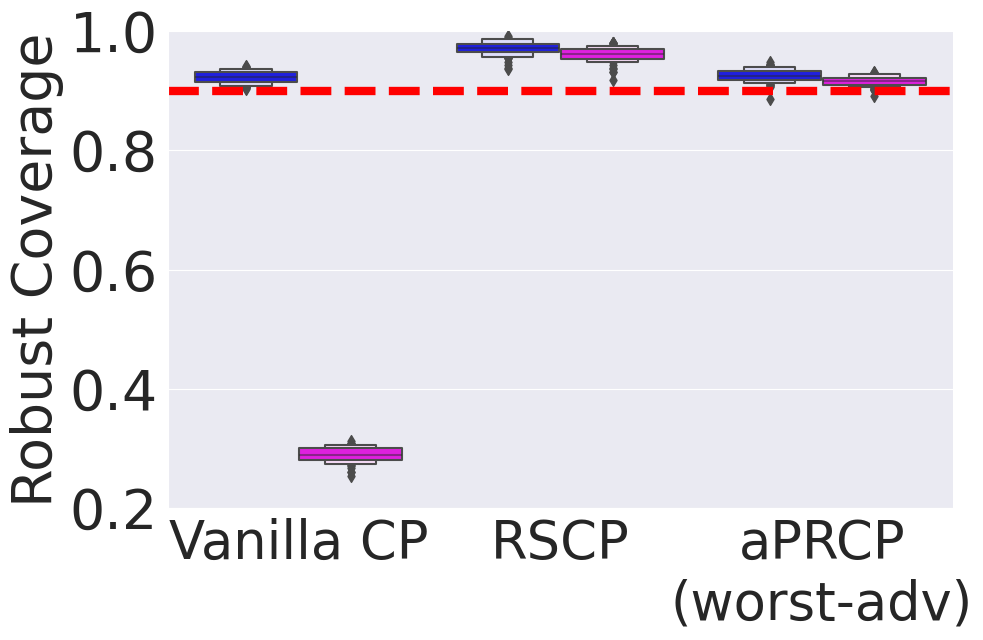}
        \end{minipage}    
        \hfill
        \begin{minipage}{.24\linewidth}
            \includegraphics[width=\linewidth]{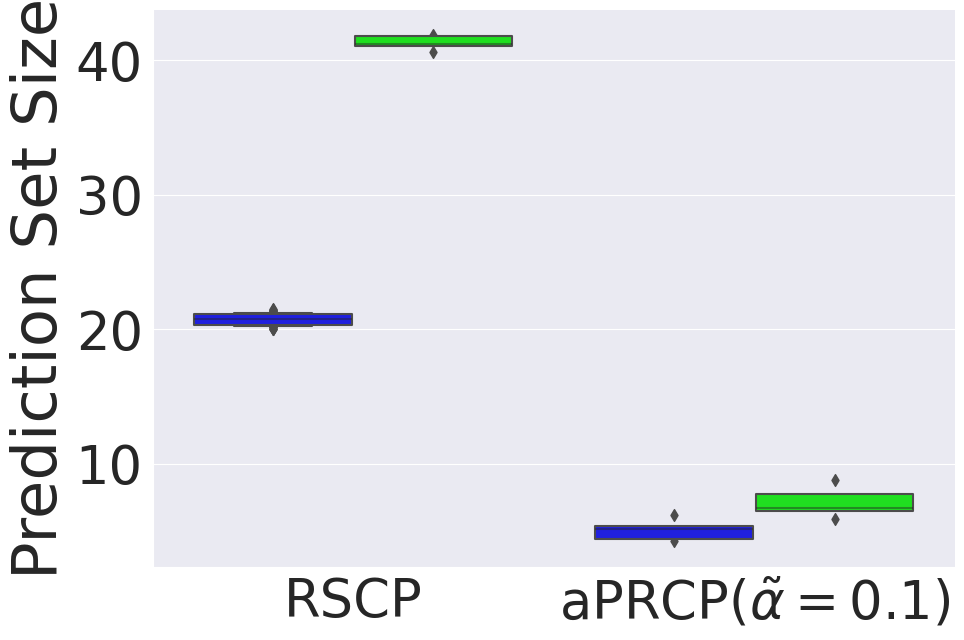}
        \end{minipage}
        \hfill
        \begin{minipage}{.24\linewidth}
            \centering
            \includegraphics[width=\linewidth]{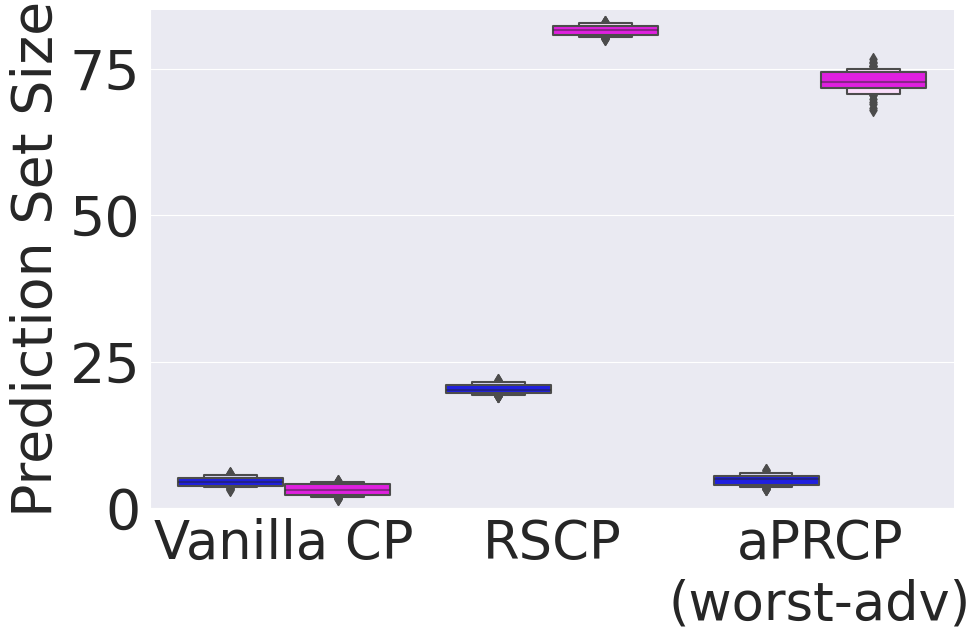}
        \end{minipage}    
    \end{minipage}
    \caption{Results on CIFAR100 dataset using a ResNet model to illustrate the trade-offs between nominal performance (evaluation on clean data) and robust performance (evaluation on adversarial examples) for Vanilla CP, RSCP, and variants of the aPRCP algorithm. (a) and (c) show the evaluation against clean examples and their corresponding noisy samples (i.e., $\tilde{X} = X + \epsilon; ||\epsilon||_2 \leq r$) w.r.t probabilistic robustness. (b) and (d) show the evaluation against clean examples and their corresponding bounded adversarial examples. aPRCP(worst-adv) is the variant of aPRCP that works for worst adversarial data. Vanilla CP fails to achieve coverage for worst-case adversarial data. RSCP achieves a robust coverage much higher than the target (nominal) coverage, resulting in large prediction sets. aPRCP achieves better results (tighter coverage and smaller prediction set size) than vanilla CP and RSCP in terms of the joint performance on clean, noisy, and worst-adversarial data.}
    \label{All_ARCP_PRCP_mainPaper1}
\end{figure*}

\section{Background and Problem Setup}

We consider the problem of uncertainty quantification (UQ) of pre-trained deep models for classification tasks in the presence of adversarial perturbations. Suppose $(X, Y)$ is a data sample where $X$ is an input from the space $\mathcal{X}$ and $Y \in \mathcal{Y}$ is the corresponding ground-truth output. For classification tasks, $\mathcal{Y}$ is a set of $C$ discrete class-labels $\{1, 2, \cdots, C\}$. 
Let $\epsilon$ denote the $l_2$-norm bounded noise, i,e,. $\calE_r = \{ \epsilon \in \calX : \| \epsilon \|_2 \leq r \}$ that is independent from data sample $(X, Y)$.
Let $\calP_{X, Y}$ and $\calP_\epsilon$  denote the underlying distribution of $(X, Y)$ and $\epsilon$, respectively. We also define $Z = (X, Y, \epsilon)$ as the joint random variable and the perturbed input example $\widetilde X = X + \epsilon$ for notational simplicity.

\noindent {\bf Uncertainty Quantification.} Let $\calD_\tr$ and $\calD_\cal$ correspond to sets of training and calibration examples drawn from a target distribution $\calP_{X, Y}$.
We assume the availability of a pre-trained deep model $F_{\theta}: \mathcal{X} \mapsto \mathcal{Y}$, where $\theta$ stands for the parameters of the deep model. For a given testing input $\widetilde X$, we want to compute UQ of the deep model $F_{\theta}$ in the form of a prediction set $\mathcal{C}(\widetilde X)$, a subset of candidate class-labels $\{1, 2, \cdots, C\}$. The performance of UQ for clean data samples (i.e.,  $\epsilon$=0) is measured using two metrics. First, the (marginal) {\em coverage} is defined as the probability that the ground-truth output $Y$ is contained in $\mathcal{C}(X)$ for a testing example $(X, Y)$ from the same data distribution $\calP_{X, Y}$, i.e., $\mathbb{P}(Y \in \mathcal{C}(X))$. The empirical coverage \texttt{Cov} is measured over a given set of testing examples $\calD_\test$. Second, {\em efficiency}, denoted by \texttt{Eff}, measures the cardinality of the prediction set $\mathcal{C}(X)$. Smaller prediction set means higher efficiency. It is easy to achieve the desired coverage (say 90\%) by always outputting $\mathcal{C}(X)$=$\mathcal{Y}$ at the expense of poor efficiency.

\noindent {\bf Conformal Prediction (CP).} 
CP is a framework that allows us to compute UQ for any given predictor through a conformalization step. 
The key element of CP is a score function $S$ that computes the {\em conformity} (or {\em non-conformity}) score, measures similarity between labeled examples, which is used to compare a given testing input to the calibration set $\calD_\cal$.  
Since any non-conformity score can be intuitively converted to a conformity measure \citep{vovk2005algorithmic}, we use non-conformity measure for ease of technical exposition. Let $S(X, Y)$ denote the non-conformity score function of data sample $(X, Y)$. For a sample $(X_i, Y_i)$ from the calibration set $\calD_\cal$, we use $S_i = S(X_i, Y_i)$ as a shorthand notation of its non-conformity score.

A typical method based on split conformal prediction has a threshold $\tau$ to compute UQ in the form of prediction set for a given testing input $X$ and deep model $F_{\theta}$. A small set of calibration examples $\calD_\cal$ are used to select the threshold $t$ for achieving the given coverage $1-\alpha$ (say 90\%) empirically on  $\calD_\cal$. 
Let $Q(\alpha) := \min \{ t : \P_{X, Y} \{ S(X, Y) \leq t \} \geq 1 - \alpha \}$ be the true quantile of the conformity score for $(X, Y)$.
Let $\calD_\cal = \{(X_i, Y_i)\}_{i=1}^n$ denote a calibration set with $n$ exchangeably drawn random samples from the underlying distribution $\calP_{X, Y}$.
We denote the ($1-\alpha$)-quantile derived from $\{S_i\}_{i=1}^n$ by $Q(\alpha; \{S_i\}_{i=1}^n)$ = $S_{(\lceil (1-\alpha) (n+1) \rceil)}$. The prediction set for a new testing input $X$ is given by $\mathcal{C}(X)$=$\{y: S(X, y) \le \tau \}$ using a threshold $\tau$. 
CP provides valid guarantees that $\mathcal{C}(X)$ has coverage $1-\alpha$ on future examples drawn from the same distribution $\calP_{X, Y}$.

For classification, several non-conformity scores can be employed. The homogeneous prediction sets (HPS) score is defined \citep{vovk2005algorithmic, lei2013distribution} as follows:
\begin{equation}
\label{HPS_eq}
    S^{\text{HPS}}(X, y ) = 1 - F_{\theta}(X)_y,
\end{equation}
where $F_{\theta}(X)_y \in [0, 1]$ is the probability corresponding to the true class $y$ using the deep model $F_{\theta}$. Recent work has proposed the adaptive prediction sets (APS) \citep{romano2020classification} score that is based on ordered probabilities. The score function of APS is defined as follows:
\begin{align}
\label{APS_eq}
    S^{\text{APS}}(X, y ) = &\sum_{y^{'} \in \mathcal{Y}} F_{\theta}(X)_{y^{'}}\mathds{1}\left\{ F_{\theta}(X)_{y^{'}} >  F_{\theta}(X)_y\right\}
\nonumber\\
&+ u. F_{\theta}(X)_y,
\end{align}
where $u$ is a random variable uniformly distributed over $[0, 1]$ and $\mathds{1}$ is the indicator function.


\noindent {\bf Problem Definition.} The high-level goal of this paper is to study methods to improve the robustness of the standard CP framework to adversarial/noisy examples of the form $\widetilde X = X + \epsilon$, where $\epsilon$ is the additive perturbation from $\calE_r = \{ \epsilon \in \R^d : \| \epsilon \|_p \leq r\}$. 
Specifically, we propose a novel adaptive probabilistically robust conformal prediction (aPRCP) algorithm which accounts for $(1-\tilde \alpha)$ (see $\tilde \alpha$ for robust quantile in (\ref{eq:robust_quantile_x})) fraction of perturbations in $\calE_r$ for each data $(X, Y)$.
Setting $\tilde \alpha = 0$ as an extreme case makes aPRCP handle all perturbations (i.e., worst-case), similar to RSCP \citep{gendler2022adversarially}.
We theoretically and empirically analyze aPRCP to demonstrate improved trade-offs between nominal performance (evaluation on clean inputs) and robust performance (evaluation on perturbation inputs). 
Figure \ref{fig:illustration} conceptually illustrates the PRCP problem setting.

\section{Robust Conformal Prediction}
\label{section:RCP}

This section describes our proposed adaptive probabilistically robust conformal prediction (aPRCP) algorithm. First, we introduce the notion of adversarially robust coverage and extend it to  probabilistically robust coverage. Next, we motivate the significance of aPRCP algorithm and study the theoretical connection between aPRCP and adversarially robust CP setting \citep{gendler2022adversarially} in terms of probabilistically robust coverage and prediction set size. Finally, we analyze the gap between empirical and population level quantiles in terms of the number of data samples.

\subsection{Probabilistically Robust Coverage}

This section introduces the expanded notation of inflation condition on the conformity scoring function from the worst-case adversarial robustness setting to the more general probabilistic robustness setting. We start with the following definitions that are originally introduced for the ARCP setting \citep{gendler2022adversarially} and capture the inflation property of the score function for deriving adversarial robustness.

\begin{definition}
\label{definition:adv_robust_coverage}
(Adversarially robust coverage)
A prediction set $\calC(\widetilde X)$ provides ($1-\alpha$)-adversarially robust coverage if for a desired coverage probability $1-\alpha \in (0,1)$:
\begin{align}
\label{eq:adv_robust_coverage}
\P_{X, Y} \{ Y \in \calC(\widetilde X = X + \epsilon), \forall \epsilon \in \calE_r \} \geq 1 - \alpha .
\end{align}
\end{definition}

\begin{definition}
\label{definition:adv_inflated_score}
($M_r$-adversarially inflated score function)
$S: \calX \times \calY \rightarrow \R$ is an $M_r$-adversarially inflated score function if the following inequality holds:
\begin{align}
\label{eq:adv_inflated_score}
&
S(X + \epsilon, Y) 
\leq 
S(X, Y) + M_r,
\nonumber\\
&
\qquad \qquad \qquad \qquad  
\forall X \in \calX, Y \in \calY \text{ and } \epsilon \in \calE_r .
\end{align}
\end{definition}

The strategy of RSCP algorithm \citep{gendler2022adversarially} for the ARCP setting is to directly add an inflated quantity $M_r$ to the quantile determined from the clean data $(X, Y)$, 
\begin{align}
\label{eq:ar_threshold}
\tau^\AR(\alpha) := Q(\alpha) + M_r,
\end{align}
and construct a prediction set with $\calC^\AR(X) = \{ y \in \calY : S(X + \epsilon, y) \leq \tau^\AR(\alpha) \}$.
To this end, since $Q(\alpha)$ provides $(1-\alpha)$ marginal coverage on clean data $(X, Y)$, $\tau^\AR(\alpha)$ thus guarantees $(1-\alpha)$-adversarially robust coverage on adversarial data $(X + \epsilon, Y)$.

This result is summarized in the following proposition.
\begin{proposition}
\label{proposition:AR_coverage_ARCP}
(Adversarially robust coverage of RSCP, Theorem 1 in \citep{gendler2022adversarially})
Assume the score function $S$ is $M_r$-adversarially inflated.
Let $\calC^\AR(\widetilde X) = \{ y \in \calY : S(\widetilde X, y) \leq \tau^\AR(\alpha) \}$ be the prediction set for a testing sample $\widetilde X$.
Then RSCP achieves ($1-\alpha$)-adversarially robust coverage.
\end{proposition}

Now we extend the notion of adversarially robust coverage to the more general and relaxed condition, i.e., probabilistically robust coverage, by introducing the definition below.
\begin{definition}
\label{definition:prob_robust_coverage}
(Probabilistically robust coverage)
A prediction set $\calC(\widetilde X)$ provides ($1-\alpha$)-probabilistically robust coverage if for a desired coverage probability $1-\alpha \in (0, 1)$:
\begin{align}
\label{eq:prob_robust_coverage}
\P_{X, Y, \epsilon} \{ Y \in \calC(\widetilde X = X + \epsilon) \} \geq 1 - \alpha .
\end{align}

\end{definition}

We highlight that the key difference between adversarially robust coverage (Definition \ref{definition:adv_robust_coverage}) and probabilistically robust coverage (Definition \ref{definition:prob_robust_coverage}) is whether the distribution of the perturbation $\epsilon$ is involved in the comparison with the target probability $1-\alpha$:
probabilistically robust coverage goes though the joint distribution involving $\epsilon$, i.e., $\P_{X, Y, \epsilon}\{ \cdot \}$ in (\ref{eq:prob_robust_coverage}) instead of $\P_{X, Y}\{ \cdot, \forall \epsilon \in \calE_r \}$ in (\ref{eq:adv_robust_coverage}).
Based on this understanding, we can see that a conformal prediction method can achieve ($1-\alpha$)-probabilistically robust coverage if it can satisfy ($1-\alpha$)-adversarially robust coverage.
For the same target probability $(1-\alpha)$, adversarially robust coverage is more difficult to achieve than probabilistically robust coverage. Hence, the notion of probabilistic robustness for CP is more general and relaxed.

Naturally, we now extend the definition of the uniform inflated score function (Definition \ref{definition:adv_inflated_score}) to the following one.
\begin{definition}
\label{definition:prob_inflated_score}
($M_{r, \eta}$-probabilistically inflated score function)
$S : \calX \times \calY \rightarrow \R$ is an $M_{r, \eta}$-probabilistically inflated score function if the following inequality holds for $\eta \in [0, \alpha]$:
\begin{align}
\label{eq:prob_inflated_score}
\P_Z \big\{ 
S(X + \epsilon, Y) 
\leq 
S(X, Y) + M_{r, \eta}
\big\} 
\geq 
1 - \eta .
\end{align}
\end{definition}

The above definition regarding the inflation of the score function is general and includes (\ref{eq:adv_inflated_score}) given in Definition \ref{definition:adv_inflated_score} as a special case:
By simply setting $\eta = 0$, we get $\P_Z\{ S(X + \epsilon, Y) \leq S(X, Y) + M_{r, 0} \} \geq 1$, i.e., $M_{r, 0} = M_r$.

Again, we highlight that the above condition involves the joint distribution on $Z$, as in Definition \ref{definition:prob_robust_coverage}.

Based on the extension from adversarial to probabilistic robustness setting, it is easy to develop a similar principle on the {\it inflated} score function to derive probabilistically robust coverage, which we refer to as inflated probabilistically robust conformal prediction (\texttt{iPRCP}).
To this end, let
$$\tau^\iPR(\alpha; \eta) := Q(\alpha^*_\iPR) + M_{r, \eta},$$ 
where $\alpha^*_\iPR = 1 - ( 1 - \alpha ) / ( 1 - \eta )$.
$\tau^\iPR(\alpha; \eta)$ is the threshold determined by iPRCP that treats $\eta$ from probabilistically inflated score function as a hyper-parameter.
We use $\alpha^*_\iPR$ as the probability for deriving the quantile on clean data, as (\ref{eq:ar_threshold}) in ARCP.

\begin{proposition}
\label{proposition:PR_coverage_iPRCP}
(Probabilistically robust coverage of iPRCP)
Assume the score function $S$ is an $M_{r, \eta}$-probabilistically inflated.
Let $\calC^\iPR(\widetilde X) = \{ y \in \calY : S(\widetilde X, y) \leq \tau^\iPR(\alpha; \eta) \}$ be the prediction set for a testing sample $\widetilde X=X+\epsilon$. 
Then iPRCP achieves ($1-\alpha$)-probabilistically robust coverage.
\end{proposition}

This result shows that we can guarantee the ($1-\alpha$)-probabilistically robust coverage if we use $\tau^\iPR(\alpha; \eta)$ to construct the prediction set $\calC^\iPR$.
While the idea is simple and follows the inflation quantile used in the ARCP setting, it implies that we {\it have to know $M_{r, \eta}$}, the inflated quantity on the clean quantile. This requires us to know the score function very well. Otherwise, we have to design a score function that satisfies the desired condition,
similar to how the randomly smoothed score function was designed by RSCP algorithm to work for the ARCP setting \citep{gendler2022adversarially}.
It was carefully designed to offer a uniform Lipschitz continuity with the requirement of an additional set of Gaussian random samples. This design may introduce additional restrictions, since extra samples are required every time the score function is applied, including each calibration and testing sample. Therefore, we would like to address the following question: {\em Can we design an adaptive algorithm to fit the underlying distribution without any prior knowledge or special design of the score function?}

\begin{algorithm}[t]

    \caption{adaptive PRCP (\texttt{aPRCP}) } 
    \label{alg:aPRCP}
    \begin{algorithmic}[1]
    
    \STATE \textbf{Input}: target probability $\alpha \in (0, 1)$; the hyper-parameter $s$; set $\tilde \alpha = 1- \frac{1-\alpha}{1-\alpha+s}$;
    split data into disjoint training set $\calD_\tr$ and calibration set $\calD_\cal$ with $|\calD_\cal| = n$.
    
    \STATE Train a classifier $F_\theta$ on $\calD_\tr$.
    
    \STATE Draw $\epsilon_{ij} \sim \calP_\epsilon$ where $i \in \{1,\cdots, n\}$ and $j \in \{1,\cdots, m\}$ denote the indices of data $(X_i, Y_i)$ and its $m$ perturbations.
    
    \STATE Compute scores: $S_{ij} = S(X_i + \epsilon_{ij}, Y_i)$, $\forall i, j$.
    
    \STATE Compute empirical robust quantiles: \\
    $\widehat Q^\rob_i = \widehat Q^\rob(X_i, Y_i; \tilde \alpha) = Q(\tilde \alpha, \{S_{ij}\}_{j=1}^m)$ via (\ref{eq:robust_quantile_x}), $\forall i$.
    \label{algorithm:line:empirical_robust_quantile}

    \STATE Determine threshold $\tau^\aPR(\alpha; s) = \widehat Q^\rob_{ ( \lceil (n+1) ( 1 - \alpha + s ) \rceil ) }$ from empirical robust quantiles according to (\ref{eq:aPRCP_threshold}).
    \label{algorithm:line:threshold_aPRCP}
    
    \STATE Receive $\widetilde{X}_{n+1}$ and construct prediction set:\\
    ~~~ $\calC(\widetilde{X}_{n+1}) = \{ y \in \mathcal{Y}: S(\widetilde{X}_{n+1}, y) \leq \tau^\aPR(\alpha; s) \}$.
    \end{algorithmic}
\end{algorithm}

\subsection{Adaptive PRCP Algorithm}
\label{subsection:aPRCP}

This section presents our adaptive algorithm for achieving probabilistically robust coverage (aPRCP).
We summarize it in Algorithm \ref{alg:aPRCP} and elaborate it below.
First, we define the $(1 - \tilde \alpha)$-{\it robust quantile} for a given $X$ as follows
\begin{align}
\label{eq:robust_quantile_x}
&
Q^\rob(X, Y; \tilde \alpha)
\nonumber\\
& \qquad \quad
:=
\min\{ t : \P_\epsilon \{ S(\widetilde X, Y) \leq t \} \geq 1 - \tilde \alpha \}.
\end{align}
Given $(X, Y$) and $\tilde \alpha$, $Q^\rob(X, Y; \tilde \alpha)$ returns the quantile from all randomly perturbed $\widetilde X=X + \epsilon$ over $\epsilon \in \calE_r$.
It acquires the inflated quantity from a local region of $X$ as $\tilde \alpha$ indicates how conservative this inflation can be.
We denote the empirical robust quantile (in Line \ref{algorithm:line:empirical_robust_quantile} of Algorithm \ref{alg:aPRCP}) by $\widehat Q^\rob$.

Next, we define the threshold of the proposed adaptive PRCP (\texttt{aPRCP}) for a hyper-parameter $s \in [0, \alpha]$ as follows.
\begin{align}
\label{eq:aPRCP_threshold}
&
\tau^\aPR(\alpha; s)
= 
\min\{ t : 
\nonumber\\
& \qquad 
\P_{X, Y} \{ Q^\rob(X, Y; \alpha^*_\aPR) \leq t \} \geq 1 - \alpha + s \} ,
\end{align}
where
$\alpha^*_\aPR
=
1 - ( 1 - \alpha ) / ( 1 - \alpha + s )$ is a conservativeness parameter for the robust quantile in (\ref{eq:robust_quantile_x}) that depends on the target probability $\alpha$ and the hyper-parameter $s$.
In practice, the empirical threshold $\widehat \tau^\aPR = \widehat Q^\rob_{( \lceil (n+1) ( 1 - \alpha + s ) \rceil ) }$ is selected from empirical robust quantiles $\{\widehat Q^\rob_i\}_{i=1}^n$ (in Line \ref{algorithm:line:threshold_aPRCP} of Algorithm \ref{alg:aPRCP}).
Our aPRCP algorithm is adaptive since it finds $\alpha^*_\aPR$ that is adaptive to the underlying distribution of $(X, Y)$ as long as $\alpha$ and $s$ are fixed apriori.
The following formal result guarantees the probabilistically robust coverage for the aPRCP algorithm.
\begin{theorem}
\label{theorem:prob_robust_coverage_aPRCP}
(Probabilistically robust coverage of aPRCP)
Let $\calC^\aPR(\widetilde X = X + \epsilon) = \{ y \in \calY : S(\widetilde X, y) \leq \tau^\aPR(\alpha; s) \}$ be the prediction set for a testing sample $\widetilde X$.
Then aPRCP achieves ($1-\alpha$)-probabilistically robust coverage.

\end{theorem}

\begin{remark}
In fact, $\tau^\aPR(\alpha; s)$ is the ($1-\alpha+s$)-th quantile (going through $(X, Y)$) of the ($1-\alpha^*_\aPR$)-robust quantiles (going through $\epsilon$).
One benefit of aPRCP is the transfer of the inflation from the score function to the specified probability (i.e., an $s$ increase in probability). Therefore, it is not required to have a prior knowledge of either $M_r$ as in ARCP or  $M_{r, \eta}$ as in iPRCP.
Instead, aPRCP requires finding a feasible and a good value for $\alpha^*_\aPR$ by treating $s$ as a hyper-parameter, though it inflates the specified probability, i.e., $1-\alpha+s \geq 1 - \alpha$, and $1-\alpha^*_\aPR \geq 1 - \alpha$.
\end{remark}

\begin{theorem}
\label{theorem:prob_robust_coverage_aPRCP_cross_domain_noise}
(Probabilistically robust coverage of aPRCP for cross-domain noise)
Let $\calP_\epsilon^{test}$ and $\calP_\epsilon^{cal}$ denote different distributions of $\epsilon$ during the testing and calibration phases, respectively.
Assume $\P_{\epsilon \sim \calP_\epsilon^{cal}}\{\epsilon\} - \P_{\epsilon \sim \calP_\epsilon^{test}}\{\epsilon\} \leq d$ for all $\| \epsilon \| \leq r$.
Set $\alpha^*_\aPR = 1 - d - ( 1 - \alpha) / (1 - \alpha + s )$ in (\ref{eq:aPRCP_threshold}).
Let $\calC^\aPR(\widetilde X = X + \epsilon) = \{ y \in \calY : S(\widetilde X, y) \leq \tau^\aPR(\alpha; s) \}$ be the prediction set for a testing sample $\widetilde X$.
Then aPRCP achieves ($1-\alpha$)-probabilistically robust coverage.
\end{theorem}

\begin{remark}
The key assumption we make is $\P_{\epsilon \sim \calP_\epsilon^{cal}}\{\epsilon\} - \P_{\epsilon \sim \calP_\epsilon^{test}}\{\epsilon\} \leq d$, which is analogous to $L^1$-distance used in the domain adaptation literature \citep{redko2020survey,ben2006analysis}.
One can interpret it as the maximal gap of the density probability between the calibration and testing distributions when fixing $\epsilon$.
As per our analysis, when this gap can be bounded by a sufficiently small constant $d$, with an inflated nominated coverage in the robust quantile (i.e., setting $\alpha^*_\aPR = 1 - d - (1-\alpha)/(1-\alpha+s)$ in (\ref{eq:aPRCP_threshold})), we can guarantee probabilistically robust coverage for aPRCP.
\end{remark}

\subsection{Connection Between ARCP and PRCP }
\label{subsection:connection}

Although ARCP algorithm can achieve adversarially robust coverage, we can still connect ARCP and PRCP in the sense of {\it probabilistically robust coverage} and understand their performance in terms of {\it efficiency}. Recall that efficiency of conformal prediction algorithms refers to the measured size of prediction sets for testing samples when some desired coverage is achieved.
For example, for the same target probability $1-\alpha$, a smaller threshold indicates better efficiency. The following result shows the possibly improved efficiency of iPRCP and aPRCP when compared to ARCP after that their hyper-parameters were tuned properly (i.e., $\eta$ for iPRCP and $s$ for aPRCP).

\begin{corollary}
\label{corollary:compare_aPRCP_ARCP}
To achieve the same ($1-\alpha$)-probabilistically robust coverage on $Z$, the following inequalities hold: \begin{align*}
\min_{ \eta \in [0, \alpha] } \tau^\iPR(\alpha; \eta) \leq \tau^\AR(\alpha), ~~
\min_{ s \in [0, \alpha] }  \tau^\aPR(\alpha; s) \leq \tau^\AR(\alpha) .
\end{align*}
\end{corollary}
When all three algorithms achieve ($1-\alpha$)-probabilistically robust coverage,
smaller thresholds yield better efficiency, i.e., iPRCP and aPRCP.
The idea of the above result is to particularly set $\eta=0$ and $s=0$, which makes iPRCP and aPRCP degenerate to ARCP, resulting in the same threshold.
For aPRCP with $s=0$, we have $\alpha^*_\aPR=0$, i.e., $1$-robust quantile 
for each $(X, Y)$ used, which recovers ARCP.

\subsection{Approximation Error of Empirical Quantiles}

In the above sections, we presented algorithms and their analysis directly in the population sense, including the true quantile $Q(\alpha)$ and $Q^\rob(X; \alpha)$. However, when executing a given conformal prediction method on exchangeable samples $\calD_\cal$, we employ empirical quantiles in practice. To close this gap between theory and practice, we additionally discuss the concentration inequalities for empirical approximation to these quantities (i.e., the gap between empirical and true quantiles) as a function of the number of samples.

\begin{proposition}
\label{proposition:empirical_quantile_concentration}
(Concentration inequality for quantiles)
Let $Q(\alpha) = \max\{ t : \P_V\{ V \leq t \} \geq 1 - \alpha \}$ be the true quantile of a random variable $V$ given $\alpha$,
and $\widehat Q_n(\alpha) = V_{ ( \lceil (n+1) ( 1 - \alpha ) \rceil ) }$ be the empirical quantile estimated by $n$ randomly sampled set $\{V_1, ..., V_n\}_{i=1}^n$.
Then with probability at least $1-\delta$, we have
$
\widehat Q_n(\alpha + \tilde O(1/\sqrt{n}))
\leq
Q(\alpha)
\leq
\widehat Q_n(\alpha - \tilde O(1/\sqrt{n}))
$   
where $\tilde O$ hides the logarithmic factor.
\end{proposition}

The above result shows that more data samples from the underlying distribution for $(X, Y)$ or $\epsilon$ will help in improving the approximation of empirical quantiles on score function $S$ at a rate of $\tilde O(1/\sqrt{n})$, where $n$ is number of samples.
Note that we only use this proposition to fill the gap between empirical and true quantiles.
Some prior work also studied similar concentration results \citep{vovk2012conditional}.

\section{Experiments and Results}

In this section, we present the empirical evaluation of our proposed aPRCP algorithm along different dimensions.

\begin{figure*}[!h]
    \centering
    \begin{minipage}{.92\linewidth}
        \begin{minipage}{\linewidth}
            \centering
            \includegraphics[width=.35\linewidth]{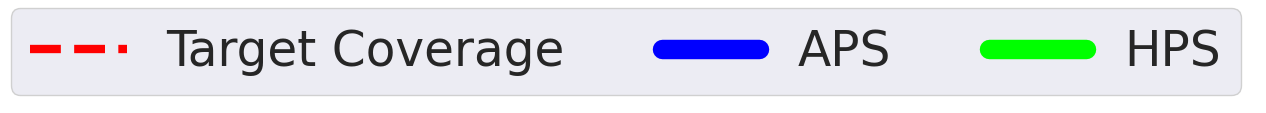}
        \end{minipage}     
        \begin{minipage}{.33\linewidth}
            \centering
            (a) CIFAR10
        \end{minipage}
        \hfill
        \begin{minipage}{.33\linewidth}
            \centering
            (b) CIFAR100
        \end{minipage} 
        \begin{minipage}{.33\linewidth}
            \centering
            (a) ImageNet
        \end{minipage}
        \hfill
        \begin{minipage}{.33\linewidth}
            \includegraphics[width=\linewidth]{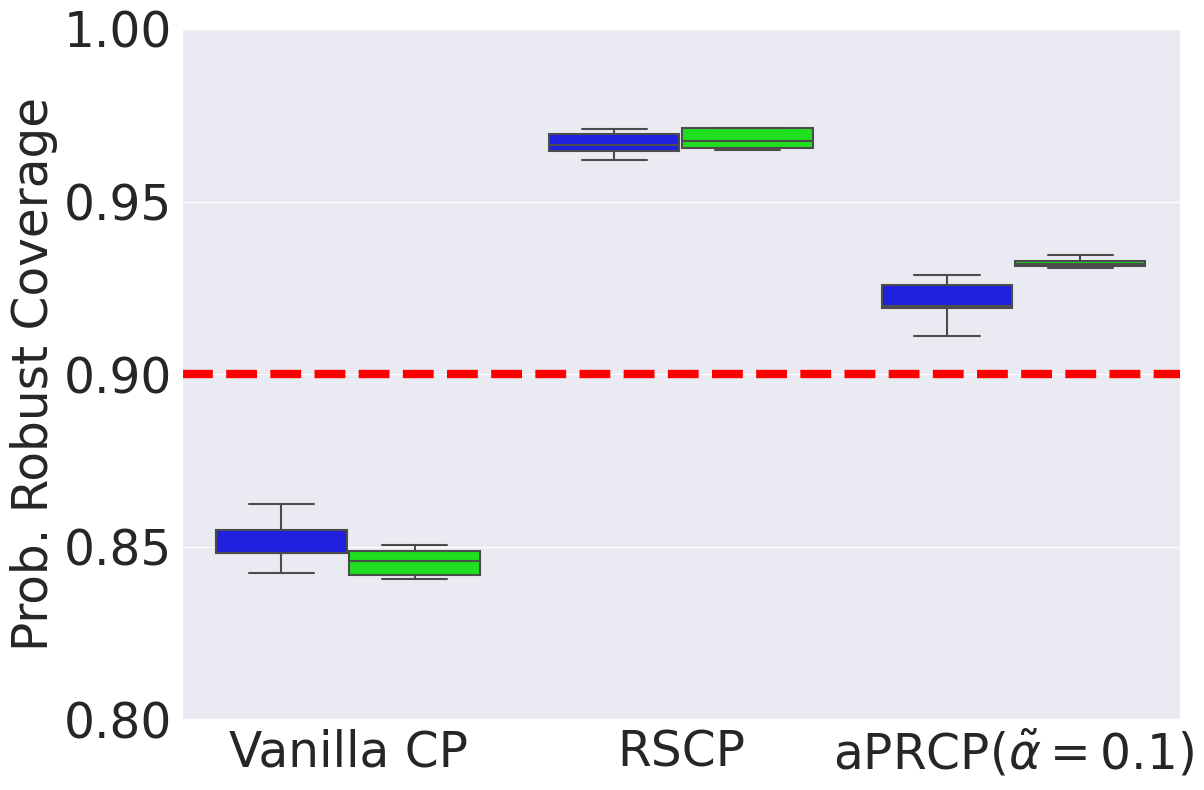}
        \end{minipage}
        \hfill
        \begin{minipage}{.33\linewidth}
            \centering
            \includegraphics[width=\linewidth]{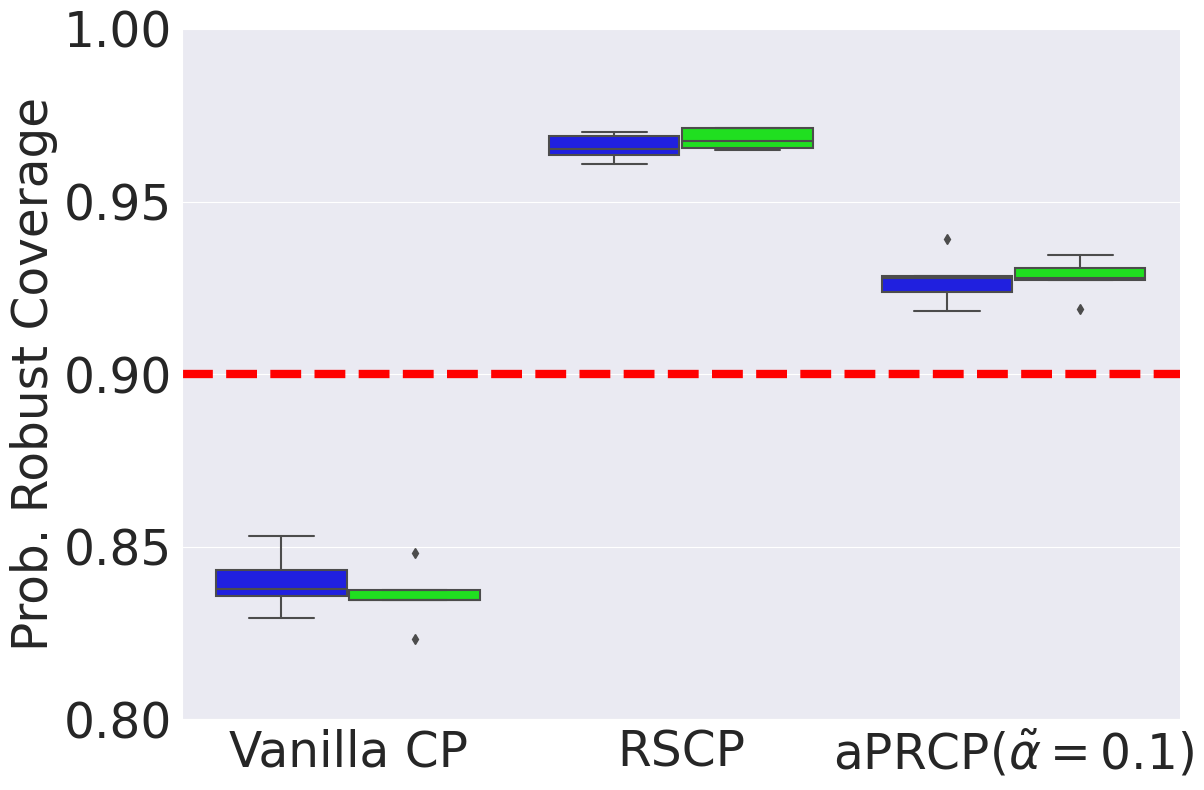}
        \end{minipage}    
        \hfill
        \begin{minipage}{.33\linewidth}
            \includegraphics[width=\linewidth]{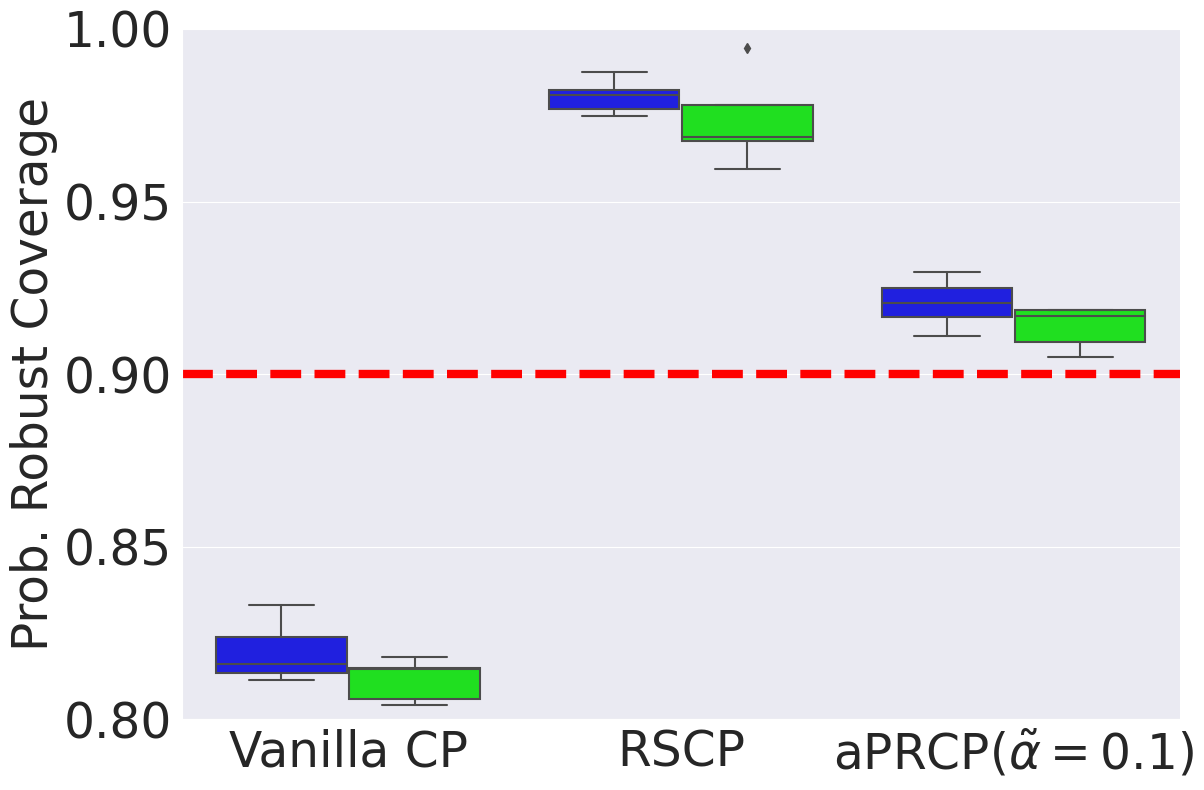}
        \end{minipage}
        \hfill
        \begin{minipage}{.33\linewidth}
            \centering
            \includegraphics[width=\linewidth]{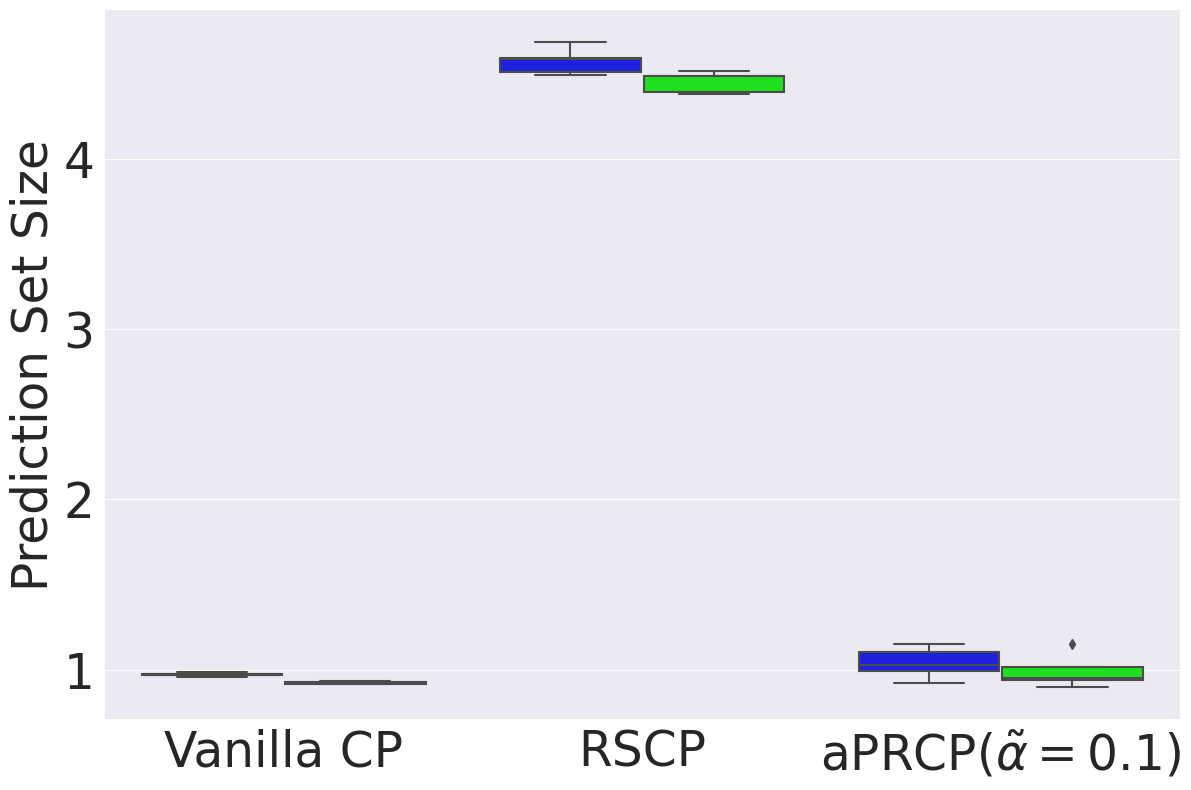}
        \end{minipage}    
        \hfill
        \begin{minipage}{.33\linewidth}
            \includegraphics[width=\linewidth]{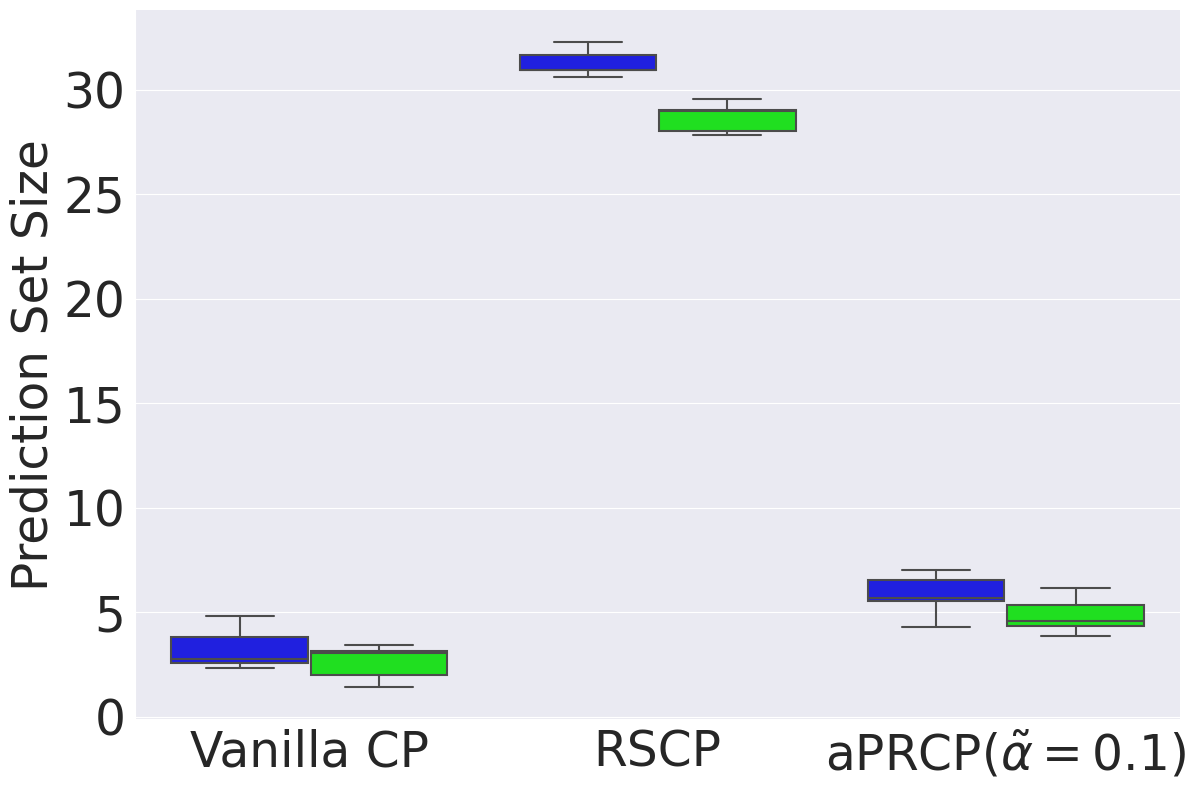}
        \end{minipage}
        \hfill
        \begin{minipage}{.33\linewidth}
            \includegraphics[width=\linewidth]{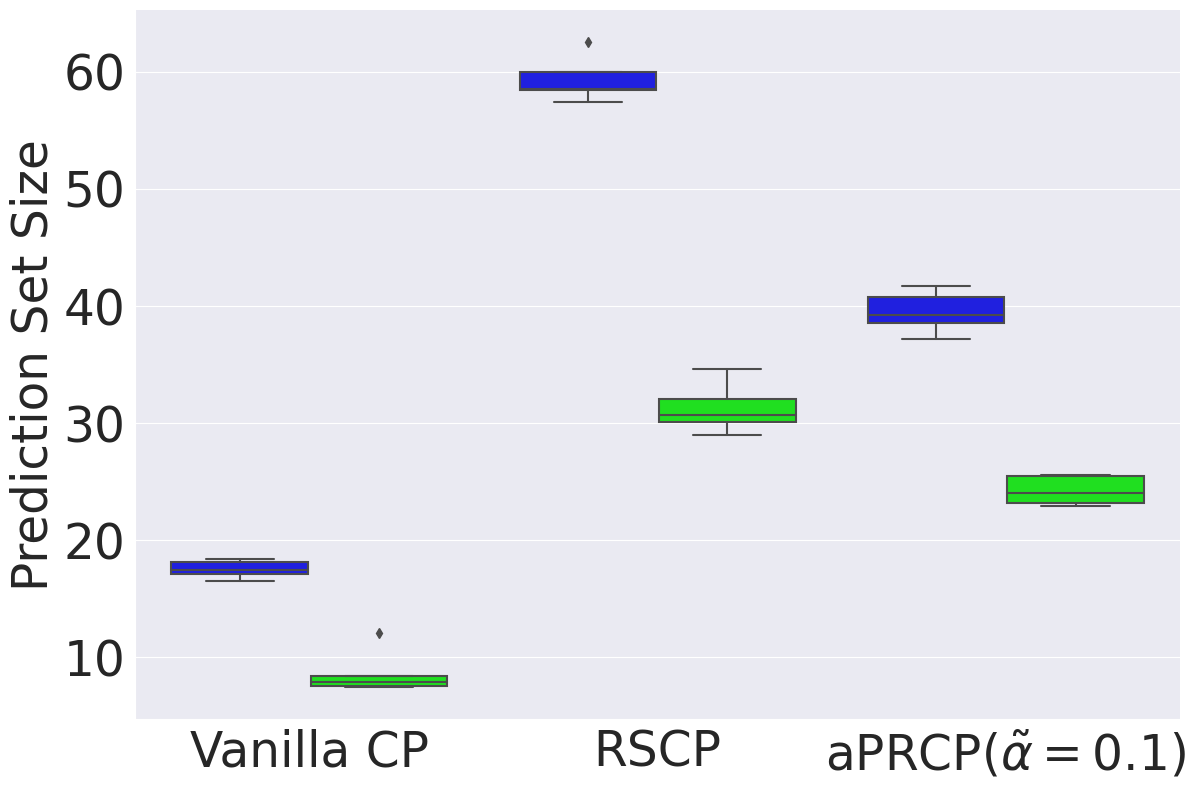}
        \end{minipage}
    \end{minipage}
    \caption{Probabilistic robust coverage (top) and prediction set size (bottom) constructed by \texttt{Vanilla CP}, \texttt{RSCP}, and \texttt{aPRCP}($\tilde{\alpha} = 0.1$) using HPS and APS scoring functions (target coverage is $90\%$). Results are reported over 50 runs.}
    \label{All_PRCP_mainPaper1}
\end{figure*}

\subsection{Experimental Setup}

\noindent {\bf Classification Datasets.} We consider three benchmark datasets for evaluation: CIFAR10 \citep{krizhevsky2009learning}, CIFAR100  \citep{krizhevsky2009learning}, and ImageNet  \citep{deng2009imagenet} using the standard training and test split. 

\noindent {\bf Deep Neural Network Models.} We consider ResNet-110  \citep{he2016deep} as the main model architecture for CIFAR10 and CIFAR100 and ResNet-50 for ImageNet in our experiments. We provide results on additional deep neural networks in the Appendix due to space constraints noting that we find similar patterns.
We train each model using two different approaches : {\em 1) Standard training:} The training is only performed using clean training examples; and {\em 2) Gaussian augmented training:} The training procedure employs Gaussian augmented examples  \citep{gendler2022adversarially} parameterized by a given standard deviation $\sigma=0.125$.

\noindent {\bf Methods and Baselines.} We consider two relevant state-of-the-art CP algorithms as our baselines. First, we employ \texttt{Vanilla CP} \citep{NEURIPS2020_244edd7e} designed for clean input examples. Second, we use randomly smooth conformal prediction (\texttt{RSCP})  \citep{gendler2022adversarially} which is designed to handle worst-case adversarial examples. We employ the publicly available implementations of \texttt{Vanilla CP}\footnote{\url{https://github.com/msesia/arc}} and \texttt{RSCP}\footnote{\url{https://github.com/Asafgendler/RSCP}} using the best settings suggested by their authors. 

We consider different configurations of our proposed adaptive probabilistically robust CP (\texttt{aPRCP}) algorithm. \texttt{aPRCP}(worst-adv) refers to the configuration where the evaluation of \texttt{aPRCP} is performed over adversarial examples generated using an adversarial attack algorithm. \texttt{aPRCP}($\tilde{\alpha}$) refers to the configuration where the evaluation is performed over noisy examples with a bounded perturbation on the test data. We provide additional results using different values for $\tilde{\alpha}$ in the Appendix.

\noindent{\bf Adversarial Attack Algorithms.}
To generate adversarial examples, we employ the white-box \texttt{PGD} attack algorithm  \citep{gendler2022adversarially} to evaluate \texttt{Vanilla CP} algorithm. For \texttt{RSCP} and \texttt{aPRCP(worst-adv)}, we employ an adapted \texttt{PGD} algorithm for smoothed classifiers as proposed in \citet{salman2019provably}. We provide additional results using different adversarial algorithms in the Appendix.

\noindent {\bf Evaluation Methodology.}  We present all our experimental results for desired coverage as $(1-\alpha)$=90\%. 
We report the average metrics (coverage and prediction set size) over 50 different runs for all datasets. We consider two different evaluation settings at the inference time as described below.

(a) \textbf{Probabilistic robustness evaluation}: We randomly sample $n_s = 128$ examples for each clean testing input: $X^{j}=X+\epsilon_j$ ($j$=1 to $n_s$), where $||\epsilon_j||_2 \leq r = 0.125$ for the CIFAR data and $||\epsilon_j||_2 \leq r = 0.25$ for the ImageNet data.  For a better span during the sampling procedure for each clean testing input, we sample two perturbations $\epsilon_j$ for each $r^{(k)}$ in $0<r^{(1)}<\cdots<r^{(k)}\le r$ such that $\|\epsilon_j\|_2=r^{(k)}$. 

We define both coverage and prediction set size metrics to adapt to the probabilistic robustness setting as follows:
{\em Coverage}: fraction of examples for which prediction set contains the ground-truth output.
\begin{equation}
\small
    \text{\normalsize Coverage} = \frac{1}{n_s} \sum_{j=1}^{n_s} \mathbbm{1}[Y_{n+1} \in \tilde{C}(X_{n+1} + \epsilon_j)].
    \label{eq:cvg_prob}
\end{equation}

{\em Efficiency}: average prediction set size, small values mean high efficiency.
\begin{equation}
    \text{Prediction Set Size} = \frac{1}{n_s}\sum_{j=1}^{n_s}\lvert\tilde{C}(X_{n+1}+\epsilon_j)\rvert,
    \label{eq:set_prob}
\end{equation}
where $||\epsilon_j||_2 \leq r = 0.125$ for CIFAR dataset, and $||\epsilon_j||_2 \leq r = 0.25$ for the ImageNet dataset. These re-defined metrics allow us to evaluate \texttt{aPRCP}($\tilde \alpha$) with different values of probability parameters $\tilde{\alpha}$ for probabilistic robustness. We provide additional results explaining the impact of the choice of the sampling distributions in the Appendix.

(b) \textbf{Worst-case evaluation:} We employ adversarial attack algorithms as mentioned above to create one worst-case adversarial example ($\tilde{X}$) for each clean testing input ($X$). We define both metrics for this setting as follows:
\begin{equation}
    \text{Coverage} = \mathbbm{1}[Y_{n+1} \in \tilde{C}(\tilde{X}_{n+1})].
    \label{eq:cvg}
\end{equation}
\begin{equation}
    \text{Prediction Set Size} = \lvert\tilde{C}(\tilde{X}_{n+1})\rvert.
    \label{eq:set}
\end{equation}

\begin{figure*}[!h]
    \centering
    \begin{minipage}{.92\linewidth}
        \begin{minipage}{\linewidth}
            \centering
            \includegraphics[width=.35\linewidth]{MainPaper/legend.png}
        \end{minipage}     
        \begin{minipage}{.33\linewidth}
            \centering
            (a) CIFAR10
        \end{minipage}
        \hfill
        \begin{minipage}{.33\linewidth}
            \centering
            (b) CIFAR100
        \end{minipage} 
        \begin{minipage}{.33\linewidth}
            \centering
            (a) ImageNet
        \end{minipage}
        \hfill
        \begin{minipage}{.33\linewidth}
            \includegraphics[width=\linewidth]{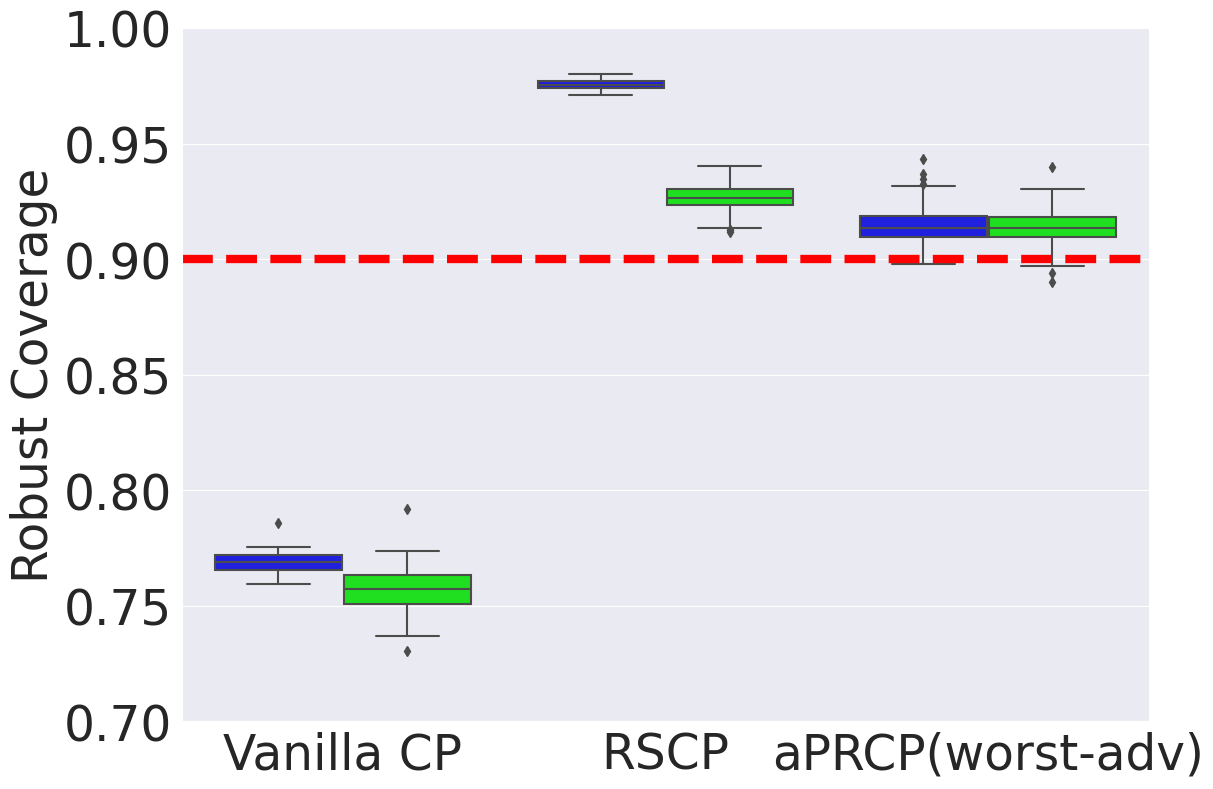}
        \end{minipage}
        \hfill
        \begin{minipage}{.33\linewidth}
            \centering
            \includegraphics[width=\linewidth]{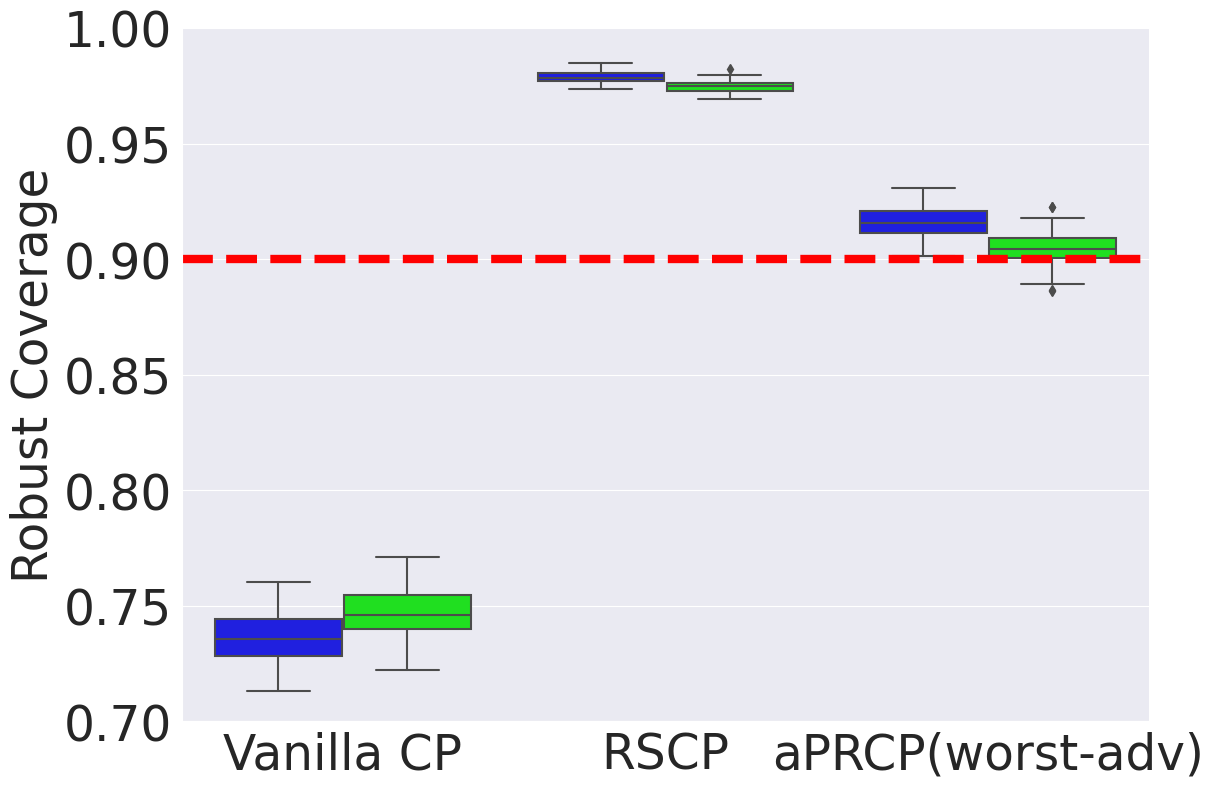}
        \end{minipage}    
        \hfill
        \begin{minipage}{.33\linewidth}
            \includegraphics[width=\linewidth]{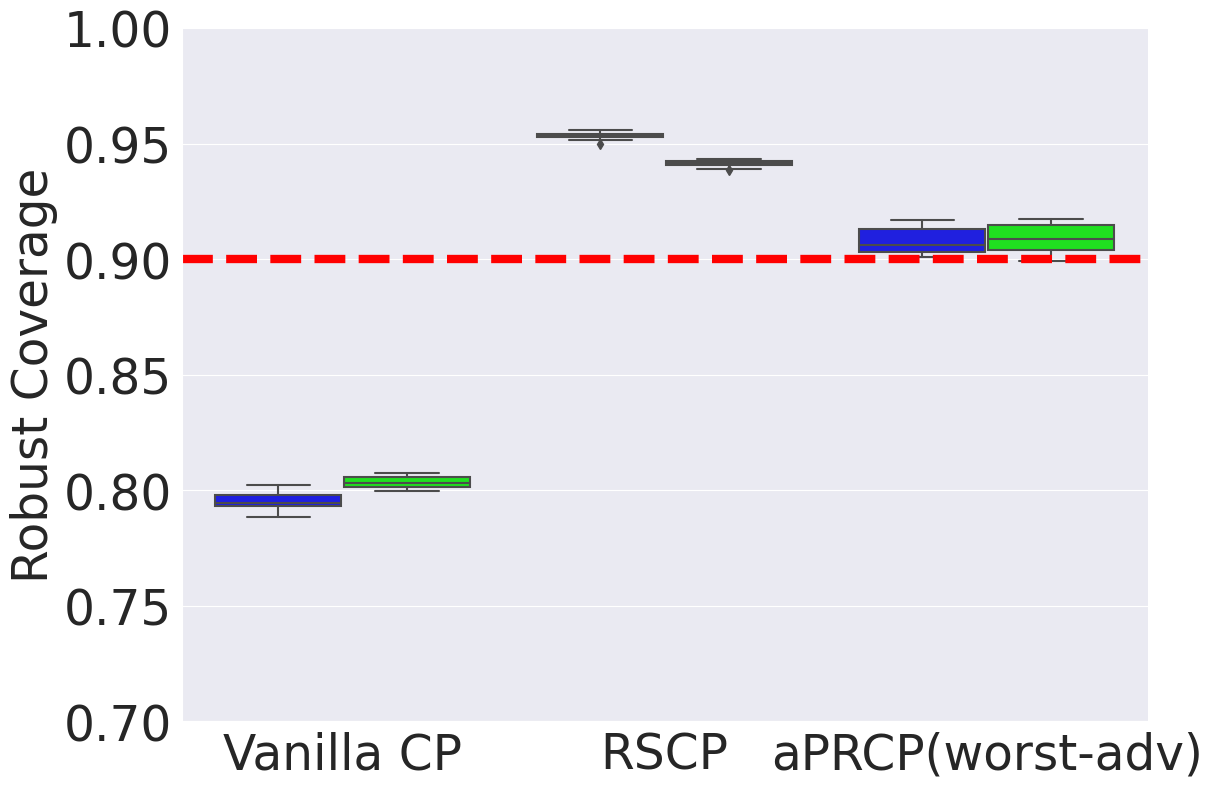}
        \end{minipage}
        \hfill
        \begin{minipage}{.33\linewidth}
            \centering
            \includegraphics[width=\linewidth]{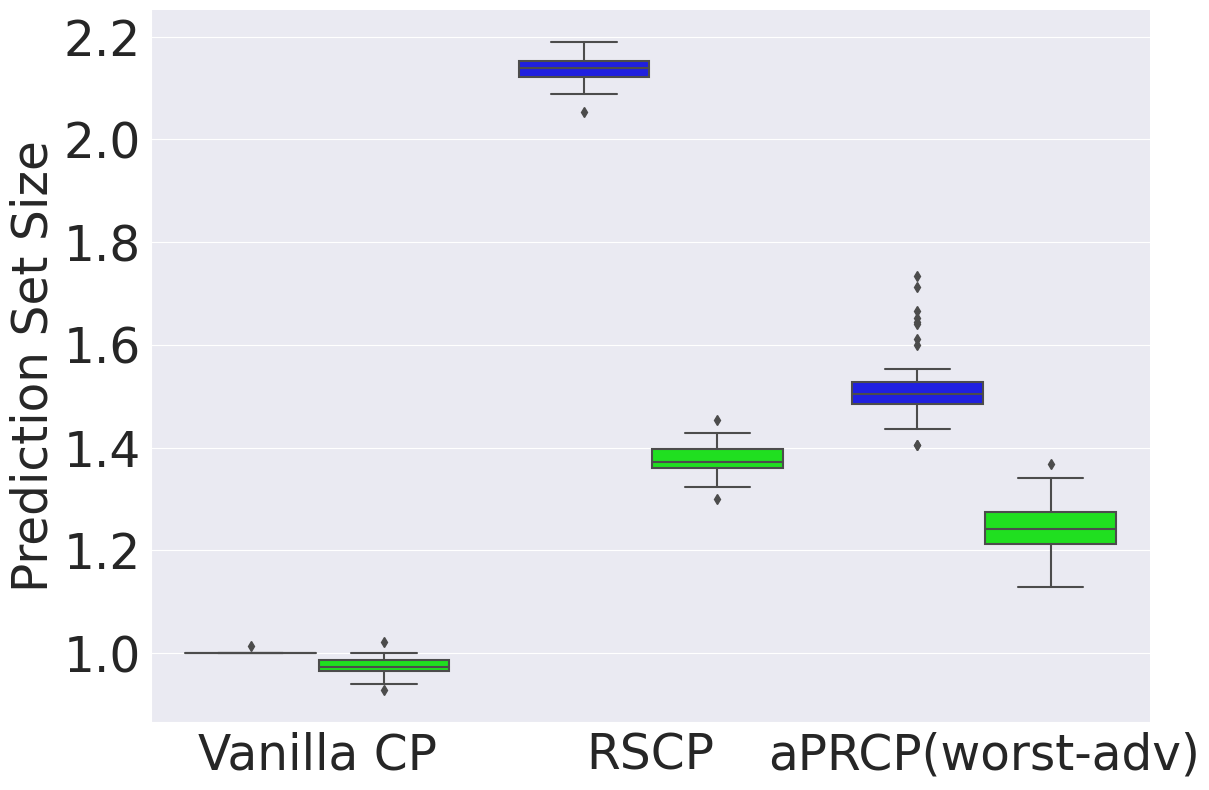}
        \end{minipage}    
        \hfill
        \begin{minipage}{.33\linewidth}
            \includegraphics[width=\linewidth]{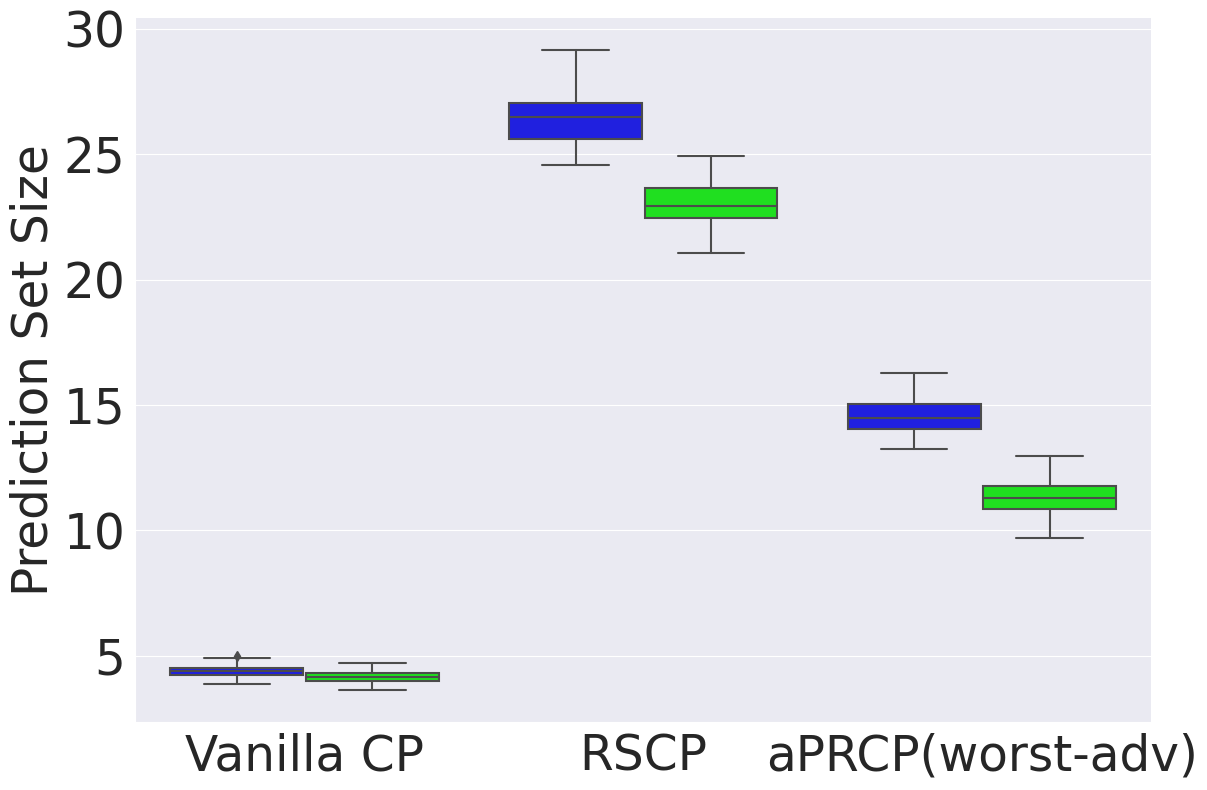}
        \end{minipage}
        \hfill
        \begin{minipage}{.33\linewidth}
            \includegraphics[width=\linewidth]{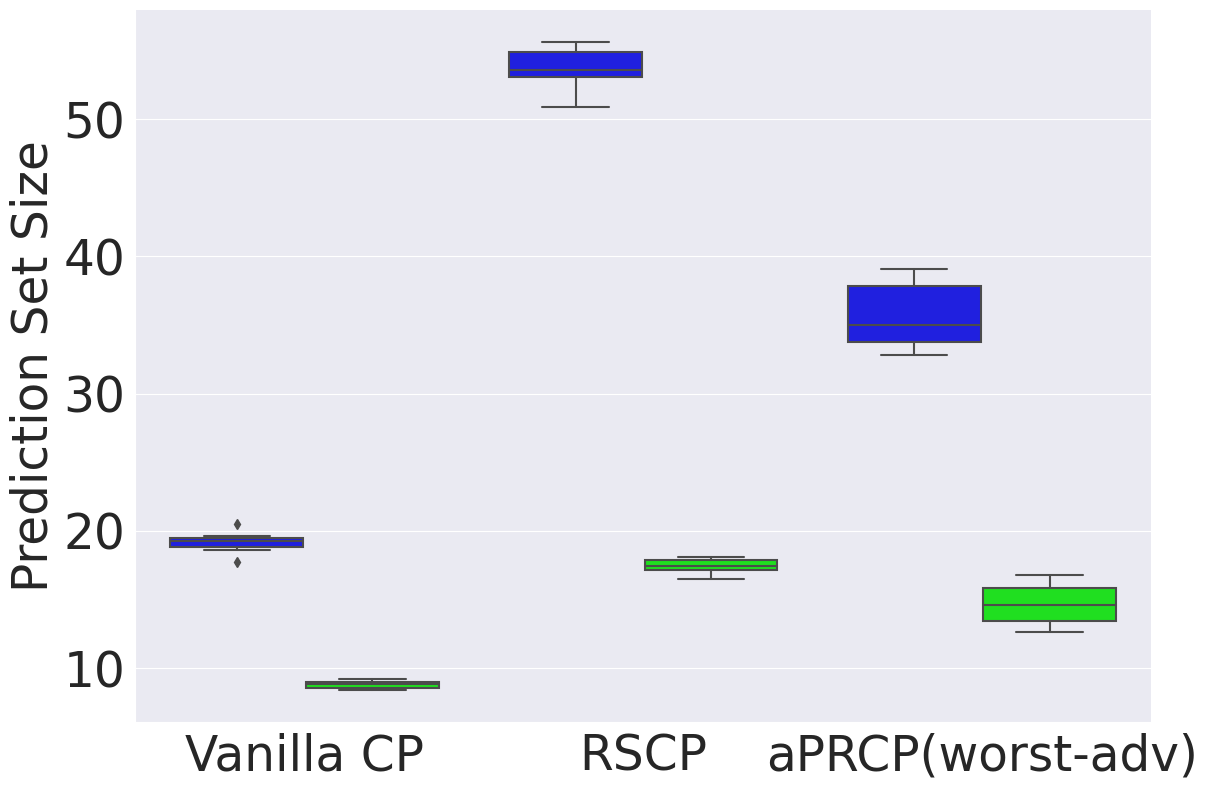}
        \end{minipage}
    \end{minipage}
    \caption{Adversarially robust coverage (top) and prediction set size (bottom) constructed by Vanilla CP, RSCP, and aPRCP(worst-adv) using HPS and APS scoring functions (target coverage is $90\%$). Results are reported over 50 runs.}
    \label{All_ARCP_mainPaper1}
\end{figure*}

\subsection{Results and Discussion}

{\bf Probabilistic Robust Coverage Performance.} 
Figure \ref{All_PRCP_mainPaper1} shows the probabilistic robustness performance (in terms of coverage and prediction set size) obtained by \texttt{Vanilla CP}, \texttt{RSCP}, and \texttt{aPRCP}($\tilde{\alpha} = 0.1$) for all three datasets using standard training. We make the following observations. 1) \texttt{Vanilla CP} algorithm fails in achieving the target probabilistic robust coverage. 2) \texttt{RSCP} algorithm achieves the desired probabilistic coverage, but has an empirical coverage significantly larger then 90\%. This yields very large prediction sets. Using APS, \texttt{RSCP} yields on average a prediction set of 30 labels for CIFAR100 and 60 for ImageNet. 3) \texttt{aPRCP}($\tilde{\alpha} = 0.1$) produces smaller prediction sets by keeping the actual coverage close to the target coverage. \texttt{aPRCP}($\tilde{\alpha} = 0.1$) reduces the prediction set by an average of 20 labels for CIFAR100 and ImageNet compared to RSCP method using any of the two non-conformity scores.

{\bf Adversarially Robust Coverage Performance.}
Figure \ref{All_ARCP_mainPaper1} shows the robust coverage and prediction set size obtained by \texttt{Vanilla CP}, \texttt{RSCP}, and \texttt{aPRCP}(worst-adv) achieved on the worst-case examples for three different datasets using Gaussian augmented training. We observe similar patterns as the probabilistic robust coverage results. 1) \texttt{Vanilla CP} fails to achieve the target coverage empirically. For all datasets, it achieves empirical coverage lower then 80\%. 2) Similar to the probabilistic robustness  results, \texttt{RSCP} method achieves an empirical coverage larger then 95\% for all datasets, yielding significantly large prediction sets for all datasets. 3) \texttt{aPRCP}(worst-adv) produces smaller prediction sets by keeping the actual coverage close to the target coverage (by a margin of 2\%) on worst-case adversarial examples. \texttt{aPRCP(worst-adv)} reduces the prediction set by more then 10 labels for CIFAR100 and ImageNet compared to RSCP method using any of the two non-conformity scores (HPS and APS).

\section{Related Work}
\textbf{Conformal Prediction.} CP is a general framework for uncertainty quantification that provides marginal coverage guarantees without any assumptions on the underlying data distribution \citep{shafer2008tutorial}. CP can be used for regression \citep{vovk2018cross,lei2018distribution,romano2019conformalized,izbicki2019flexible,guan2019conformal,gupta2022nested,kivaranovic2020adaptive,barber2021predictive,foygel2021limits} to produce prediction intervals and for classification \citep{lei2013distribution,sadinle2019least,romano2020classification,angelopoulos2021uncertainty,NCP} to produce prediction sets. Prior work has also considered instantiations of the CP framework to handle the differences between training and test distributions that is caused by long-term distribution shift \citep{gibbs2021adaptive}, covariate shift\citep{tibshirani2019conformal}, and label-distribution shift \citep{podkopaev2021distribution}. However, none of these existing works focus on the robustness setting where the distributional shift is caused by a bounded  adversarial perturbation. While using adversarial training seems intuitive to mitigate this problem, it was shown that vanilla CP cannot achieve the target coverage on adversarial data \citep{gendler2022adversarially}.

\textbf{Robust Conformal Prediction.} CP methods for robust coverage due to natural or adversarial perturbations is a new line of research that requires theoretical and empirical analysis.
Very few works have proposed variants of CP to handle adversarial robust settings.  The work on cautious deep learning \citep{hechtlinger2018cautious} proposed a CP-based prediction set construction that accounts for adversarial examples. However, this method does not provide any theoretical guarantees. Recently, randomly smoothed conformal prediction (RSCP) \citep{gendler2022adversarially} was proposed as a generalization for adversarial examples using randomized smoothing. This generalization is achieved by introducing a constant inflation condition that adjusts the CP quantile to adversarial perturbations. This adjustment is proportional to the potential adversarial perturbations that can affect the test data. Hence, RSCP is prone to produce large prediction sets along with high marginal coverage to achieve  robustness. 

We study the general setting of probabilistically robust CP and develop probably correct algorithms to achieve improved trade-offs for nominal and robust performance over vanilla CP and RSCP. The key differences between our work (aPRCP) and RSCP are: 1) aPRCP uses a {\em quantile-of-quantile} design and does not require finding a score inflation constant like RSCP. 2) RSCP requires the design of a specialized scoring function while aPRCP can employ any existing score function. 3) aPRCP does not have test-time overhead unlike RSCP due to the generation of samples.

\section{Summary and Future Work}

This paper studied the novel problem of probabilistic robustness for conformal prediction (PRCP) based uncertainty quantification of deep classifiers. We developed the adaptive PRCP (aPRCP) algorithm based on the principle of quantile-of-quantile design and theoretically analyzed its effectiveness to achieve improved trade-offs between performance on clean data and robustness to adversarial examples. Our experiments on multiple image datasets using deep classifiers demonstrated the effectiveness of aPRCP over vanilla CP methods and adversarially robust CP methods. Future work should study and analyze end-to-end PRCP algorithms.

\section*{Acknowledgements}

This research is supported in part by Proofpoint Inc. and the
AgAID AI Institute for Agriculture Decision Support, supported by the National Science Foundation and United States
Department of Agriculture - National Institute of Food and
Agriculture award \#2021-67021-35344. The authors would
like to thank the feedback from anonymous reviewers who
provided suggestions to improve the paper.

\newpage
\bibliography{reference}
\newpage
\input{appendix}

\end{document}

%% file: appendix.tex
\appendix
\onecolumn
\section{Technical Proofs}

In this section, we prove the theoretical results in the main paper.
To make it complete and self-contained, we also include the proof of Proposition 1, i.e., Theorem 1 in \citep{gendler2022adversarially}, with the framework and notations used in our paper.

\begin{proposition}
\label{proposition:AR_coverage_ARCP_appendix}
(Proposition 1 restated, adversarially robust coverage of RSCP, Theorem 1 in \citep{gendler2022adversarially})
Assume the score function $S$ is $M_r$-adversarially inflated.
Let $\calC^\AR(\widetilde X) = \{ y \in \calY : S(\widetilde X, y) \leq \tau^\AR(\alpha) \}$ be the prediction set for a testing sample $\widetilde X$.
Then RSCP achieves ($1-\alpha$)-adversarially robust coverage.
\end{proposition}
\begin{proof}
(of Proposition \ref{proposition:AR_coverage_ARCP_appendix})

After reviewing the inflated quantile in the adversarial sense, we extend it to the following probabilistic sense.
\begin{align*}
\P_Z \{ 
S(X + \epsilon, Y) \leq \tau^\AR(\alpha)
\}
\geq &
\P_Z \{
S(X, Y) + M_r
\leq 
\tau^\AR(\alpha)
\}
\\
= &
\P_Z \{
S(X, Y) + M_r
\leq 
Q(\alpha) + M_r
\}
\\
= &
\P_Z \{ S(X, Y) \leq Q(\alpha) \}
\\
= &
\P_{X,Y} \{ S(X, Y) \leq Q(\alpha) \}
\geq 
1 - \alpha ,
\end{align*}
where the first inequality is due to the condition of $M_r$-adversarially inflated conformity score function (Definition 2), the first equality is due to the setting of the inflated threshold $\tau^\AR(\alpha) = Q(\alpha) + M_r$,
and the last inequality is due to the definition of quantile $Q(\alpha)$.
\end{proof}

\begin{proposition}
\label{proposition:PR_coverage_iPRCP_appendix}
(Proposition 2 restated, probabilistically robust coverage of iPRCP)
Assume the score function $S$ is $M_{r, \eta}$-probabilistically inflated.
Let $\calC^\iPR(\widetilde X) = \{ y \in \calY : S(\widetilde X, y) \leq \tau^\iPR(\alpha; \eta) \}$ be the prediction set for a testing sample $\widetilde X=X+\epsilon$. 
Then iPRCP achieves ($1-\alpha$)-probabilistically robust coverage.
\end{proposition}

\begin{proof}
(of Proposition \ref{proposition:PR_coverage_iPRCP_appendix})

Denote $A_{r, \eta} = \{ Z \in \calX \times \calY \times \calE_r : S(X + \epsilon, Y) \leq S(X, Y) + M_{r, \eta} \}$, which implies $\P_Z \{ Z \in A_{ r, \eta } \} \geq 1 - \eta$.
Recall $\tau^\iPR(\alpha'; \eta) = Q(\alpha') + M_{r, \eta}$ for $\alpha'$ and $\eta$.
\begin{align*}
&
\P_Z \{ S(X + \epsilon, Y) \leq \tau^\iPR(\alpha'; \eta) \}
\\
= &
\P\{ Z \in A_{r, \eta} \} \cdot \P_Z \{ S(X + \epsilon, Y) \leq \tau^\iPR(\alpha'; \eta) | Z \in A_{r, \eta} \}
\\
&
+ \P\{ Z \notin A_{r, \eta} \} \cdot \P_Z \{ S(X + \epsilon, Y) \leq \tau^\iPR(\alpha'; \eta) | Z \notin A_{r, \eta} \}
\\
\geq & 
( 1 - \eta ) \cdot \P_Z \{ S(X + \epsilon, Y) \leq \tau^\iPR(\alpha'; \eta) | Z \in A_{r, \eta} \}
\\
\geq &
( 1 - \eta ) \cdot \P_Z \{ S(X, Y) + M_{r, \eta} \leq Q(\alpha') + M_{r, \eta} | Z \in A_{r, \eta} \} 
\\
= &
( 1 - \eta ) \cdot \P_{X, Y} \{ S(X, Y) \leq Q(\alpha') \}
\\
\geq &
( 1 - \eta ) ( 1 - \alpha' ) ,
\end{align*}
where the first inequality is due to the non-negativity of probability and the definition of $A_{r, \eta}$, and
the second inequality is due to $M_{r, \eta}$-probabilistically inflated score function (7).

In this case, define $\alpha^*_\iPR(\alpha; \eta) := \max\{ \alpha' : (1-\eta)(1-\alpha') \geq 1-\alpha \}$, and we can use $\tau^\iPR(\alpha^*_\iPR(\alpha; \eta); \eta)$ as the threshold to derive $(1-\alpha)$-probabilistically robust coverage.
However, we have to know the conformity score function very well, so that we access the value of $M_{r, \eta}$ given $\eta$ to determine $\tau^*_\iPR(\alpha; \eta)$, which is not always possible in practice.
\end{proof}

\begin{theorem}
\label{theorem:appendix:prob_robust_coverage_aPRCP}
(Theorem 1 restated, probabilistically robust coverage of aPRCP)
Let $\calC^\aPR(\widetilde X = X + \epsilon) = \{ y \in \calY : S(\widetilde X, y) \leq \tau^\aPR(\alpha; s) \}$ be the prediction set for a testing sample $\widetilde X$.
Then aPRCP achieves ($1-\alpha$)-probabilistically robust coverage.
\end{theorem}

\begin{proof}
(of Theorem \ref{theorem:appendix:prob_robust_coverage_aPRCP})

Denote $B = \{ (X, Y) \in \calX \times \calY : Q^\rob(X, Y; \alpha^*_\aPR) \leq \tau^\aPR(\alpha; s)\}$,
which implies that 
\begin{align}\label{eq:prob_B}
\P_{X,Y}\{ (X,Y) \in B \} \geq 1-\alpha+s
\end{align}
due to the definition of $\tau^\aPR(\alpha; s)$ in (9).
We simply check whether $\tau^\aPR(\alpha; s)$ can give us probabilistically robust coverage as follows:
\begin{align}\label{eq:PRCP_coverage}
&
\P_Z \{ S(X + \epsilon, Y) \leq \tau^\aPR(\alpha; s) \}
\nonumber\\
= &
\P_{X, Y}\{ X, Y : Q^\rob(X, Y; \alpha^*_\aPR) \leq \tau^\aPR(\alpha; s) \} \cdot \P_{\epsilon | X, Y} \{ S(X + \epsilon, Y) \leq \tau^\aPR(\alpha; s) \}
\nonumber\\
&
+ \P_{X, Y}\{ X : Q^\rob(X, Y; \alpha^*_\aPR) > \tau^\aPR(\alpha; s) \} \cdot \P_{\epsilon | X, Y} \{ S(X + \epsilon, Y) \leq \tau^\aPR(\alpha; s) \}
\nonumber\\
\geq &
\P_{X, Y}\{ X, Y : Q^\rob(X, Y; \alpha^*_\aPR) \leq \tau^\aPR(\alpha; s) \} \cdot \P_{\epsilon | (X, Y) \in B} \{ S(X + \epsilon, Y) \leq \tau^\aPR(\alpha; s) \}
\nonumber\\
\geq &
\P_{X, Y}\{ (X, Y) \in B \} \cdot \P_{\epsilon | (X, Y) \in B} \{ S(X + \epsilon, Y) \leq Q^\rob(X, Y; \alpha^*_\aPR) \}
\nonumber\\
\geq &
( 1 - \alpha + s ) \cdot \P_{\epsilon | (X, Y) \in B } \{ S(X + \epsilon, Y) \leq Q^\rob(X, Y; \alpha^*_\aPR) \}
\\
\geq &
( 1 - \alpha + s ) ( 1 - \alpha^*_\aPR ),
\nonumber
\end{align}
where the first inequality is due to the non-negativity of probability,
the second inequality is due to $Q^\rob(X,Y;\alpha^\aPR(\alpha)) \leq \tau^\aPR(\alpha; s)$ for $(X,Y) \in B$,
the third inequality is due to (\ref{eq:prob_B}),
and the last inequality is due to the definition of robust quantile $Q^\rob(X, Y; \tilde \alpha)$ in (8).

Recall $\alpha^*_\aPR = 1 - (1-\alpha) / (1-\alpha + s)$, so $( 1 - \alpha + s ) ( 1 - \alpha^*_\aPR ) = 1-\alpha$, which shows
\begin{align*}
\P_Z \{ S(X + \epsilon, Y) \leq \tau^\aPR(\alpha; s) \}
\geq 
1 - \alpha .
\end{align*}
\end{proof}

\begin{lemma}
\label{lemma:cross_domain_noise_coverage}
(Inflated probability for cross domain noise)
Assume $ \P_{\epsilon \sim \calP_\epsilon^{cal}}\{\epsilon\} - \P_{\epsilon \sim \calP_\epsilon^{test}}\{\epsilon\} \leq d$ for all $\| \epsilon \| \leq r$.
Then, for any threshold $\tau$, the following inequality holds:
\begin{align}
\label{eq:lemma1}
\P_{ \epsilon \sim \calP_\epsilon^{cal} | X, Y } \{ S(X + \epsilon, Y) \leq \tau \}
-
\P_{ \epsilon \sim \calP_\epsilon^{test} | X, Y } \{ S(X + \epsilon, Y) \leq \tau \} 
\leq
d .
\end{align}
\end{lemma}

\begin{proof}
(of Lemma \ref{lemma:cross_domain_noise_coverage})

\begin{align*}
&
\P_{\epsilon \sim \calP_\epsilon^{cal}} \{ S(X + \epsilon, Y) \leq \tau \}
-
\P_{\epsilon \sim \calP_\epsilon^{test}} \{ S(X + \epsilon, Y) \leq \tau \}
\\
= &
\E_{\epsilon \sim \calP_\epsilon^{cal}} [ \indicator [ S(X + \epsilon, Y) \leq \tau ] ]
-
\E_{\epsilon \sim \calP_\epsilon^{test}} [ \indicator[ S(X + \epsilon, Y) \leq \tau ] ]
\\
= &
\int_\epsilon \P_{\epsilon \sim \calP_\epsilon^{cal}} \{ \epsilon \} \cdot \indicator [ S(X + \epsilon, Y) \leq \tau ] d\epsilon
- \int_\epsilon \P_{\epsilon \sim \calP_\epsilon^{test}} \{ \epsilon \} \cdot \indicator [ S(X + \epsilon, Y) \leq \tau ] d\epsilon
\\
= &
\int_\epsilon \Big( \P_{\epsilon \sim \calP_\epsilon^{cal}} \{ \epsilon \} - \P_{\epsilon \sim \calP_\epsilon^{test}} \{ \epsilon \} \Big) \cdot \indicator [ S(X + \epsilon, Y) \leq \tau ] d\epsilon
\\
\leq &
\int_\epsilon ( d \cdot 1 ) d\epsilon
=
d .
\end{align*}
\end{proof}

\begin{theorem}
\label{theorem:appendix:prob_robust_coverage_aPRCP_cross_domain_noise}
(Theorem 2 restated, probabilistically robust coverage of aPRCP for cross domain noise)
Let $\calP_\epsilon^{test}$ and $\calP_\epsilon^{cal}$ denote different distributions of $\epsilon$ during the testing and calibration phase, respectively.
Assume $\P_{\epsilon \sim \calP_\epsilon^{cal}}\{\epsilon\} - \P_{\epsilon \sim \calP_\epsilon^{test}}\{\epsilon\} \leq d$ for all $\| \epsilon \| \leq r$.
Set $\alpha^*_\aPR = 1 - d - ( 1 - \alpha) / (1 - \alpha + s )$ in (9).
Let $\calC^\aPR(\widetilde X = X + \epsilon) = \{ y \in \calY : S(\widetilde X, y) \leq \tau^\aPR(\alpha; s) \}$ be the prediction set for a testing sample $\widetilde X$.
Then aPRCP achieves ($1-\alpha$)-probabilistically robust coverage under $\calP_\epsilon^{test}$.
\end{theorem}

\begin{proof}
(of Theorem \ref{theorem:appendix:prob_robust_coverage_aPRCP_cross_domain_noise})
We start with (\ref{eq:PRCP_coverage}) in the proof of Theorem \ref{theorem:appendix:prob_robust_coverage_aPRCP_cross_domain_noise} which only considers the noise $\epsilon$ drawn from the same distribution during calibration and testing as follows.
\begin{align*}
&
\P_{ X, Y, \epsilon \sim \calP_\epsilon^{test}} \{ S(X + \epsilon, Y) \leq \tau^\aPR(\alpha; s) \}
\\
\geq &
( 1 - \alpha + s ) \cdot \P_{\epsilon \sim \calP_\epsilon^{test} | (X, Y) \in B } \{ S(X + \epsilon, Y) \leq Q^\rob(X, Y; \alpha^*_\aPR) \}
\\
\geq &
( 1 - \alpha + s ) \cdot \Big( \P_{\epsilon \sim \calP_\epsilon^{cal} | (X, Y) \in B } \{ S(X + \epsilon, Y) \leq Q^\rob(X, Y; \alpha^*_\aPR) \} - d \Big)
\\
\geq &
( 1 - \alpha + s ) \cdot \Bigg( 1 - \Big( 1 - d - \frac{1-\alpha}{1-\alpha+s} \Big) - d \Bigg)
\\
= &
( 1 - \alpha + s ) \cdot \frac{ 1 - \alpha }{ 1 - \alpha + s }
= 
1 - \alpha ,
\end{align*}
where the first inequality follows (\ref{eq:PRCP_coverage}), 
the second inequality is due to inequality \ref{eq:lemma1} in Lemma \ref{lemma:cross_domain_noise_coverage}, and
the third inequality is due to the definition $Q^\rob(X, Y; \alpha^*_\aPR)$ in (8) with $\alpha^*_\aPR = 1 - d - (1-\alpha) / (1-\alpha+s)$.
\end{proof}

\begin{corollary}
\label{corollary:compare_aPRCP_ARCP_appendix}
(Corollary 3 restated)
To achieve the same ($1-\alpha$)-probabilistically robust coverage on $Z$, the following inequalities hold: \begin{align*}
\min_{ \eta \in [0, \alpha] } \tau^\iPR(\alpha; \eta) \leq \tau^\AR(\alpha), ~~
\min_{ s \in [0, \alpha] }  \tau^\aPR(\alpha; s) \leq \tau^\AR(\alpha) .
\end{align*}
\end{corollary}

\begin{proof}
(of Corollary \ref{corollary:compare_aPRCP_ARCP_appendix})
For adaptive PRCP, if $s = 0$, to achieve ($1-\alpha$)-probabilistically robust coverage over $Z$, we must have $\alpha^*_\aPR = 0$.
Since $\alpha^*_\aPR$ controls how aggressively we derive the robust quantile for $(X, Y)$, it indicates that we have to consider $1$-robust quantile.
This is equivalent to deriving the adversarial $S(X+\epsilon, Y)$ for all $(X, Y)$.

For inflated PRCP, if $\eta=0$, to achieve ($1-\alpha$)-probabilistically robust coverage, we have $M_{\delta, \eta} = M_\delta$ and $\alpha^*_\iPR = \alpha$, recovering ARCP (adversarially robust conformal prediction).
This case is exactly the same with adpative PRCP with $s=0$.
Therefore, $\tau^\AR(\alpha) = \tau^\iPR(\alpha; 0) = \tau^\aPR(\alpha; 0)$.

Note that $\min_{s \in [0, \alpha]} \tau^\aPR(\alpha; s) \leq \tau^\aPR(\alpha; 0)$ 
and $\min_{\eta\in [0, \alpha]} \tau^\iPR(\alpha; \eta) \leq \tau^\iPR(\alpha; 0)$,
so by tuning the value of $s$ for aPRCP and the value of $\eta$ for iPRCP, to achieve the same probabilistically robust coverage $1-\alpha$, we can have a more efficient threshold than ARCP.
\end{proof}

\begin{proposition}
\label{proposition:empirical_quantile_concentration_appendix}
(Proposition 3 restated, concentration inequality for quantiles)
Let $Q(\alpha) = \max\{ t : \P_V\{ V \leq t \} \geq 1 - \alpha \}$ be the true quantile of a random variable $V$ given $\alpha$,
and $\widehat Q_n(\alpha) = V_{ ( \lceil (n+1) ( 1 - \alpha ) \rceil ) }$ be the empirical quantile estimated by $n$ randomly sampled set $\{V_1, ..., V_n\}_{i=1}^n$.
Then with probability at least $1-\delta$, we have
$
\widehat Q_n(\alpha + \tilde O(1/\sqrt{n}))
\leq
Q(\alpha)
\leq
\widehat Q_n(\alpha - \tilde O(1/\sqrt{n}))
$  
where $\tilde O$ hides the logarithmic factor.
\end{proposition}

\begin{proof}
(of Proposition \ref{proposition:empirical_quantile_concentration_appendix})

Define $Z_i = \indicator{ [ V_i \leq Q(\alpha) ] }$ where $1 \leq i \leq n$ and $\indicator[\cdot]$ is an indicator function.
Then $Z_{i}$ is a Bernoulli random variable with $\P\{ Z_i = 1 \} = 1 - \alpha$ and $\P\{ Z_i = 0 \} = \alpha$ from the definition of $Q(\alpha)$.
Let $\widehat Z = \frac{1}{n} \sum_{i=1}^n Z_i$ and $\E[\widehat Z] = 1-\alpha.$

According to Chernoff bound, we know
\begin{align*}
\P\Bigg\{ \Bigg| \frac{1}{n} \sum_{i=1}^n Z_i - \E[\widehat Z] \Bigg| \geq \varepsilon \E[\widehat Z] \Bigg\}
\leq 
2 \exp\Bigg( - \E[\widehat Z] \varepsilon^2 / 3 \Bigg) 
=
2 \exp\Bigg( - n (1-\alpha) \varepsilon^2 / 3 \Bigg) .
\end{align*}

By setting $\delta = 2 \exp( - n (1-\alpha) \varepsilon^2 / 3 )$, i.e., $\varepsilon = \sqrt{ ( 3 \log(2/\delta) ) / ( ( 1 - \alpha ) n  ) }$, we have with probability at least $1-\delta$:
\begin{align}\label{eq:abs_bound}
\Bigg| \frac{1}{n} \sum_{i=1}^n \indicator[ V_i \leq Q(\alpha) ] - ( 1 - \alpha ) \Bigg| 
\leq 
\varepsilon ( 1 - \alpha )
=
\sqrt{ ( 3 ( 1 - \alpha )  \log(2/\delta) ) / n }
=
\tilde O(1 / \sqrt{n}) .
\end{align}

Recall the definition of the empirical quantile $\widehat Q_n(\alpha)$ given $\alpha$:
\begin{align*}
\widehat Q_n(\alpha) = \max\Bigg\{ t : \frac{1}{n} \sum_{i=1}^n \indicator[ V_i \leq t ] \geq 1 - \alpha \Bigg\} .
\end{align*}
Then we know the following upper bound and lower bound for $1-\alpha$:
\begin{align*}
( 1 - \alpha )
\leq 
\frac{1}{n} \sum_{i=1}^n \indicator[ V_i \leq \widehat Q_n(\alpha) ] , ~~~
( 1 - \alpha )
\geq 
\frac{1}{n} \sum_{i=1}^n \indicator[ V_i \leq \widehat Q_n(\alpha + 1 / n ) ] .
\end{align*}

Re-arranging (\ref{eq:abs_bound}) and using the above upper/lower bounds, with probability at least $1-\delta$, we have
\begin{align*}
&
( 1 - \alpha ) ( 1 - \varepsilon )
\leq 
\frac{1}{n} \sum_{i=1}^n \indicator[ V_i \leq Q(\alpha) ]
\leq 
( 1 - \alpha ) ( 1 + \varepsilon)
\\
\Leftrightarrow ~~~
& 
1 - ( \underbrace{ 1 - ( 1 - \alpha ) ( 1 - \varepsilon ) }_{ = \alpha' } )
\leq 
\frac{1}{n} \sum_{i=1}^n \indicator[ V_i \leq Q(\alpha) ]
\leq 
1 - ( \underbrace{ 1 - ( 1 - \alpha ) ( 1 + \varepsilon) }_{ = \alpha'' } )
\\
\Rightarrow ~~~ 
& 
\frac{1}{n} \sum_{i=1}^n \indicator[ V_i \leq \widehat Q_n( \alpha' + 1/n ) ]
\leq
\frac{1}{n} \sum_{i=1}^n \indicator[ V_i \leq Q(\alpha) ]
\leq 
\frac{1}{n} \sum_{i=1}^n \indicator[ V_i \leq \widehat Q_n( \alpha'' ) ] 
\\
\Leftrightarrow ~~~
&
\widehat Q_n(\alpha' + 1/n)
\leq 
Q(\alpha)
\leq 
\widehat Q_n(\alpha'') .
\end{align*}

Finally, we analyze $\alpha'$ and $\alpha''$ as follows
\begin{align*}
\alpha' 
=
1 - (1-\alpha) (1-\varepsilon)
=
\alpha + \varepsilon (1-\alpha)
=
\alpha + \sqrt{ 3 ( 1 - \alpha ) \log(2/\delta) / n }
=
\alpha + \tilde O(1/\sqrt{n}),
\\
\alpha''
=
1 - (1-\alpha) (1+\varepsilon)
=
\alpha - \varepsilon (1-\alpha)
=
\alpha - \sqrt{ 3 ( 1 - \alpha ) \log(2/\delta) / n }
=
\alpha - \tilde O(1/\sqrt{n}).
\end{align*}

Therefore, we have
\begin{align*}
\widehat Q_n(\alpha + \tilde O(1/\sqrt{n}))
\leq
Q(\alpha)
\leq
\widehat Q_n(\alpha - \tilde O(1/\sqrt{n})) .
\end{align*}
\end{proof}

\section{ADDITIONAL EXPERIMENTS AND IMPLEMENTATION DETAILS}

\textbf{Implementation details.}
Table \ref{tab:appendix_Acc_clean_adv} shows the testing accuracy of the different deep models using both standard training ($\sigma=0$) and Gaussian augmented training ($\sigma>0$). 
\begin{table*}[!h]
\centering
\begin{tabular}{|c|c|cc|cc|cc|}
\hline
\multirow{2}{*}{Architecture} & \multirow{2}{*}{Training} & \multicolumn{2}{c|}{CIFAR10}             & \multicolumn{2}{c|}{CIFAR100}            & \multicolumn{2}{c|}{ImageNet}            \\ \cline{3-8} 
                              &                                    & \multicolumn{1}{c|}{Clean(\%)} & Adv(\%) & \multicolumn{1}{c|}{Clean(\%)} & Adv(\%) & \multicolumn{1}{c|}{Clean(\%)} & Adv(\%) \\ \hline
\multirow{2}{*}{ResNet-110}   & $\sigma = 0.0$                                & \multicolumn{1}{c|}{89.99}     & 26.71   & \multicolumn{1}{c|}{71.12}     & 12.20   & \multicolumn{1}{c|}{-}         & -       \\ \cline{2-8} 
                              & $\sigma = 0.125$                              & \multicolumn{1}{c|}{81.70}     & 67.80   & \multicolumn{1}{c|}{58.11}     & 42.01   & \multicolumn{1}{c|}{-}         & -       \\ \hline
\multirow{2}{*}{VGG-19}       & $\sigma = 0.0$                                & \multicolumn{1}{c|}{93.10}     & 54.96   & \multicolumn{1}{c|}{72.22}     & 23.10   & \multicolumn{1}{c|}{-}         & -       \\ \cline{2-8} 
                              & $\sigma = 0.125$                              & \multicolumn{1}{c|}{86.50}     & 72.10   & \multicolumn{1}{c|}{55.12}     & 40.85   & \multicolumn{1}{c|}{-}         & -       \\ \hline
\multirow{2}{*}{DenseNet-161} & $\sigma = 0.0$                                & \multicolumn{1}{c|}{95.42}     & 23.28   & \multicolumn{1}{c|}{77.10}     & 04.30   & \multicolumn{1}{c|}{-}     & -   \\ \cline{2-8} 
                              & $\sigma = 0.125$                              & \multicolumn{1}{c|}{88.17}     & 73.15   & \multicolumn{1}{c|}{60.32}     & 46.91   & \multicolumn{1}{c|}{-}         & -       \\ \hline
\multirow{2}{*}{ResNet-50}    & $\sigma = 0.0$                                & \multicolumn{1}{c|}{-}         & -       & \multicolumn{1}{c|}{-}         & -       & \multicolumn{1}{c|}{75.69}     & 19.56   \\ \cline{2-8} 
                              & $\sigma = 0.250$                               & \multicolumn{1}{c|}{-}         & -       & \multicolumn{1}{c|}{-}         & -       & \multicolumn{1}{c|}{68.62}     & 56.15   \\ \hline
\end{tabular}
\caption{Testing accuracy of different deep models on clean and adversarial test examples (generated using the PGD attack algorithm) for all three data sets.}
\label{tab:appendix_Acc_clean_adv}
\end{table*}

\subsection{Case of Similar Noise Distribution for both Calibration  and Testing}
\noindent{\bf Performance evaluation with a fixed $s$ hyper-parameter and varying $\tilde{\alpha}$. }
We present in Figures \ref{C100_uni_cal_uni_eval_ratio_0.0PRCP_fixed_ns_cvg_Size_Both} and \ref{C10_uni_cal_uni_eval_ratio_0.0PRCP_fixed_ns_cvg_Size_Both}
the probabilistic robust coverage and prediction set size performance of aPRCP using \textit{the Uniform distribution as a noise distribution for both calibration and testing purposes} respectively for the CIFAR100 and CIFAR10 datasets with the three different models that are trained with clean data. Similarly, we present in Figures \ref{C100_gaussian_cal_gaussian_eval_ratio_0.0PRCP_fixed_ns_cvg_Size_Both} and \ref{C10_gaussian_cal_gaussian_eval_ratio_0.0PRCP_fixed_ns_cvg_Size_Both}
the probabilistic robust coverage and prediction set size performance of aPRCP using \textit{the Gaussian distribution as a noise distribution for both calibration and testing purposes}. For calibration, we sample $m_s = 128$ noisy data points from the surrounding of each data point ($||\epsilon||_2 \leq 0.125$). For testing, we sample $n_s = 128$ data points from the surrounding of each testing point ($||\epsilon||_2 \leq 0.125$). We observe that the probabilistic robust coverage for noisy data increases monotonically as we increase the quantile robust coverage for each ball from $1 - \tilde{\alpha} = 0.90$ to $1 - \tilde{\alpha} = 1.0$. These observations hold for both conformal scores (HPS and APS) and using different deep neural network models.

\begin{figure}[h!]
\centering
\includegraphics[width=.9\linewidth]{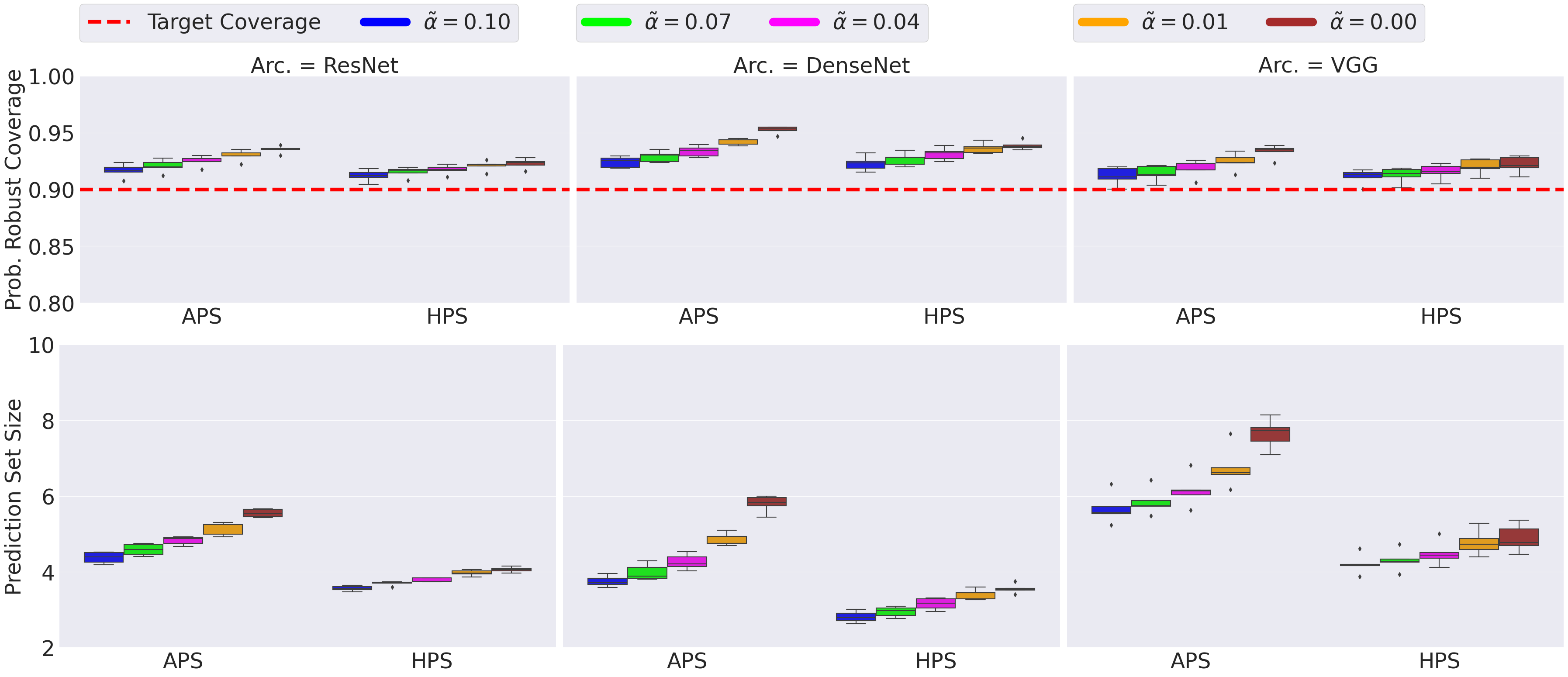}
\caption{Probabilistic robust coverage (top) and prediction set size (bottom) obtained by aPRCP$(\tilde{\alpha} = 0.10)$, aPRCP$(\tilde{\alpha} = 0.03)$, 
PRCP$(\tilde{\alpha} = 0.06)$, aPRCP$(\tilde{\alpha} = 0.09)$,
and aPRCP$(\tilde{\alpha} = 0.00)$, evaluated on CIFAR100 dataset for three different deep models. The target coverage is $90\%$. The results are shown over 50 different runs.}
\label{C100_uni_cal_uni_eval_ratio_0.0PRCP_fixed_ns_cvg_Size_Both}
\end{figure}

\begin{figure}[h!]
\centering
\includegraphics[width=.9\linewidth]{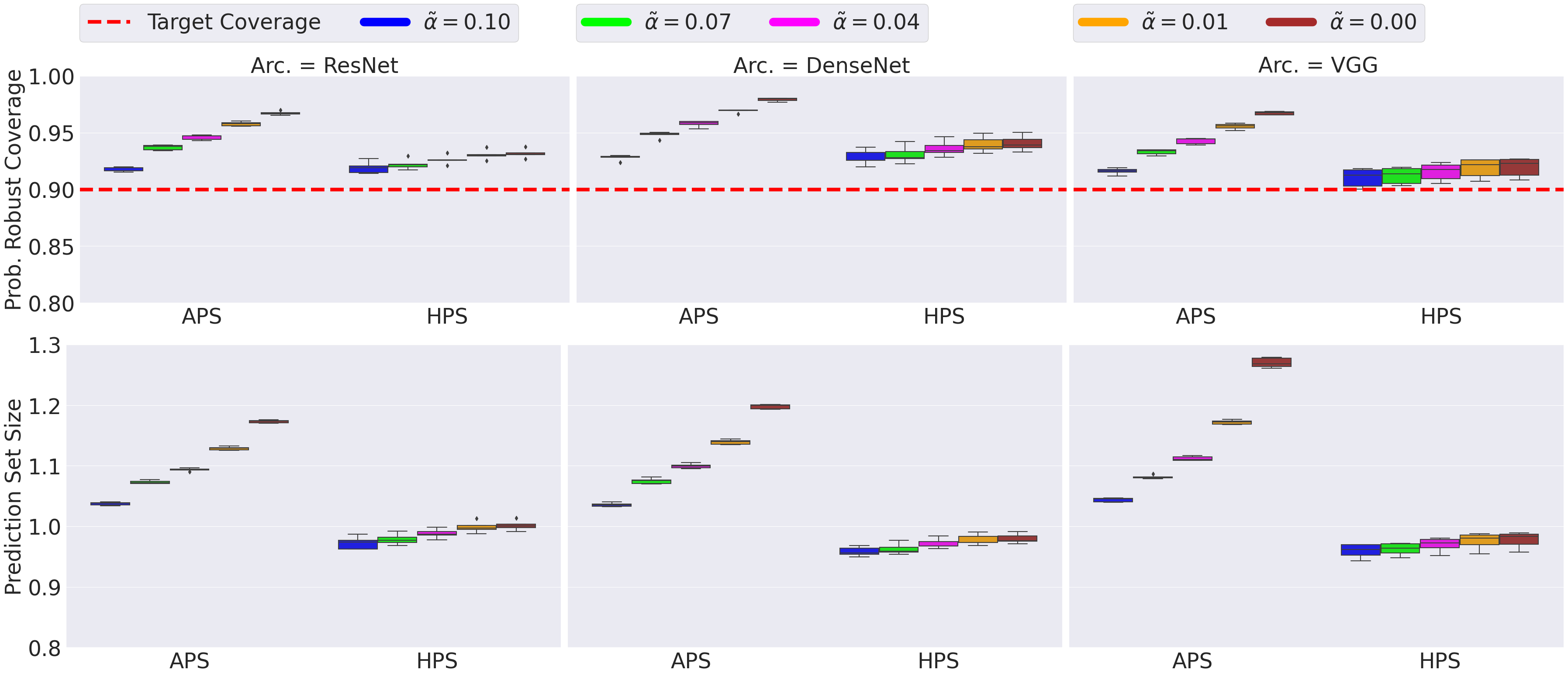}
\caption{Probabilistic robust coverage (top) and prediction set size (bottom) obtained by aPRCP$(\tilde{\alpha} = 0.10)$, aPRCP$(\tilde{\alpha} = 0.03)$, 
PRCP$(\tilde{\alpha} = 0.06)$, aPRCP$(\tilde{\alpha} = 0.09)$,
and aPRCP$(\tilde{\alpha} = 0.00)$, evaluated on CIFAR10 dataset for three different deep models. The target coverage is $90\%$. The results are shown over 50 different runs.}
\label{C10_uni_cal_uni_eval_ratio_0.0PRCP_fixed_ns_cvg_Size_Both}
\end{figure}

\begin{figure}[h!]
\centering
\includegraphics[width=.9\linewidth]{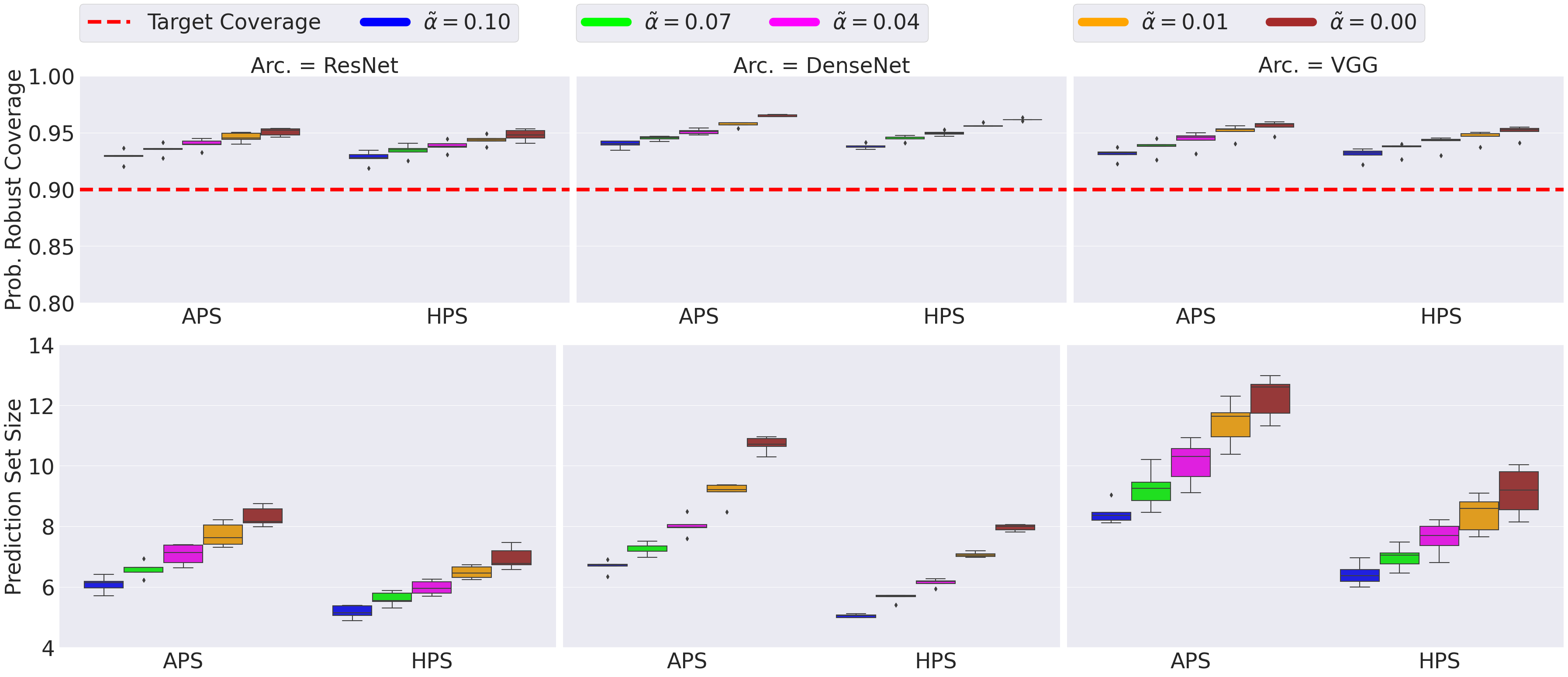}
\caption{Probabilistic robust coverage(top) and Prediction set size(bottom) obtained by aPRCP$(\tilde{\alpha} = 0.10)$, aPRCP$(\tilde{\alpha} = 0.03)$, 
PRCP$(\tilde{\alpha} = 0.06)$, aPRCP$(\tilde{\alpha} = 0.09)$,
and aPRCP$(\tilde{\alpha} = 0.00)$, evaluated on CIFAR100 dataset for three different deep models. The target coverage is $90\%$. The results are shown over 50 different runs.}
\label{C100_gaussian_cal_gaussian_eval_ratio_0.0PRCP_fixed_ns_cvg_Size_Both}
\end{figure}
\clearpage
\begin{figure}[h!]
\centering
\includegraphics[width=.9\linewidth]{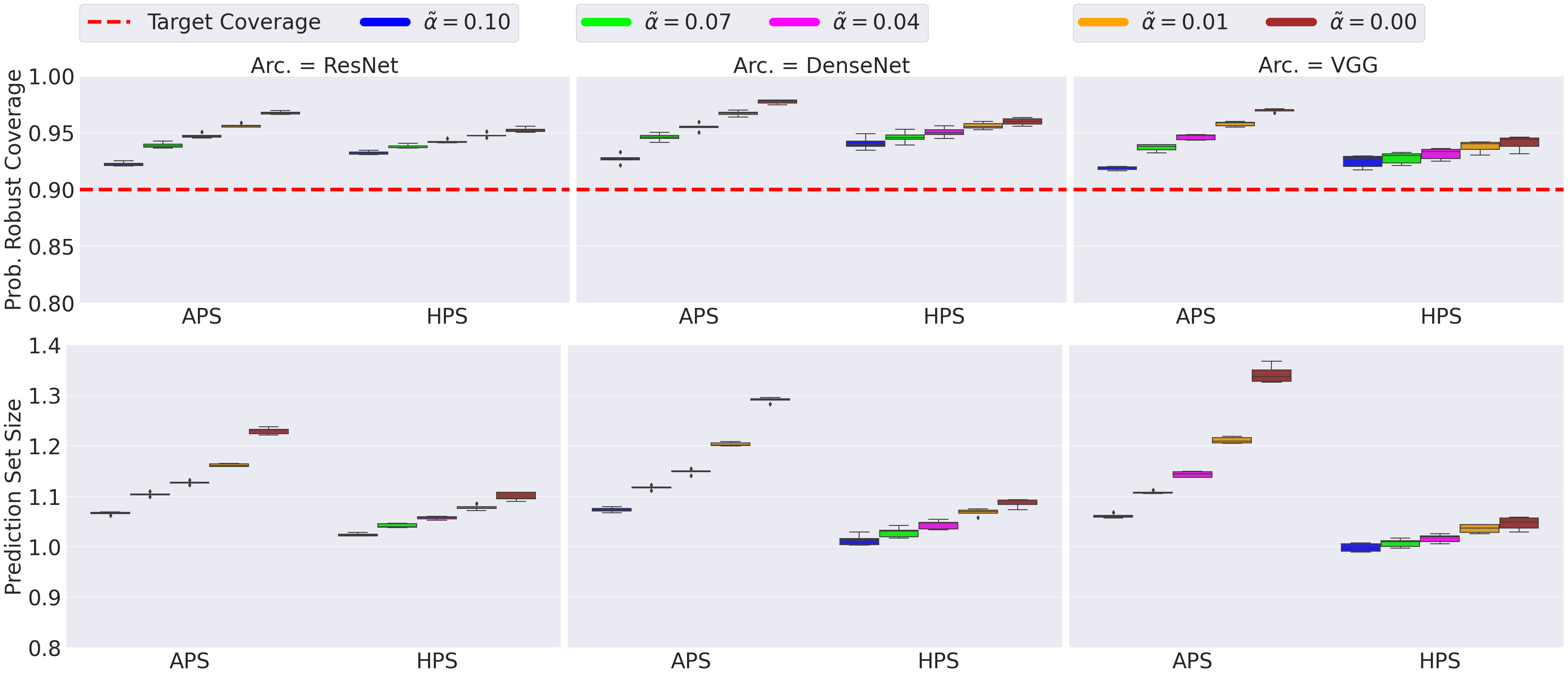}
\caption{Probabilistic robust coverage(top) and Prediction set size(bottom) obtained by aPRCP$(\tilde{\alpha} = 0.10)$, aPRCP$(\tilde{\alpha} = 0.03)$, 
PRCP$(\tilde{\alpha} = 0.06)$, aPRCP$(\tilde{\alpha} = 0.09)$,
and aPRCP$(\tilde{\alpha} = 0.00)$, evaluated on CIFAR10 dataset for three different deep models. The target coverage is $90\%$. The results are shown over 50 different runs.}
\label{C10_gaussian_cal_gaussian_eval_ratio_0.0PRCP_fixed_ns_cvg_Size_Both}
\end{figure}

\noindent{\bf Performance evaluation with a fixed $\tilde{\alpha}$ hyper-parameter and varying $s$.}

Figures \ref{C100_s_changes_ratio_0.0PRCP_fixed_ns_cvg_Size_Both} and \ref{C10_s_changes_ratio_0.0PRCP_fixed_ns_cvg_Size_Both}
show the probabilistic robust coverage and prediction set size respectively for the CIFAR100 and CIFAR10 datasets with three different deep models that are trained using standard training. For calibration, we sample $m_s = 128$ noisy data points using the uniform sampling distribution from the surrounding of each data point ($||\epsilon||_2 \leq 0.125$). For testing, we sample $n_s = 128$ data points uniformly from the surrounding of each testing point ($||\epsilon||_2 \leq 0.125$). We observe that the probabilistic robust coverage for noisy data increases as we increase the $s$ parameter value from $0.0$ to $0.09$. This observation matches our proposition as a higher $s$ value produces higher coverage. The above observations hold for both conformal scores (APS and HPS) using different deep neural network models. 

\begin{figure}[h!]
\centering
\includegraphics[width=.9\linewidth]{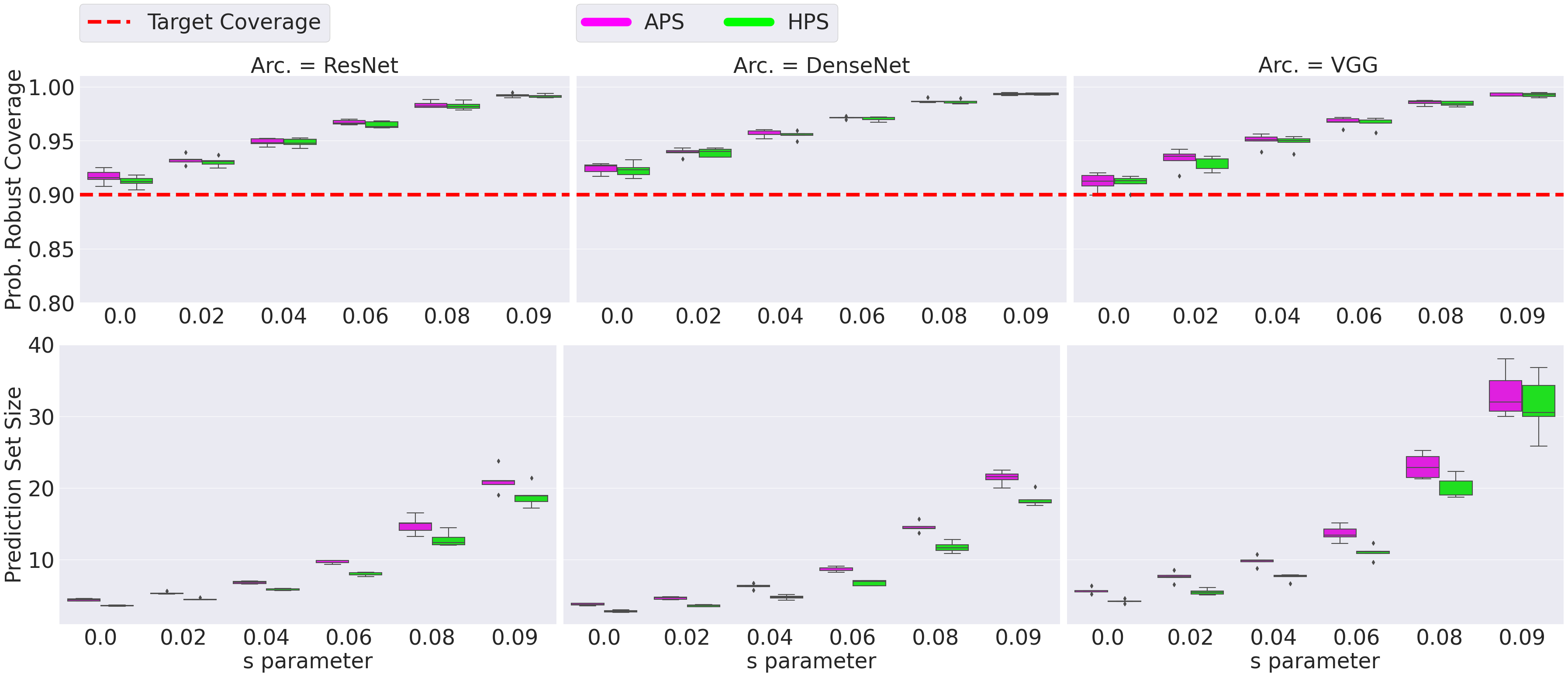}
\caption{Probabilistic robust coverage(top) and Prediction set size(bottom) obtained by aPRCP$(\tilde{\alpha} = 0.10)$ while varying the $s$ parameter, evaluated on CIFAR100 dataset for three different deep models. The target coverage is $90\%$. The results are shown over 50 different runs.}
\label{C100_s_changes_ratio_0.0PRCP_fixed_ns_cvg_Size_Both}
\end{figure}

\begin{figure}[h!]
\centering
\includegraphics[width=.9\linewidth]{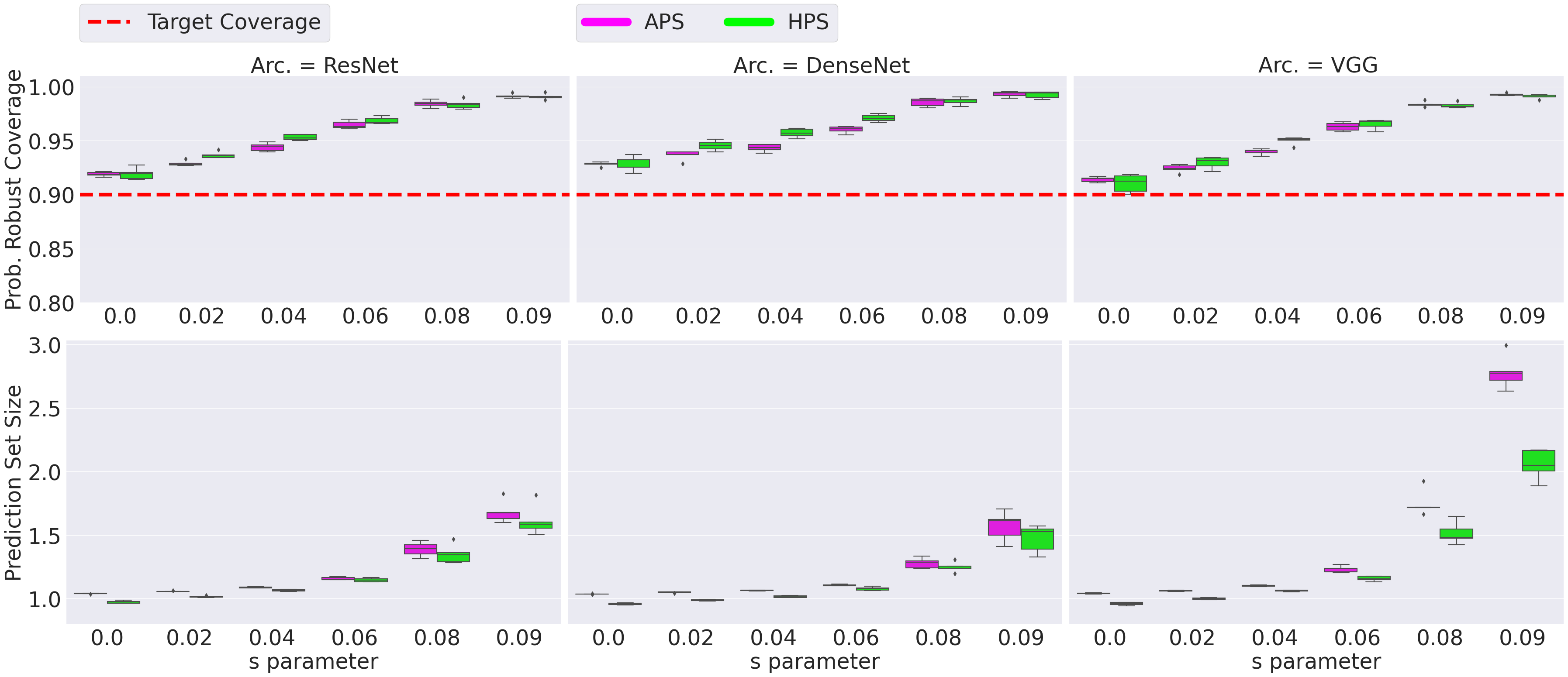}
\caption{Probabilistic robust coverage(top) and Prediction set size(bottom) obtained by aPRCP$(\tilde{\alpha} = 0.10)$ while varying the $s$ parameter, evaluated on CIFAR10 dataset for three different models. The target coverage is $90\%$. The results are shown over 50 runs forall three neural network models.}
\label{C10_s_changes_ratio_0.0PRCP_fixed_ns_cvg_Size_Both}
\end{figure}

\noindent{\bf Performance evaluation with fixed $s$ and $\tilde{\alpha}$ hyper-parameter and varying sampling radius ($||\epsilon||_2 \leq r$) around test samples.}
Figures \ref{C10_delta_changes_cvg} and \ref{C10_delta_changes_size}
present the probabilistic robust coverage and the prediction set size respectively for the CIFAR10 dataset. Similarly, figures \ref{C100_delta_changes_cvg} and \ref{C100_delta_changes_size}
present probabilistic robust coverage and prediction set size for the CIFAR100 dataset. We employ three different deep models that are trained with clean data. For calibration, we sample $m_s = 128$ noisy data points using the uniform sampling distribution from the surrounding of each data point ($||\epsilon||_2 \leq 0.125$), where $\epsilon$ is sampled uniformly over the segment $[0, 0.125]$. For testing, we sample $n_s = 128$ data points uniformly from the surrounding of each testing point ($||\epsilon||_2 \leq \{1.0, 2.0, 3.0\}$), where $\epsilon$ is uniformly sampled over the segment $[0, 1], [0, 2], [0, 3]$ respectively.  
We observe that the probabilistic robust coverage for noisy data decays as we increase the sampling radius. Additionally, we note that when we set the $d$ parameter to $0.1$ (accounting for the change in noise distribution between calibration and testing as per Theorem 2), we guarantee achieving the target coverage. These observations hold for both conformal scores (APS and HPS) using different deep neural network models.

\begin{figure}[h!]
\centering
\includegraphics[width=.9\linewidth]{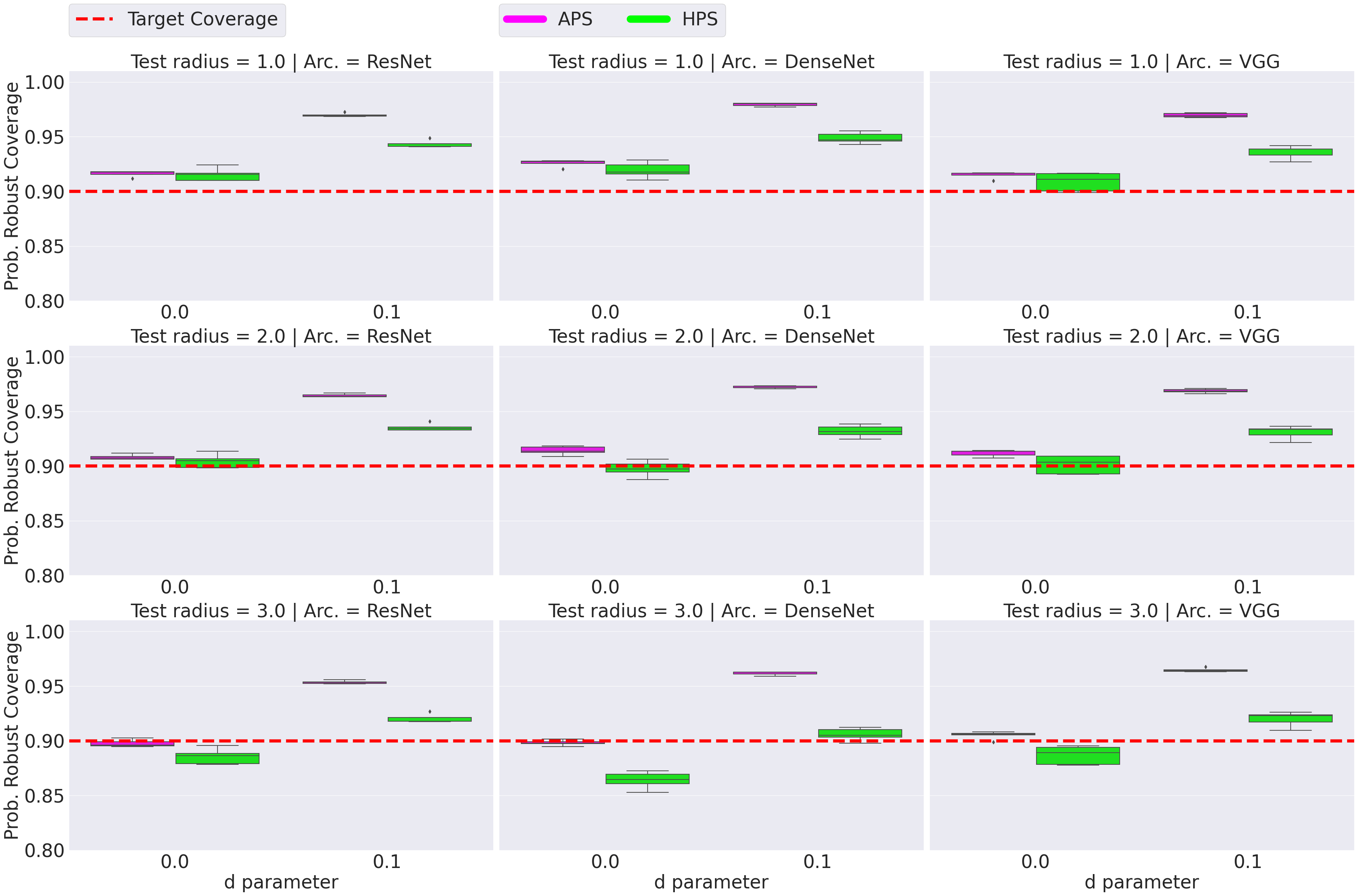}
\caption{Probabilistic robust coverage evaluated on CIFAR10 dataset for three different models. The target coverage is $90\%$. The results are shown over 50 runs forall three neural network models.}
\label{C10_delta_changes_cvg}
\end{figure}

\begin{figure}[h!]
\centering
\includegraphics[width=.9\linewidth]{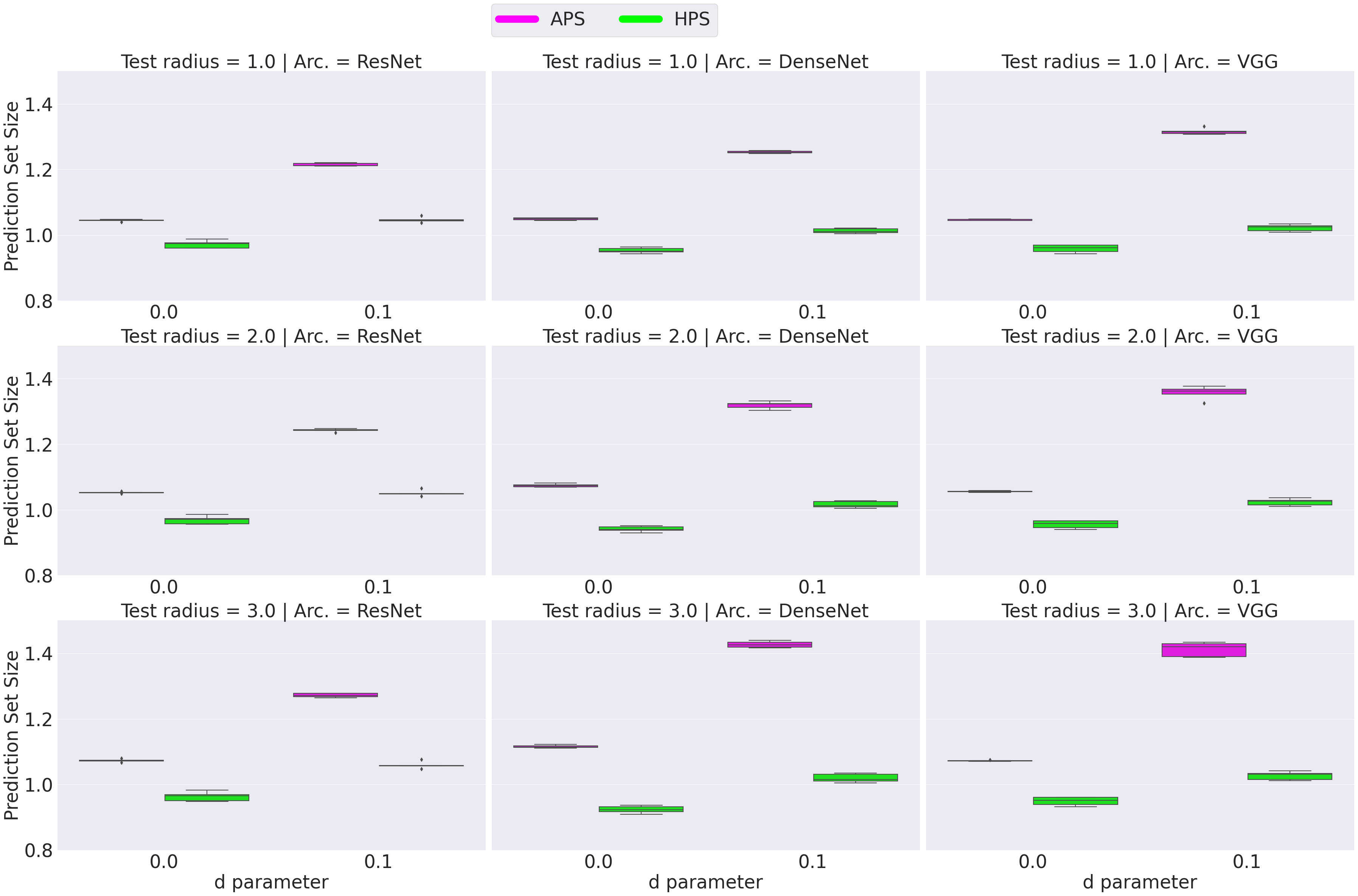}
\caption{Prediction set size evaluated on CIFAR10 dataset for three different deep models. The results are shown over 50 different runs.}
\label{C10_delta_changes_size}
\end{figure}

\begin{figure*}[h!]
\centering
\includegraphics[width=.9\linewidth]{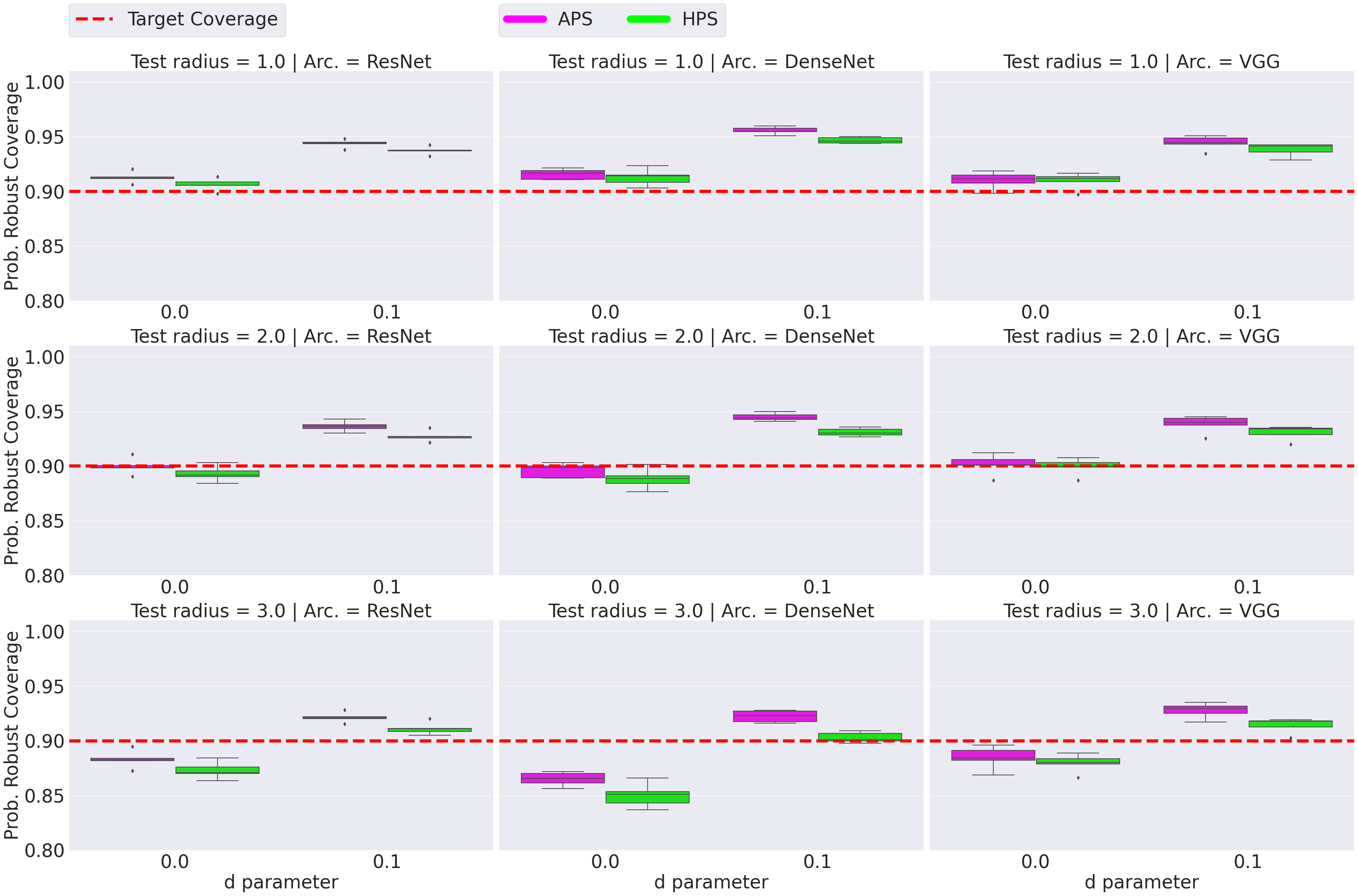}
\caption{Probabilistic robust coverage evaluated on CIFAR100 dataset for three different deep models. The target coverage is $90\%$. The results are shown over 50 different runs.}
\label{C100_delta_changes_cvg}
\end{figure*}

\clearpage

\begin{figure*}[h!]
\centering
\includegraphics[width=.9\linewidth]{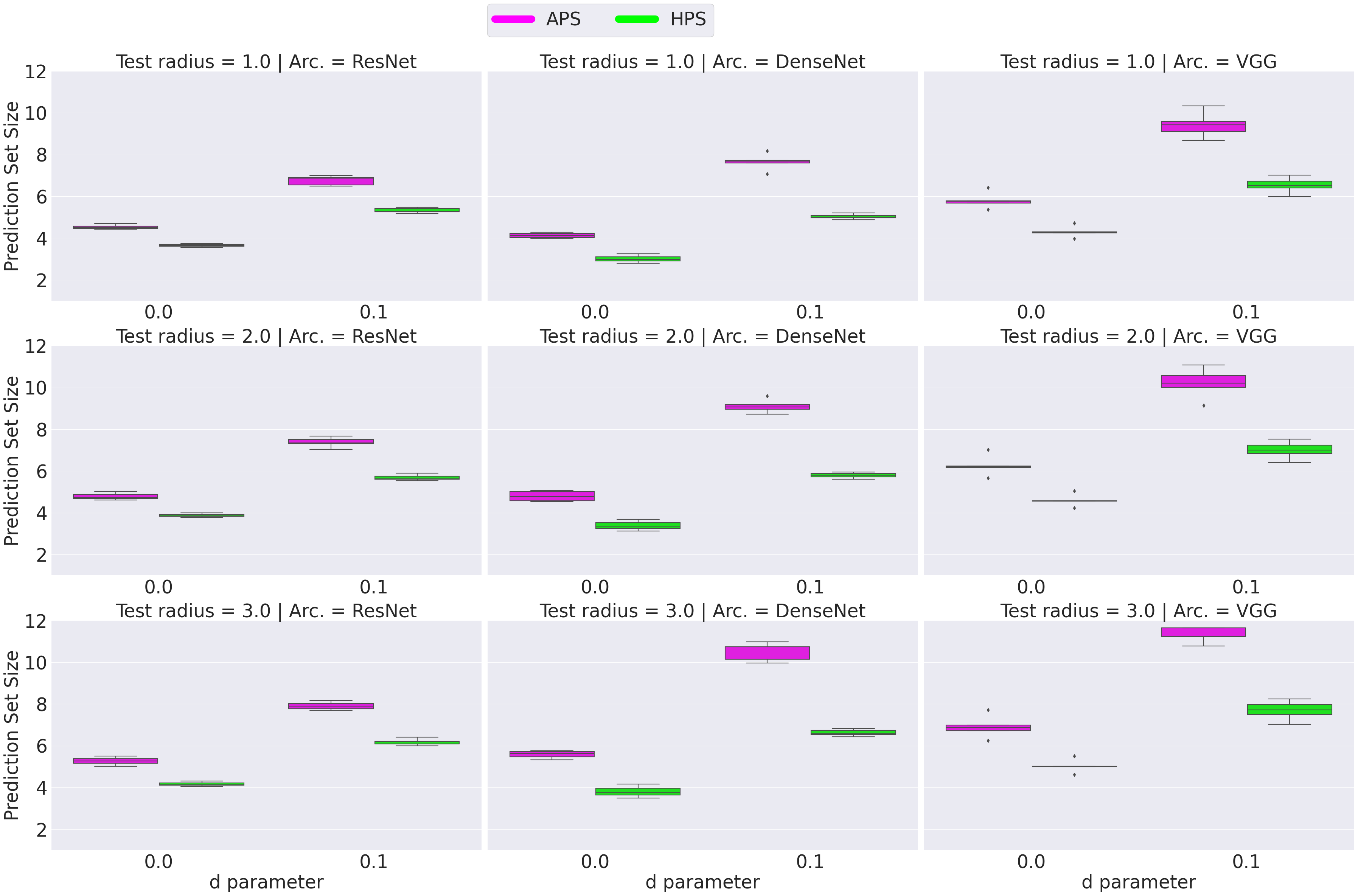}
\caption{Prediction set size evaluated on CIFAR100 dataset for three different deep models. The results are shown over 50 different runs.}
\label{C100_delta_changes_size}
\end{figure*}

\subsection{Case of Dissimilar Noise Distributions for Calibration and Testing}
\noindent{\bf Gaussian distribution for Calibration  and  Uniform distribution for Testing with a fixed $s$ hyper-parameter and varying $\tilde{\alpha}$.}
Figures \ref{C100_gaussian_cal_uni_eval_ratio_0.0PRCP_fixed_ns_cvg_Size_Both} and \ref{C10_gaussian_cal_uni_eval_ratio_0.0PRCP_fixed_ns_cvg_Size_Both} present probabilistic robust coverage and prediction set size respectively for the CIFAR100 and CIFAR10 datasets with  three different deep models that are trained with clean data. For calibration, we sample $m_s = 128$ data points using the Gaussian sampling distribution from the surrounding of each data point($||\epsilon||_2 \leq 0.125$). For testing, we sample $n_s = 128$ data points uniformly from the surrounding of each testing point($||\epsilon||_2 \leq 0.125$). We observe that the probabilistic robust coverage increased over the case of using the same distribution for sampling during the testing and calibration phases. 

\begin{figure}[h!]
\centering
\includegraphics[width=.9\linewidth]{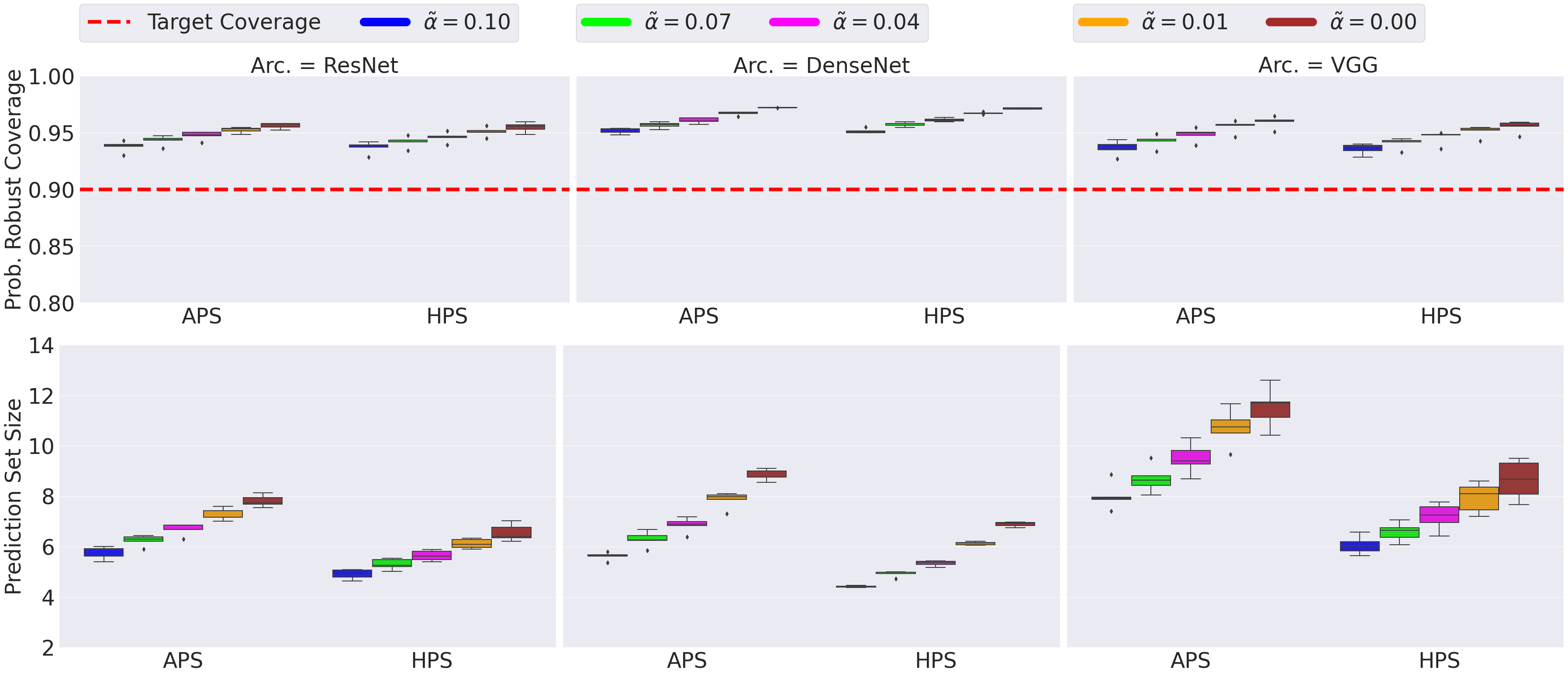}
\caption{Probabilistic robust coverage(top) and Prediction set size(bottom) obtained by aPRCP$(\tilde{\alpha} = 0.10)$, aPRCP$(\tilde{\alpha} = 0.03)$, 
PRCP$(\tilde{\alpha} = 0.06)$, aPRCP$(\tilde{\alpha} = 0.09)$,
and aPRCP$(\tilde{\alpha} = 0.00)$, evaluated on CIFAR100 dataset for three different deep models. The target coverage is $90\%$. The results are shown over 50 different runs.}
\label{C100_gaussian_cal_uni_eval_ratio_0.0PRCP_fixed_ns_cvg_Size_Both}
\end{figure}

\begin{figure}[h!]
\centering
\includegraphics[width=.9\linewidth]{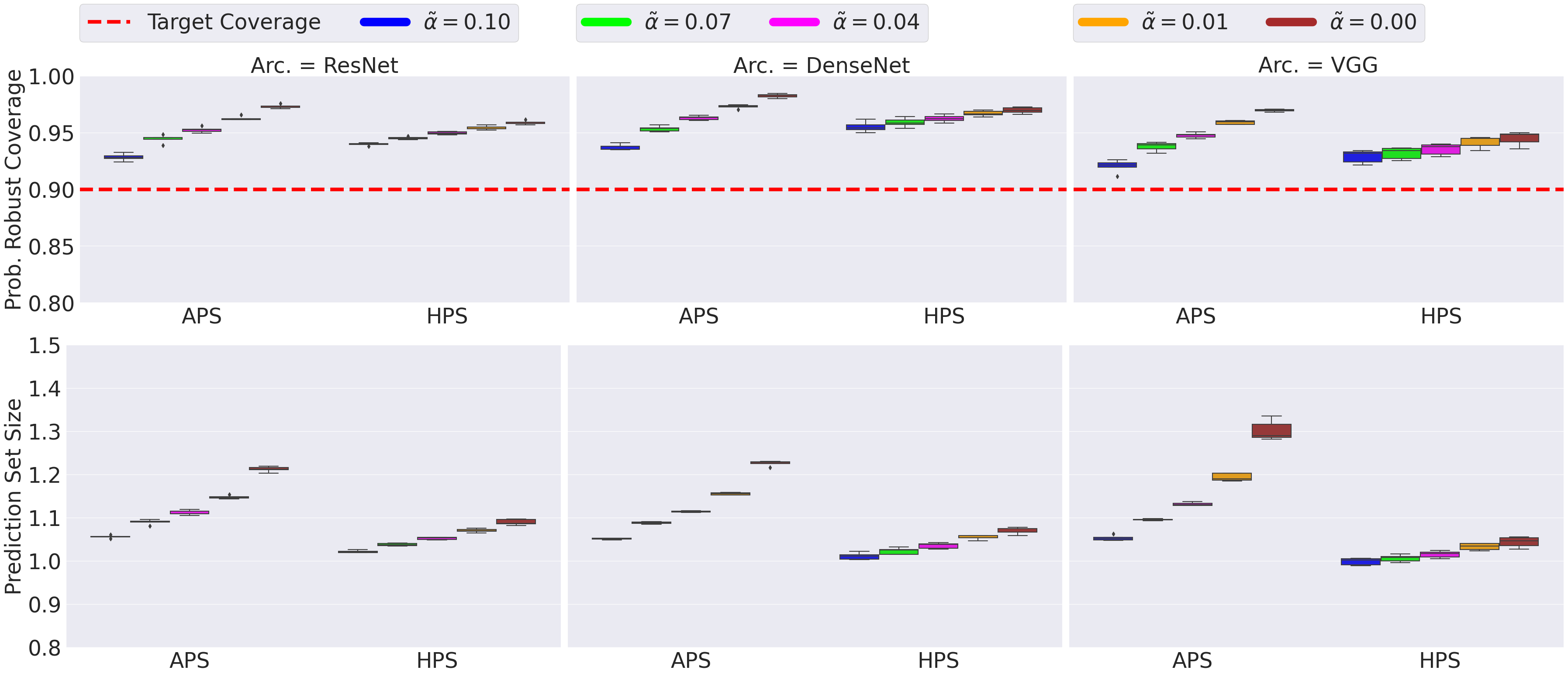}
\caption{Probabilistic robust coverage(top) and Prediction set size(bottom) obtained by aPRCP$(\tilde{\alpha} = 0.10)$, aPRCP$(\tilde{\alpha} = 0.03)$, 
PRCP$(\tilde{\alpha} = 0.06)$, aPRCP$(\tilde{\alpha} = 0.09)$,
and aPRCP$(\tilde{\alpha} = 0.00)$, evaluated on CIFAR10 dataset for three different deep models. The target coverage is $90\%$. The results are shown over 50 different runs.}
\label{C10_gaussian_cal_uni_eval_ratio_0.0PRCP_fixed_ns_cvg_Size_Both}
\end{figure}

\noindent{\bf Uniform distribution for Calibration and  Gaussian distribution for Testing with a fixed $s$ hyper-parameter and varying $\tilde{\alpha}$.}
Figures \ref{uni_cal_C100_fixed_ms_APS_HPS_sigma_0_0_cvg} and \ref{uni_cal_C100_fixed_ms_APS_HPS_sigma_0_0_size} present probabilistic robust coverage and prediction size for CIFAR100 and CIFAR10 datasets respectively with three different deep models that are trained with clean data. For calibration, we sample $m_s = 128$ data points using the Uniform sampling distribution from the surrounding of each data point ($||\epsilon||_2 \leq 0.125$). For testing, we sample $n_s = 128$ data points using Gaussian distribution from the surrounding of each testing point ($||\epsilon||_2 \leq 0.125$). We observe a slightly different performance of aPRCP compared to the case of using the same distribution for noise during the testing and calibration phases. This observation corroborate the statement of Theorem 2 and Remark 2 explaining the relation between the gap of the density probability between the calibration and testing noise distributions with the probabilistic robust coverage for aPRCP.

\begin{figure}[h!]
\centering
\includegraphics[width=.9\linewidth]{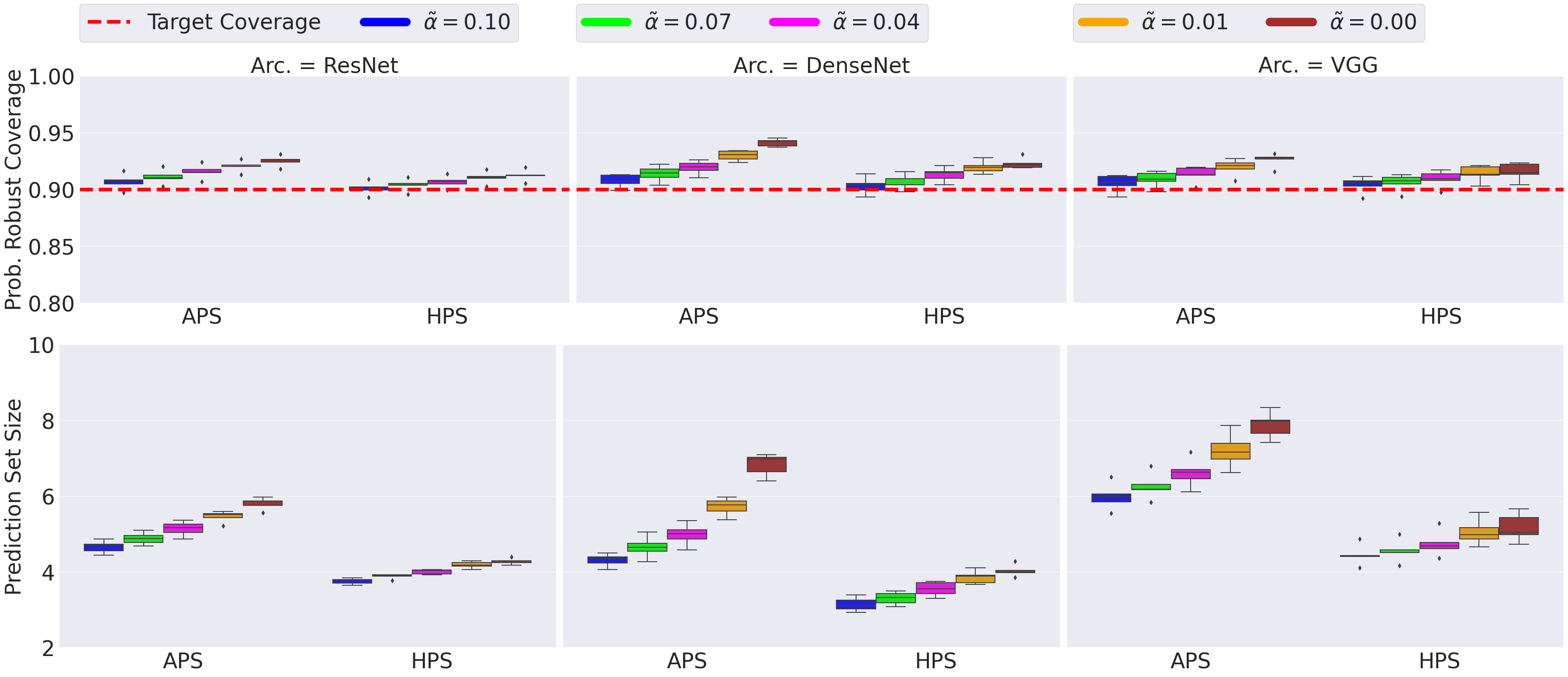}
\caption{Probabilistic robust coverage(top) and Prediction set size(bottom) obtained by aPRCP$(\tilde{\alpha} = 0.10)$, aPRCP$(\tilde{\alpha} = 0.03)$, 
PRCP$(\tilde{\alpha} = 0.06)$, aPRCP$(\tilde{\alpha} = 0.09)$,
and aPRCP$(\tilde{\alpha} = 0.00)$, evaluated on CIFAR100 dataset for three different deep models. The target coverage is $90\%$. The results are shown over 50 different runs.}
\label{uni_cal_C100_fixed_ms_APS_HPS_sigma_0_0_cvg}
\end{figure}

\begin{figure}[!h]
\centering
\includegraphics[width=.9\linewidth]{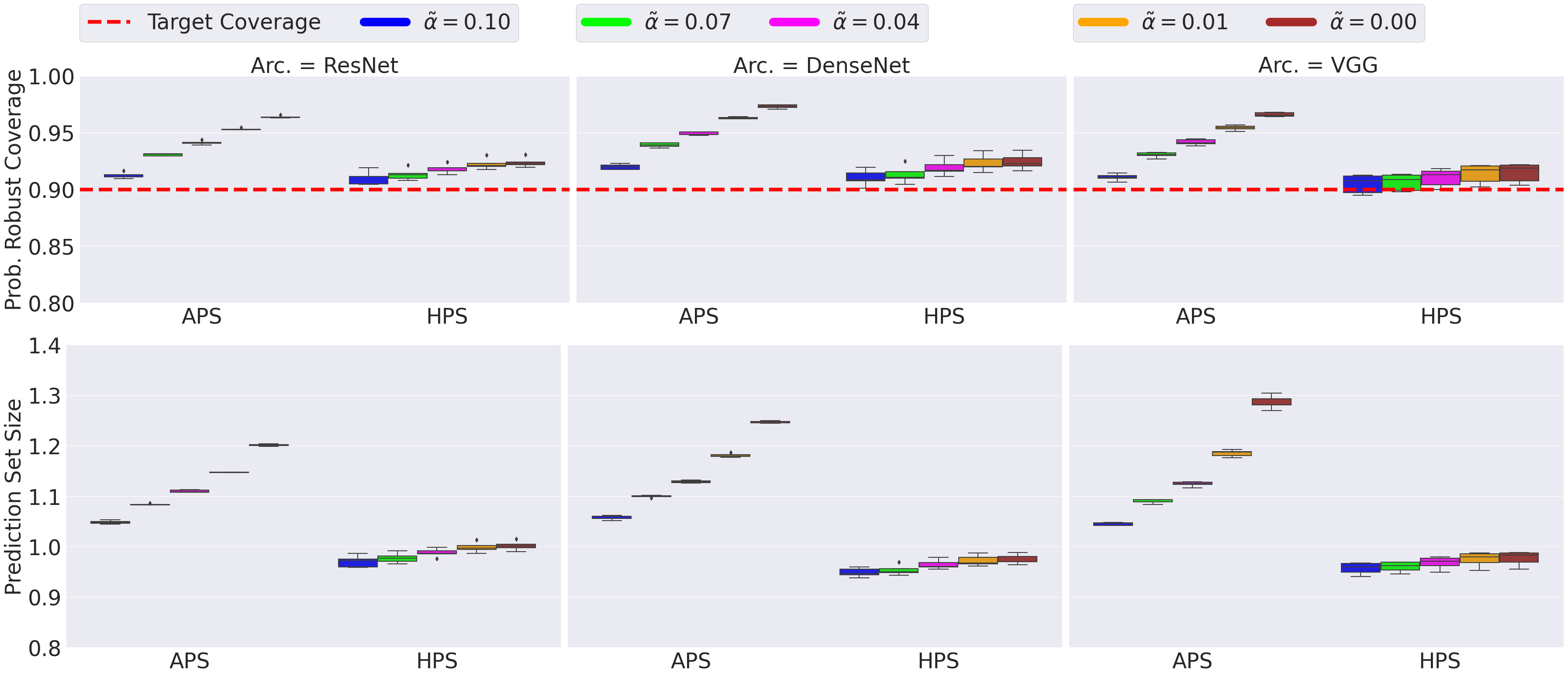}
\caption{Probabilistic robust coverage(top) and Prediction set size(bottom) obtained by aPRCP$(\tilde{\alpha} = 0.10)$, aPRCP$(\tilde{\alpha} = 0.03)$, 
PRCP$(\tilde{\alpha} = 0.06)$, aPRCP$(\tilde{\alpha} = 0.09)$,
and aPRCP$(\tilde{\alpha} = 0.00)$, evaluated on CIFAR10 dataset for three different deep models. The target coverage is $90\%$. The results are shown over 50 different runs.}
\label{uni_cal_C100_fixed_ms_APS_HPS_sigma_0_0_size}
\end{figure}

\clearpage

\subsection{Performance of \texttt{aPRCP(worst-adv)} with varying $m_s$}
Figures \ref{C10_ms_APS_HPS_sigma_0.25} and \ref{C100_ms_APS_HPS_sigma_0.25} show the performance of aPRCP with three different deep models when varying $m_s$ (number of noisy samples for calibration) for CIFAR10 and CIFAR100 datasets respectively. We show the robust coverage and prediction set size for both \texttt{APS} and \texttt{HPS} conformity scores. Both figures show that the \texttt{aPRCP(worst-adv)} reported performance is consistent for different values of $m_s$.

We show in Figure \ref{C100_ms_APS_HPS_ARCPworstadv_RSCP} the comparison of the prediction set size and the coverage between \texttt{RSCP} and \texttt{aPRCP(worst-adv)} using both \texttt{APS} and \texttt{HPS}. We employ ResNet110 model trained with Gaussian augmented data ($\sigma = 0.125$). We observe that \texttt{RSCP} is more conservative compared to our method \texttt{aPRCP(worst-adv)} for both \texttt{APS} and \texttt{HPS} conformity scores.

We show in Figure \ref{C10_ms_APS_ARCPworstadv_RSCP_different_sigma} and \ref{C100_ms_APS_ARCPworstadv_RSCP_different_sigma} the comparison of the  prediction set size and coverage between \texttt{RSCP} and \texttt{aPRCP(worst-adv)} for two different deep models trained with Gaussian augmented data ($\sigma = 0.0625$ and $\sigma = 0.125$). We observe that \texttt{aPRCP(worst-adv)} produces smaller prediction sets than \texttt{RSCP}.

\begin{figure}[h!]
\centering
\includegraphics[width=\linewidth]{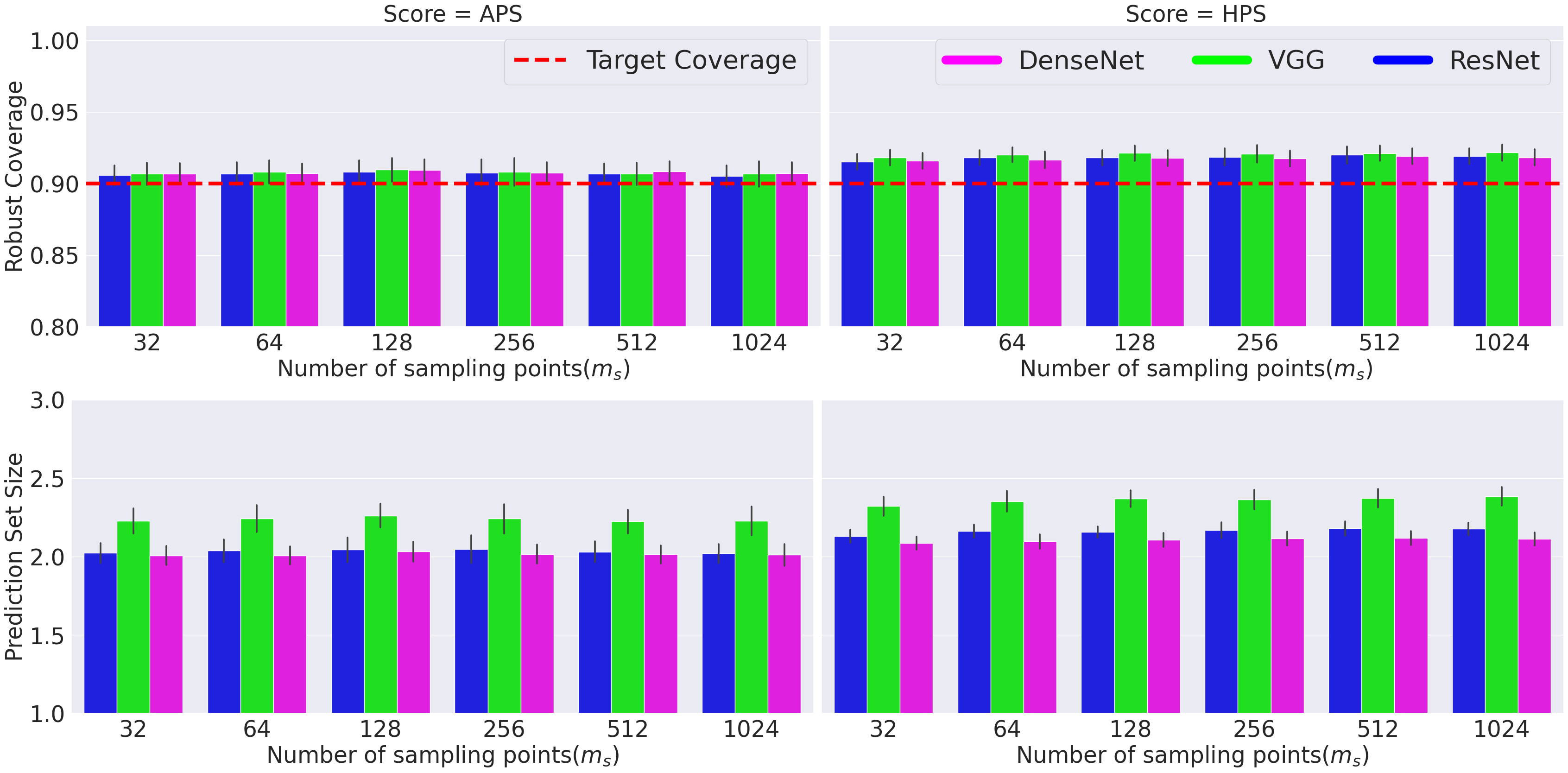}
\caption{Robust coverage (top) and prediction set size (bottom) performance of two conformity scores (APS and HPS) for different deep models with varying $m_s$ samples on calibration data for CIFAR10 dataset. The results are reported over 50 different runs. We use all models trained with Gaussian augmented data using standard deviation $\sigma = 0.25$.}
\label{C10_ms_APS_HPS_sigma_0.25}
\end{figure}

\begin{figure}[h!]
\centering
\includegraphics[width=\linewidth]{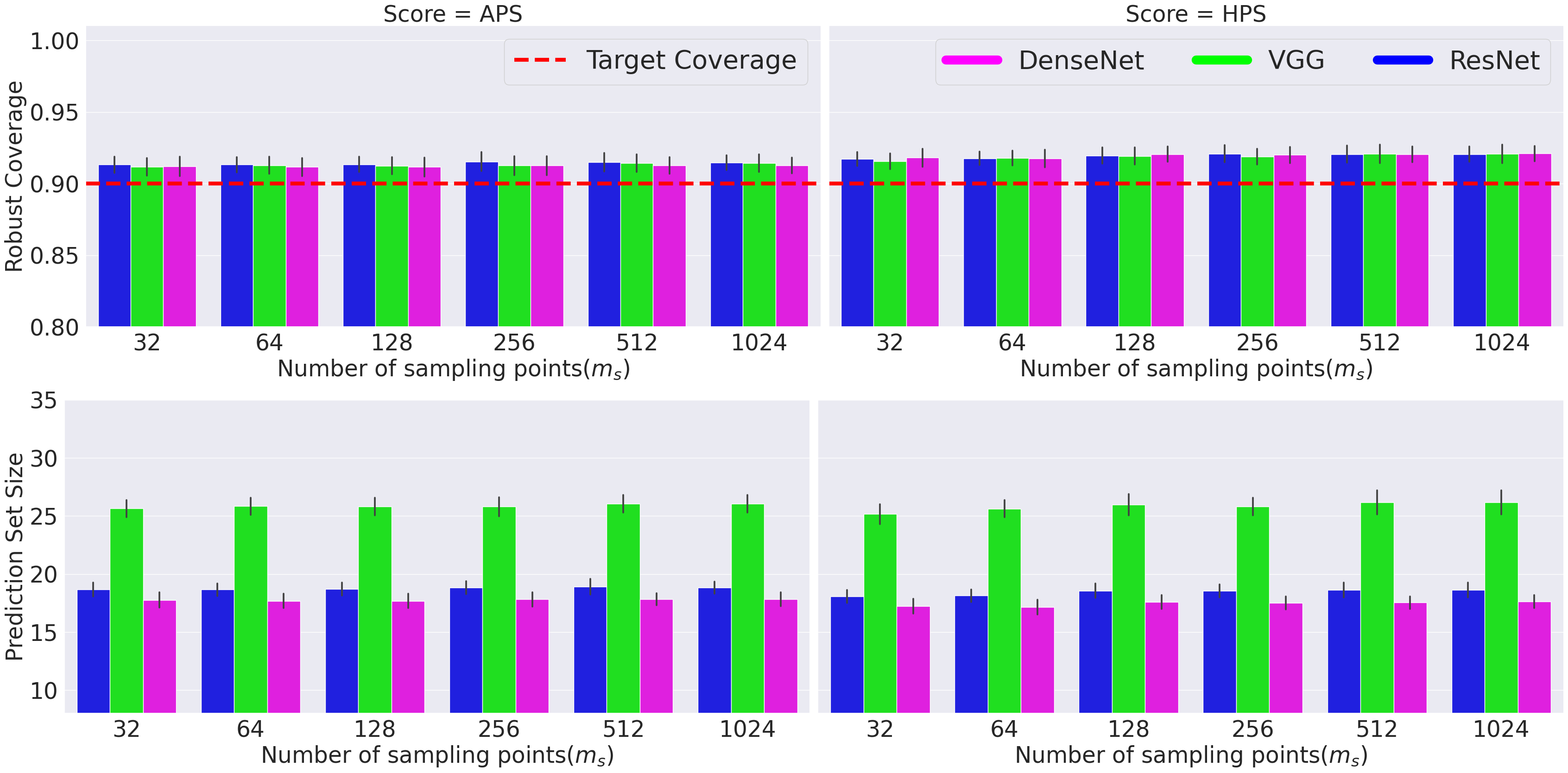}
\caption{Robust coverage (top) and prediction set size (bottom) performance of two scores for different deep models with varying $m_s$ samples on calibration data for CIFAR100 dataset. The results are reported over 50 different runs. We use all models trained with Gaussian augmented data with standard deviation $\sigma = 0.25$.}
\label{C100_ms_APS_HPS_sigma_0.25}
\end{figure}

\begin{figure}[h!]
\centering
\includegraphics[width=\linewidth]{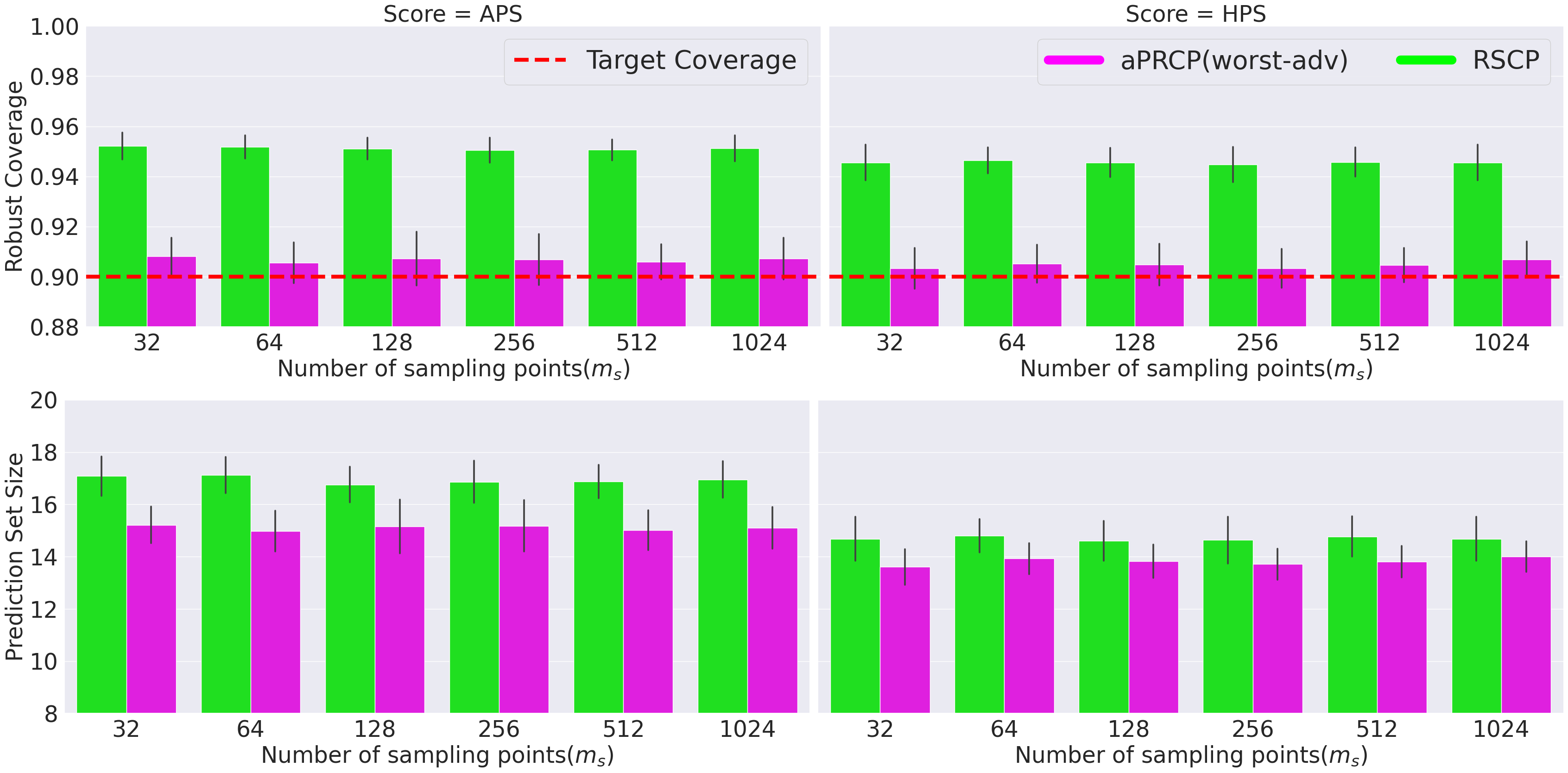}
\caption{Robust coverage (top) and prediction set size (bottom) performance of two methods, namely, aPRCP(worst-adv) and RSCP, with varying $m_s$ samples on calibration data for CIFAR100 dataset. The results are reported over 50 different runs. We use all models trained with Gaussian augmented data of standard deviation $\sigma = 0.125$.}
\label{C100_ms_APS_HPS_ARCPworstadv_RSCP}
\end{figure}

\begin{figure}[h!]
\centering
\includegraphics[width=\linewidth]{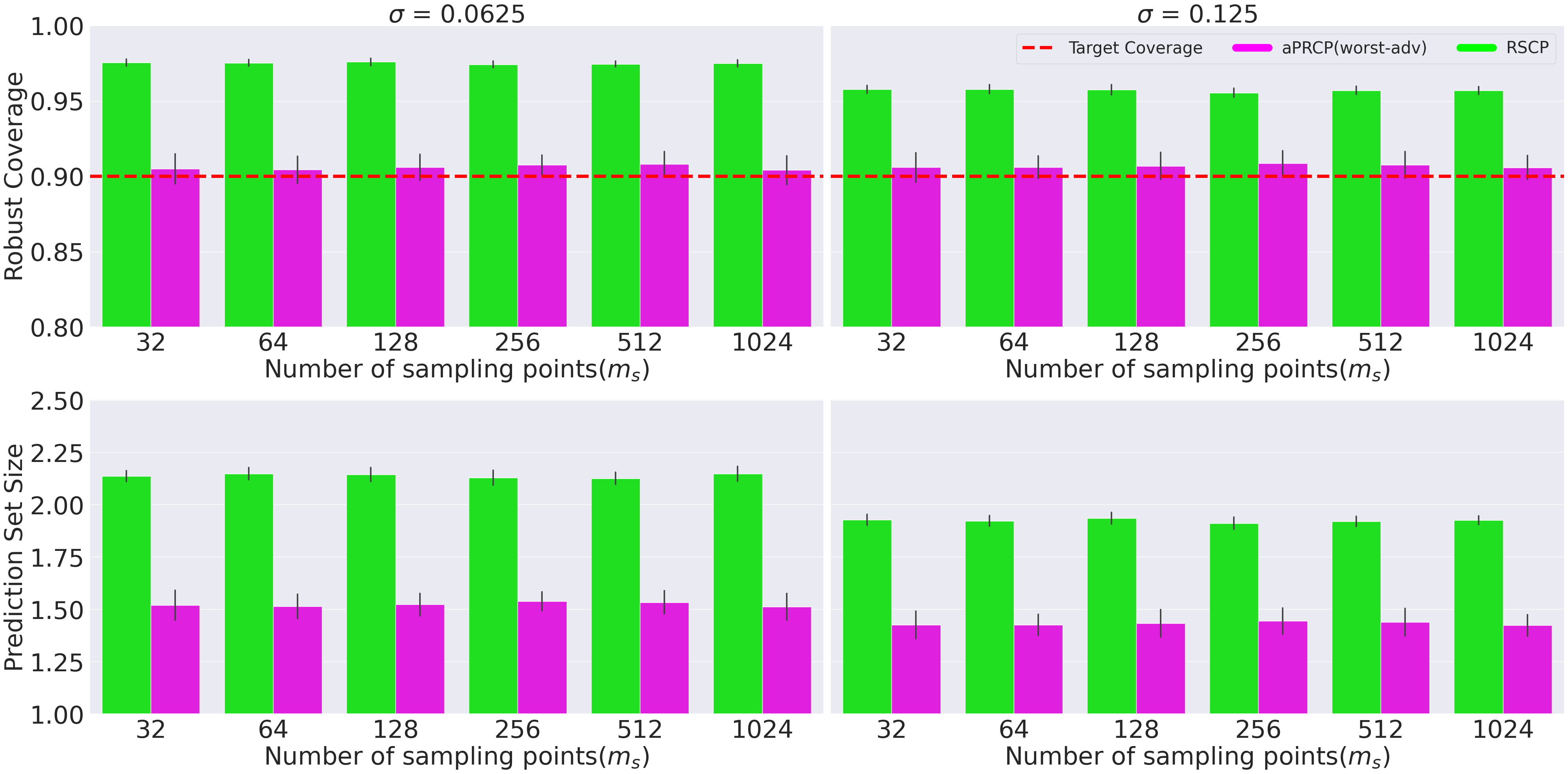}
\caption{Robust coverage (top) and prediction set size (bottom) performance of two different models trained with Gaussian augmented data using standard deviation $\sigma = 0.0625$ and $\sigma = 0.125$ with varying $m_s$ samples on calibration data for CIFAR10 dataset. The results are reported over 50 different runs.}
\label{C10_ms_APS_ARCPworstadv_RSCP_different_sigma}
\end{figure}

\begin{figure}[h!]
\centering
\includegraphics[width=\linewidth]{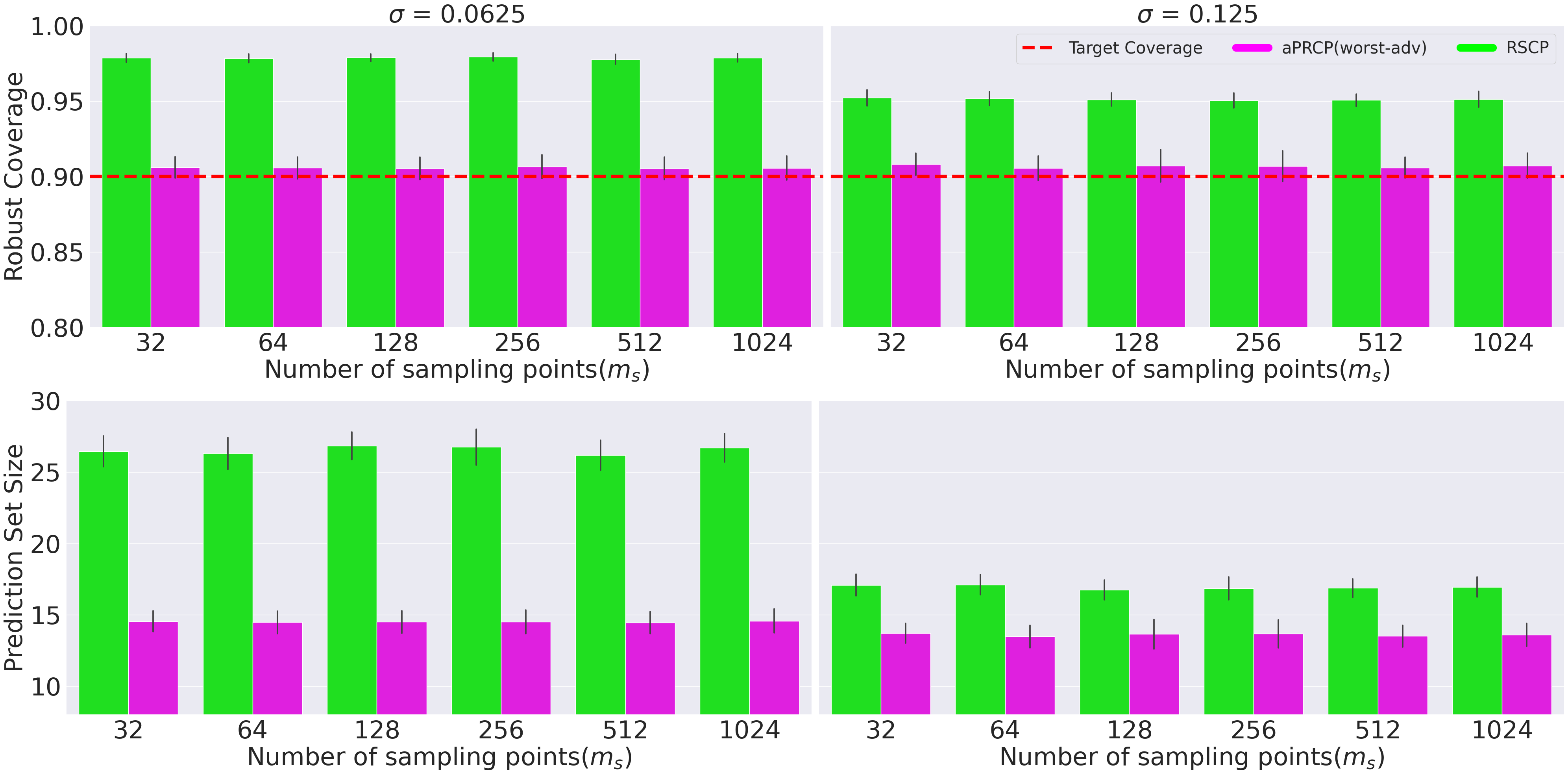}
\caption{Robust coverage (top) and prediction set size (bottom) performance of two different models trained with Gaussian augmented data using standard deviation $\sigma = 0.0625$ and $\sigma = 0.125$ with varying $m_s$ samples on calibration data for CIFAR100 dataset. The results are reported over 50 different runs.}
\label{C100_ms_APS_ARCPworstadv_RSCP_different_sigma}
\end{figure}

\clearpage
\subsection{The effect of Varying $||\epsilon||_2 \leq r$ during calibration}
We show in Figure \ref{fig:C100_radius_changes_worst_adv} 
the robust coverage and the prediction set size achieved by aPRCP(worst-adv) on CIFAR100 with a ResNet model that is trained with Gaussian augmented data ($\sigma = 0.125$). For calibration, we sample $m_s = 128$ noisy data points using the uniform sampling distribution from the surrounding of each data point ($||\epsilon||_2 \leq r$), where $r = \{0.125, 0.250, 1.0\}$. For testing, we generate data using an adversarial attack algorithm of energy $0.125$. We observe that the effect of the small changes in the sampling radius is negligible.
\begin{figure*}[!h]
    \centering
    \begin{minipage}{\linewidth}   
        \hfill
        \begin{minipage}{\linewidth}
        \centering
            \includegraphics[width=0.6\linewidth]{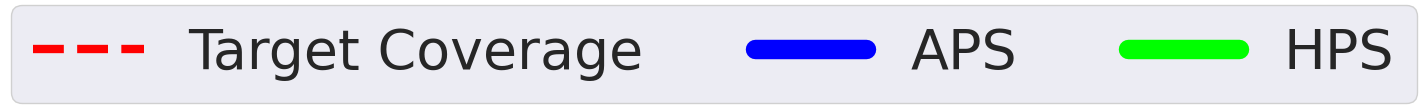}
        \end{minipage}
        \hfill       
        \begin{minipage}{\linewidth}
            \includegraphics[width=\linewidth]{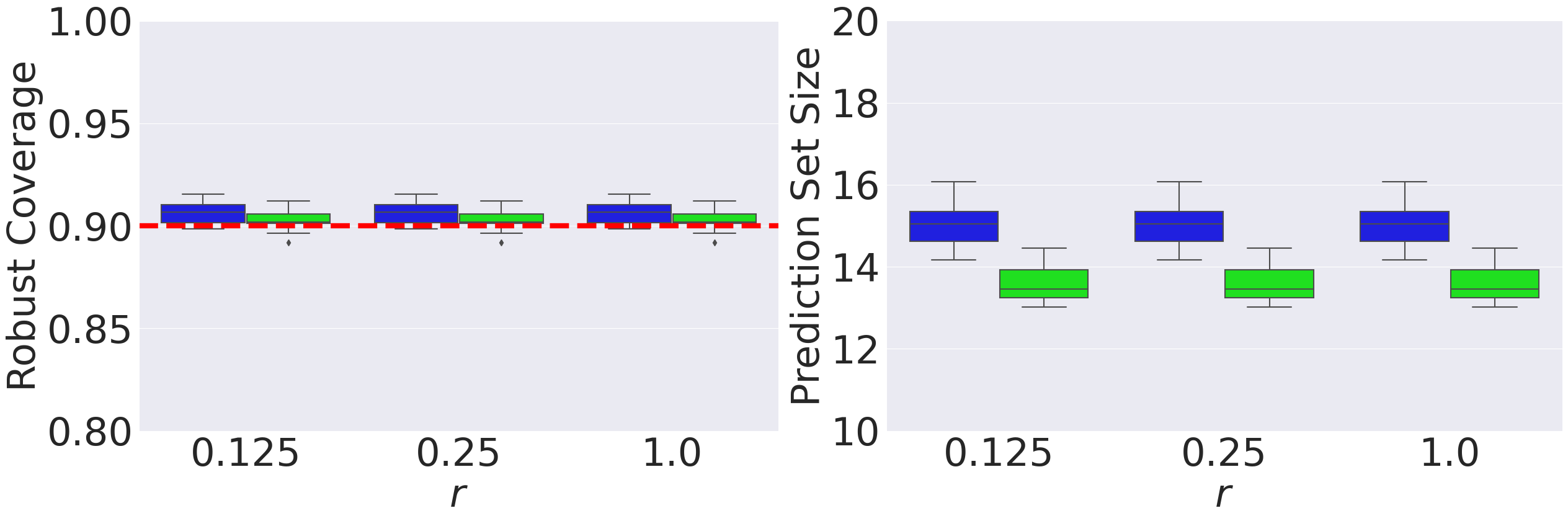}
        \end{minipage}
        \hfill
    \end{minipage}
    \caption{Robust coverage (left) and prediction set size (bottom) performance of the ResNet model trained with Gaussian augmented data using standard deviation $\sigma = 0.125$ with varying radius of robust quantile balls during calibration for CIFAR100 dataset. The results are reported over 50 different runs.}
    \label{fig:C100_radius_changes_worst_adv}
\end{figure*}

\subsection{Performance of aPRCP(worst-adv) with different Deep models}
Figure \ref{arcp_model_changes_C10_C100} shows the performance of our \texttt{aPRCP}(worst-adv) using DenseNet\citep{iandola2014densenet} and VGG\citep{simonyan2014very} models on the CIFAR10 and CIFAR100 datasets. We use the same adversarial attack algorithm for test examples with a magnitude of $r = 0.125$. During calibration, we sample $m_s = 128$ noisy samples ($r = 0.125$) for each calibration example. We observe that the robust coverage is achieved on all three deep models with small prediction sets.

\begin{figure*}[!h]
    \centering
    \begin{minipage}{\linewidth}
    \hfill
        \begin{minipage}{\linewidth}
        \centering
            \includegraphics[width=0.6\linewidth]{Figures1/legend_clean.png}
        \end{minipage}
        \begin{minipage}{.48\linewidth}
            \centering
            (a) CIFAR10
        \end{minipage}
        \hfill
        \begin{minipage}{.48\linewidth}
            \centering
            (b) CIFAR100
        \end{minipage} 
        \begin{minipage}{.48\linewidth}
            \includegraphics[width=\linewidth]{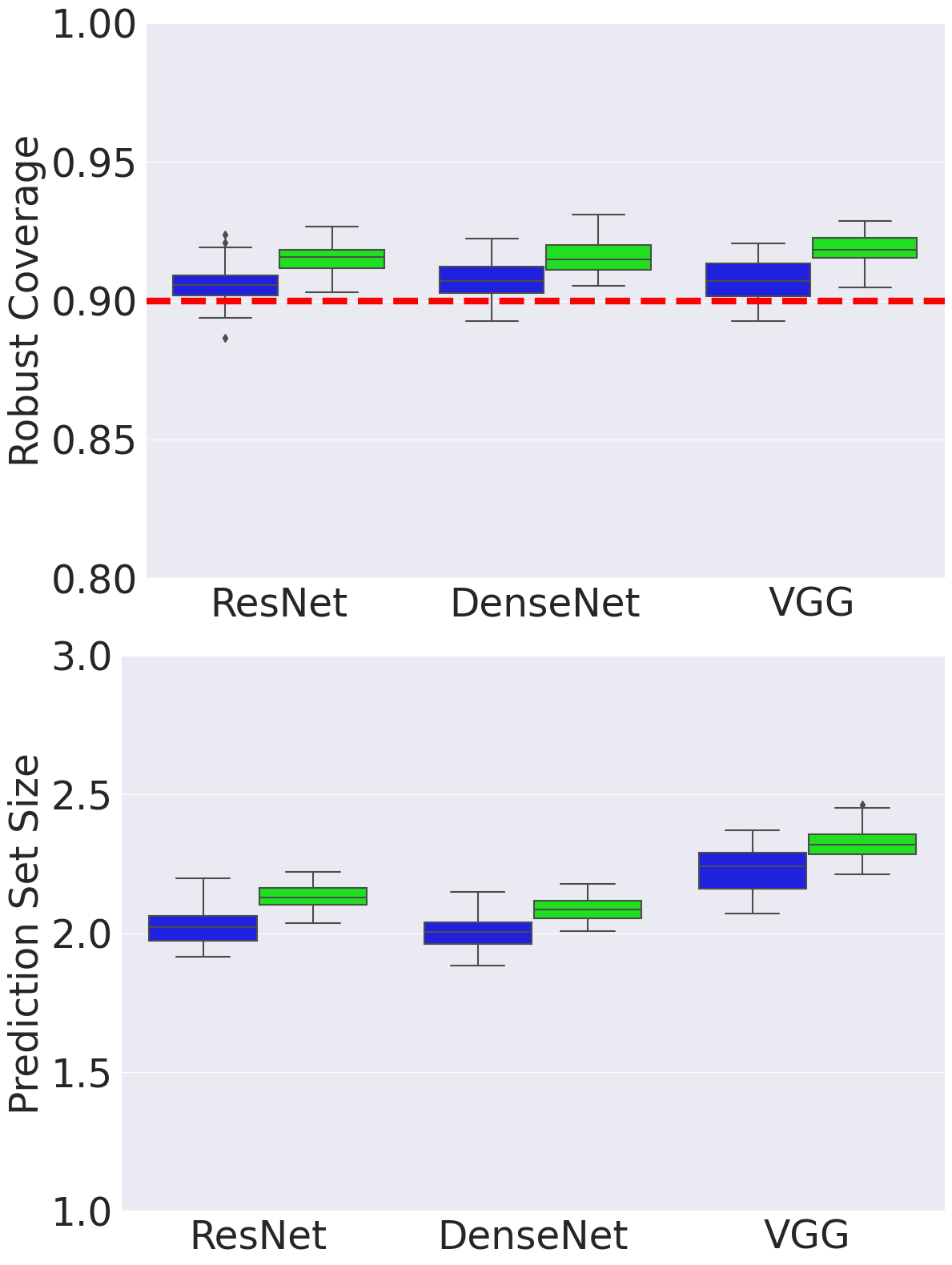}
        \end{minipage}
        \hfill       
        \begin{minipage}{.48\linewidth}
            \includegraphics[width=\linewidth]{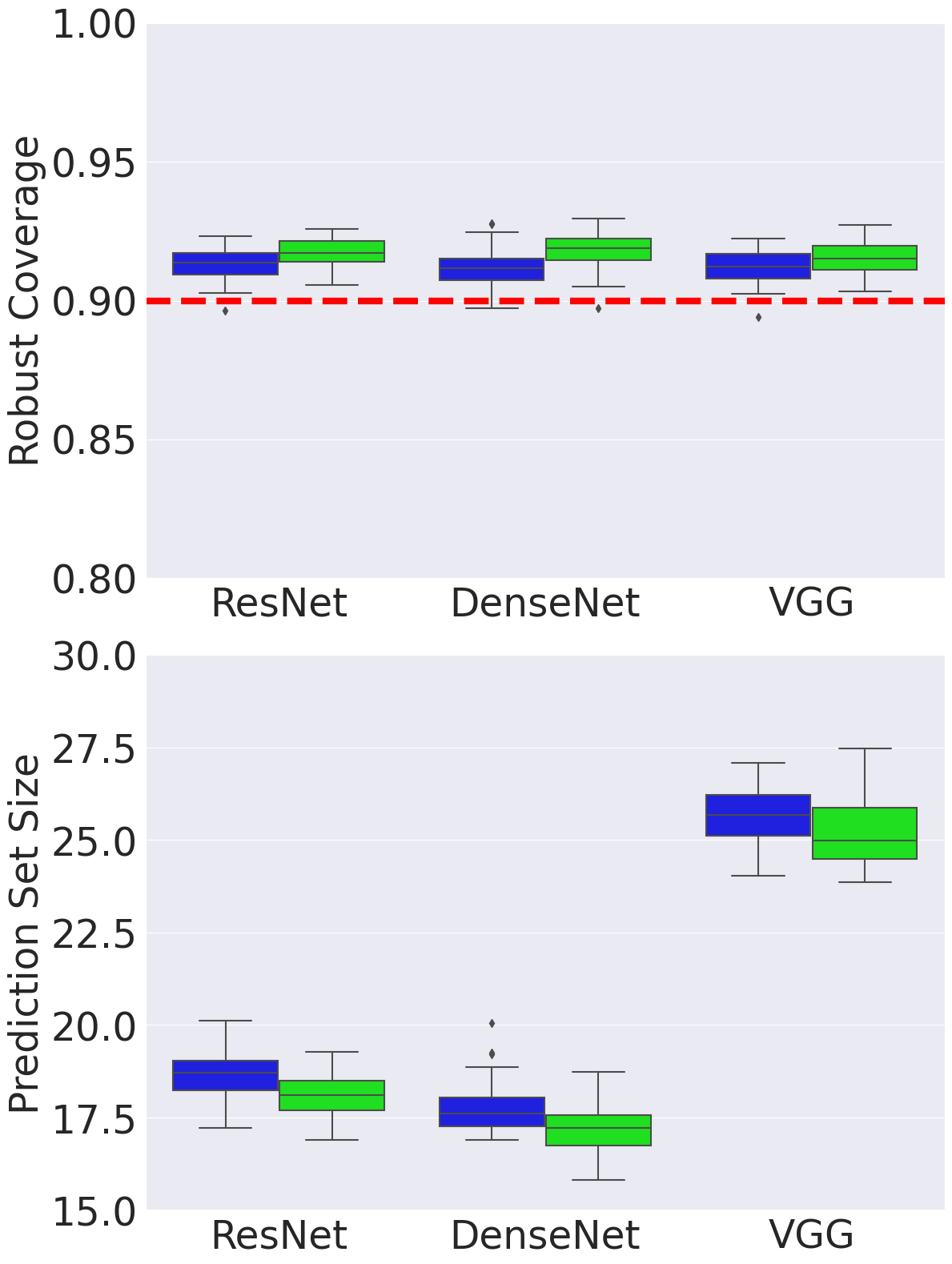}
        \end{minipage}
        \hfill

    \end{minipage}
    \caption{Robust coverage (top) and prediction set size (bottom) constructed by \texttt{aPRCP(worst-adv)} method for CIFAR10 (left) and CIFAR100 (right) datasets. The neural network models used are trained with Gaussian augmented data using standard deviation $\sigma = 0.25$. The results are reported  over 50 different runs. As can be seen, all models guarantee target coverage and VGG produces larger prediction sizes compared to other models.}
    \label{arcp_model_changes_C10_C100}
\end{figure*}

\clearpage 

\subsection{Results on Adversarial Examples Generated from a probability density distribution}
We evaluate the performance of aPRCP with a  different adversarial attack algorithm, namely NATTACK \citep{Black_box}. This attack algorithm generates a probability density distribution centered around an input from which adversarial examples can be sampled.
We employ this algorithm using an adversarial magnitude $||\epsilon||_2 \leq r = 0.125$ to generate adversarial examples for the test data of CIFAR10 and CIFAR100 on three different deep models trained with Gaussian augmented data ($\sigma = 0.125$). In all our experiments, we set $T = 1000$ as the number of maximum iterations, and a learning rate $\eta = 0.008$.

Both Figures \ref{All_blackBox_C10} and \ref{All_blackBox_C100} show that aPRCP is the only algorithm that can guarantee the adversarial robust coverage. This can be explained by the fact that RSCP requires the design of a specialized scoring function to guarantee coverage while aPRCP uses a quantile-of-quantile design and can employ any existing score function.

\begin{figure*}[!h]
    \centering
    \begin{minipage}{.98\linewidth}
        \begin{minipage}{\linewidth}
            \centering
            \includegraphics[width=.5\linewidth]{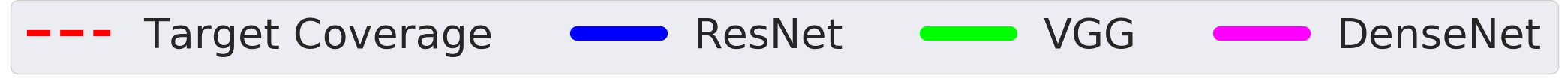}
        \end{minipage}     
        \begin{minipage}{\linewidth}
            \includegraphics[width=\linewidth]{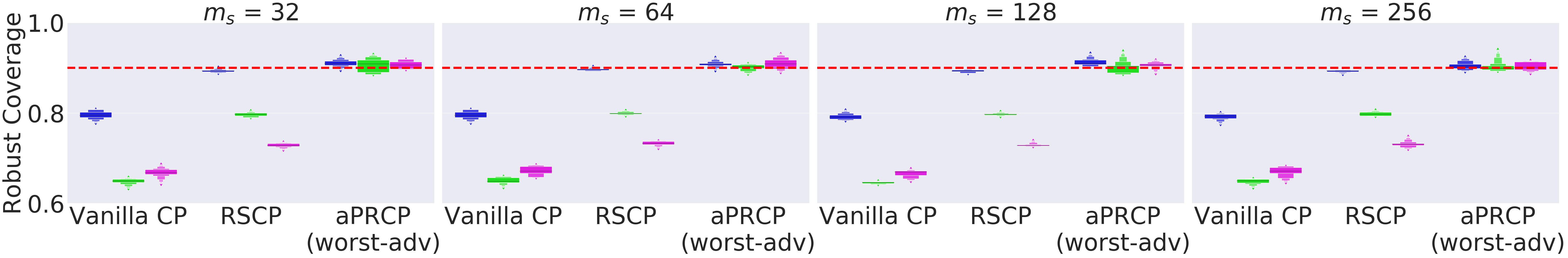}
        \end{minipage}
        \hfill
        \begin{minipage}{\linewidth}
            \includegraphics[width=\linewidth]{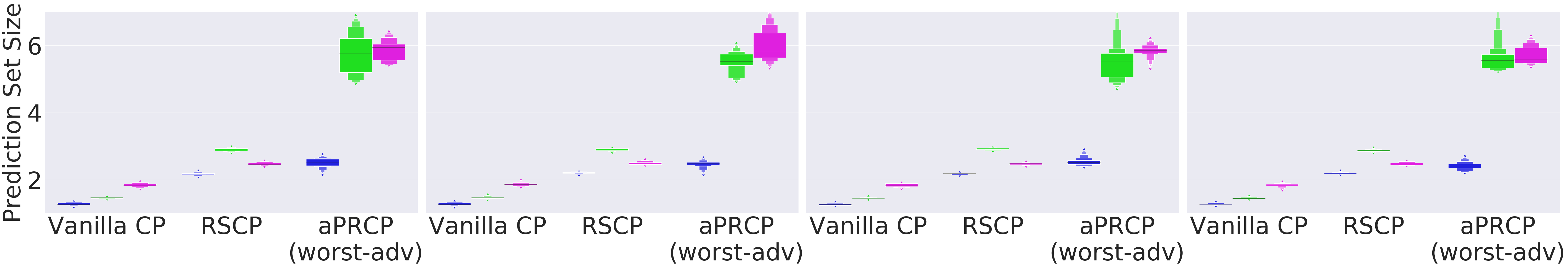}
        \end{minipage}
    \end{minipage}
    \caption{Robust coverage (top) and prediction set size (bottom) constructed by three different CP methods. The target coverage is $90\%$. The results are reported over 50 different runs for the CIFAR10 data set.}
    \label{All_blackBox_C10}
\end{figure*}

\begin{figure*}[!h]
    \centering
    \begin{minipage}{.98\linewidth}
        \begin{minipage}{\linewidth}
            \centering
            \includegraphics[width=.5\linewidth]{MainPaper/legend3.png}
        \end{minipage}     
        \begin{minipage}{\linewidth}
            \includegraphics[width=\linewidth]{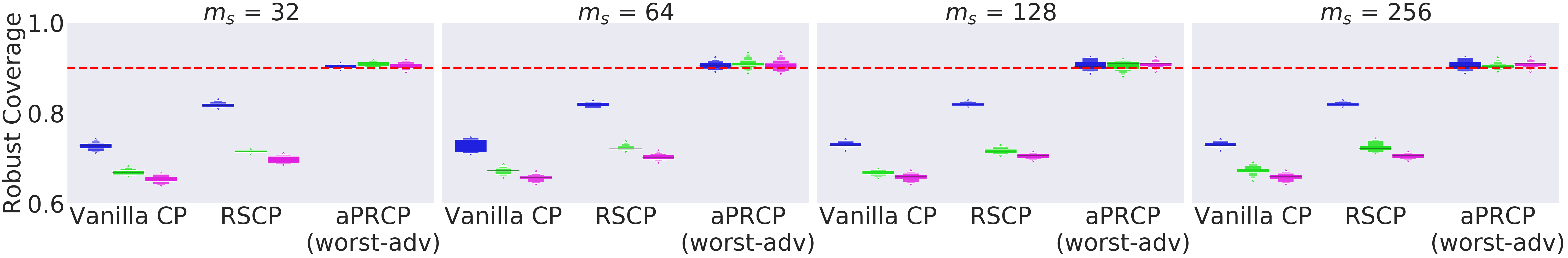}
        \end{minipage}
        \hfill
        \begin{minipage}{\linewidth}
            \includegraphics[width=\linewidth]{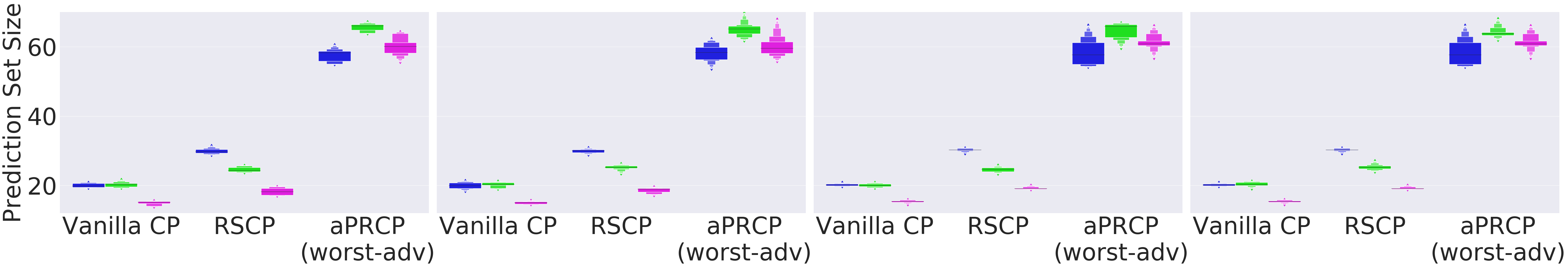}
        \end{minipage}
    \end{minipage}
    \caption{Robust coverage (top) and prediction set size (bottom) constructed by three different CP methods. The target coverage is $90\%$. The results are reported over 50 different runs for the CIFAR100 data set.}
    \label{All_blackBox_C100}
\end{figure*}

\subsection{Importance of Gaussian Augmented Training}
While aPRCP can work without any assumption on the base classifier, Figure \ref{Why_gaussian_training_C10} shows the importance of the model robustness to produce smaller prediction sets. Both \texttt{RSCP} and \texttt{aPRCP}(worst-adv) construct prediction sets that are larger when the base model is not adversarially robust.
\begin{figure*}[!h]
    \centering
    \begin{minipage}{.98\linewidth}     
        \begin{minipage}{\linewidth}
            \includegraphics[width=\linewidth]{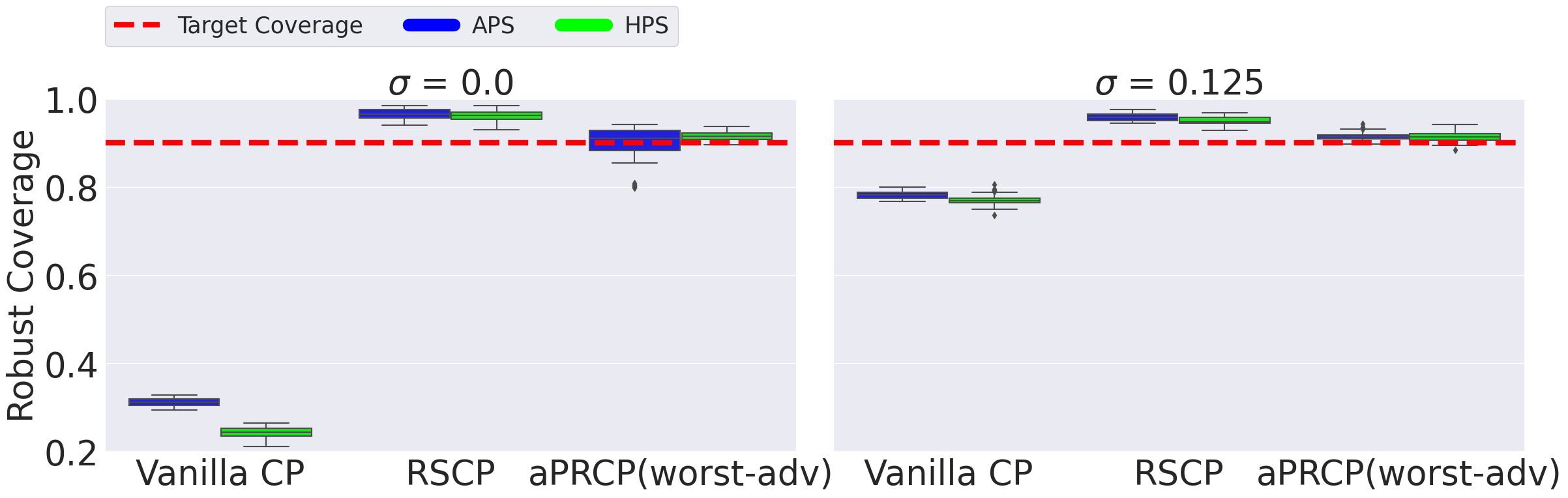}
        \end{minipage}
        \hfill
        \begin{minipage}{\linewidth}
            \includegraphics[width=\linewidth]{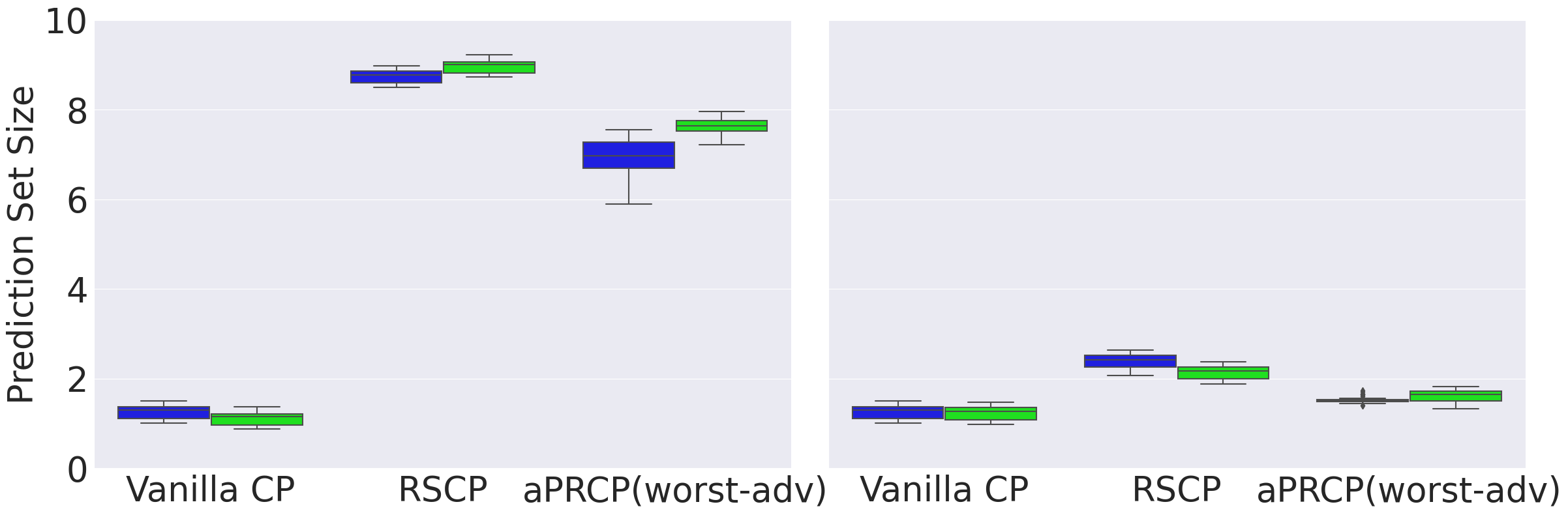}
        \end{minipage}
    \end{minipage}
    \caption{Robust coverage (top) and prediction set size (bottom) constructed by three different CP methods. The target coverage is $90\%$. The results are reported over 50 different runs for the CIFAR10 data set.}
    \label{Why_gaussian_training_C10}
\end{figure*}

\subsection{aPRCP nominal performance}
Figure \ref{Clean_data_results} shows a comparison of the nominal performance (evaluation on only clean inputs) on CIFAR10 and CIFAR100 datasets. We employ $m_s = 128$ for calibration and standard training to train the base model. We can observe that aPRCP achieves better trade-off between the nominal performance (evaluation on clean inputs) and the robust performance (evaluation on perturbed inputs). For both datasets, aPRCP achieves a tighter empirical coverage (closer to 90\%) with smaller prediction sets than RSCP.

\begin{figure*}[!h]
    \centering
    \begin{minipage}{.98\linewidth}
        \begin{minipage}{\linewidth}
            \centering
            \includegraphics[width=.6\linewidth]{Figures1/legend_clean.png}
        \end{minipage}     
        \begin{minipage}{.49\linewidth}
            \centering
            (a) CIFAR10
        \end{minipage}
        \hfill
        \begin{minipage}{.49\linewidth}
            \centering
            (b) CIFAR100
        \end{minipage} 
        \hfill
        \begin{minipage}{.49\linewidth}
            \includegraphics[width=\linewidth]{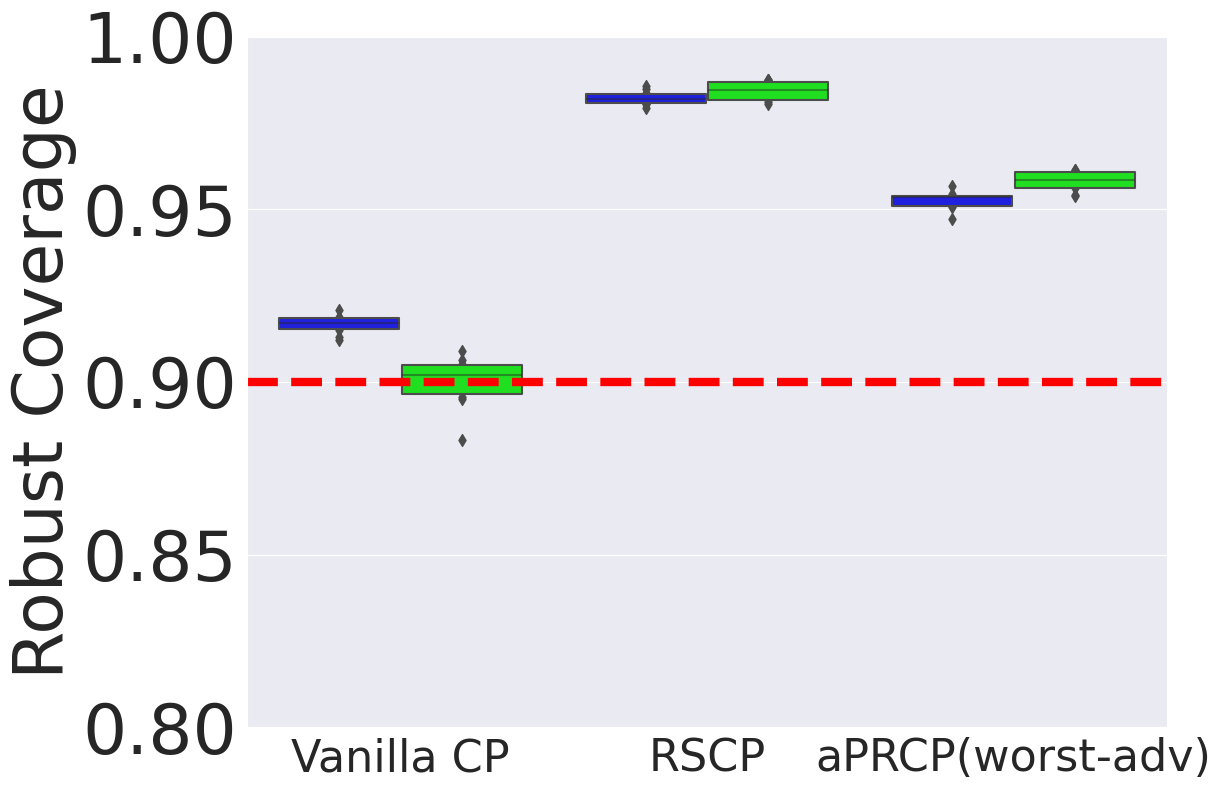}
        \end{minipage}
        \hfill
        \begin{minipage}{.49\linewidth}
            \centering
            \includegraphics[width=\linewidth]{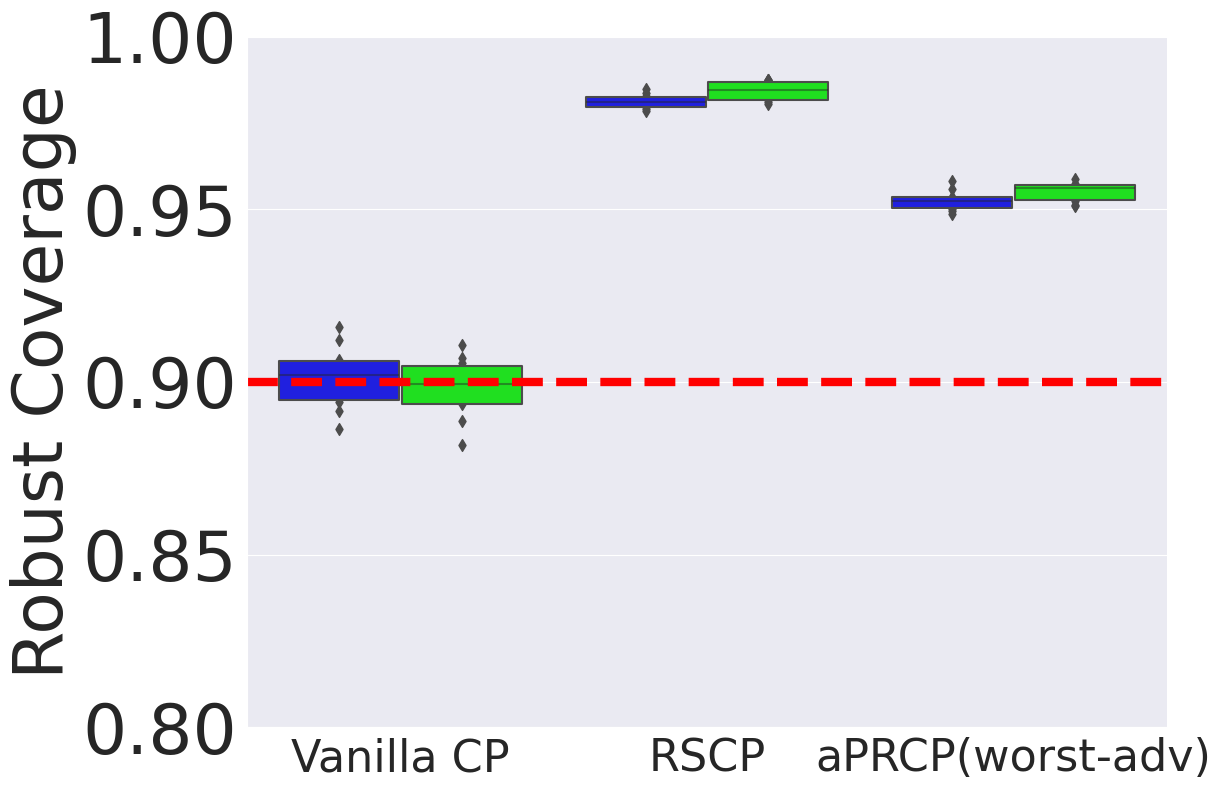}
        \end{minipage}    
        \hfill
        \begin{minipage}{.49\linewidth}
            \includegraphics[width=\linewidth]{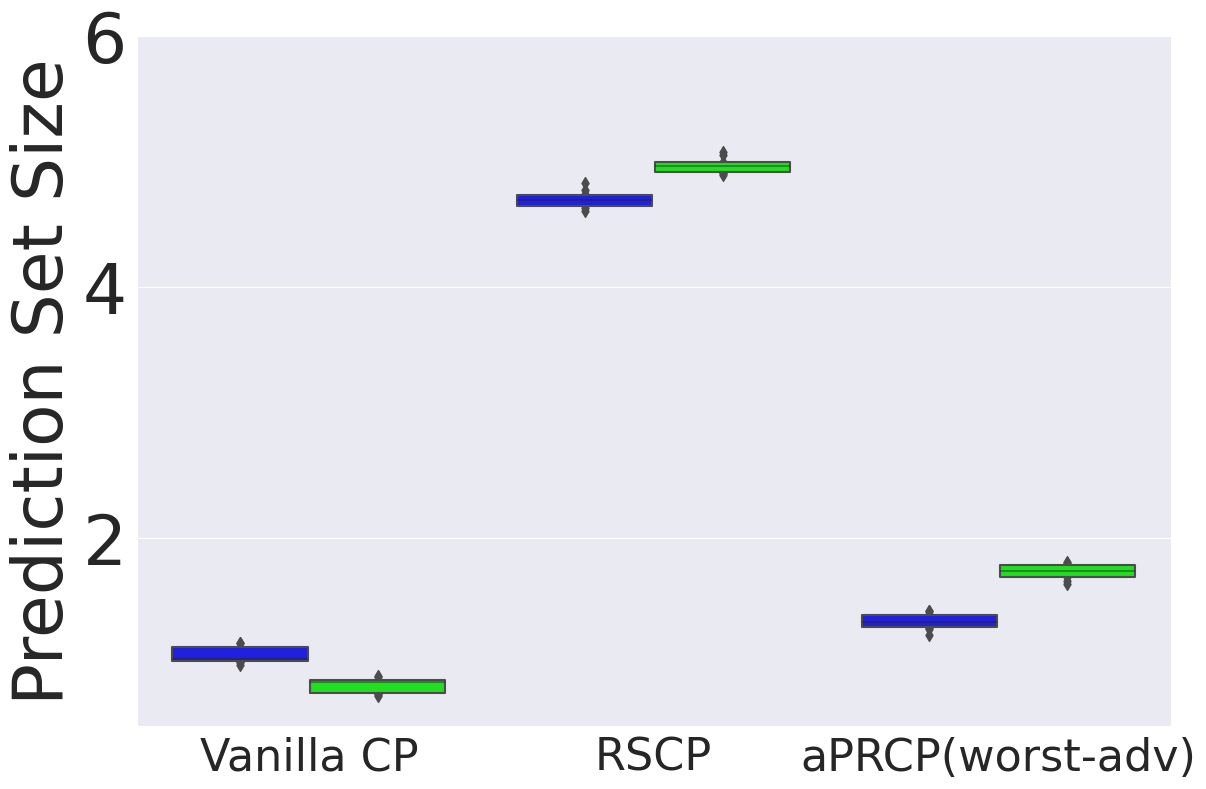}
        \end{minipage}
        \hfill
        \begin{minipage}{.49\linewidth}
            \centering
            \includegraphics[width=\linewidth]{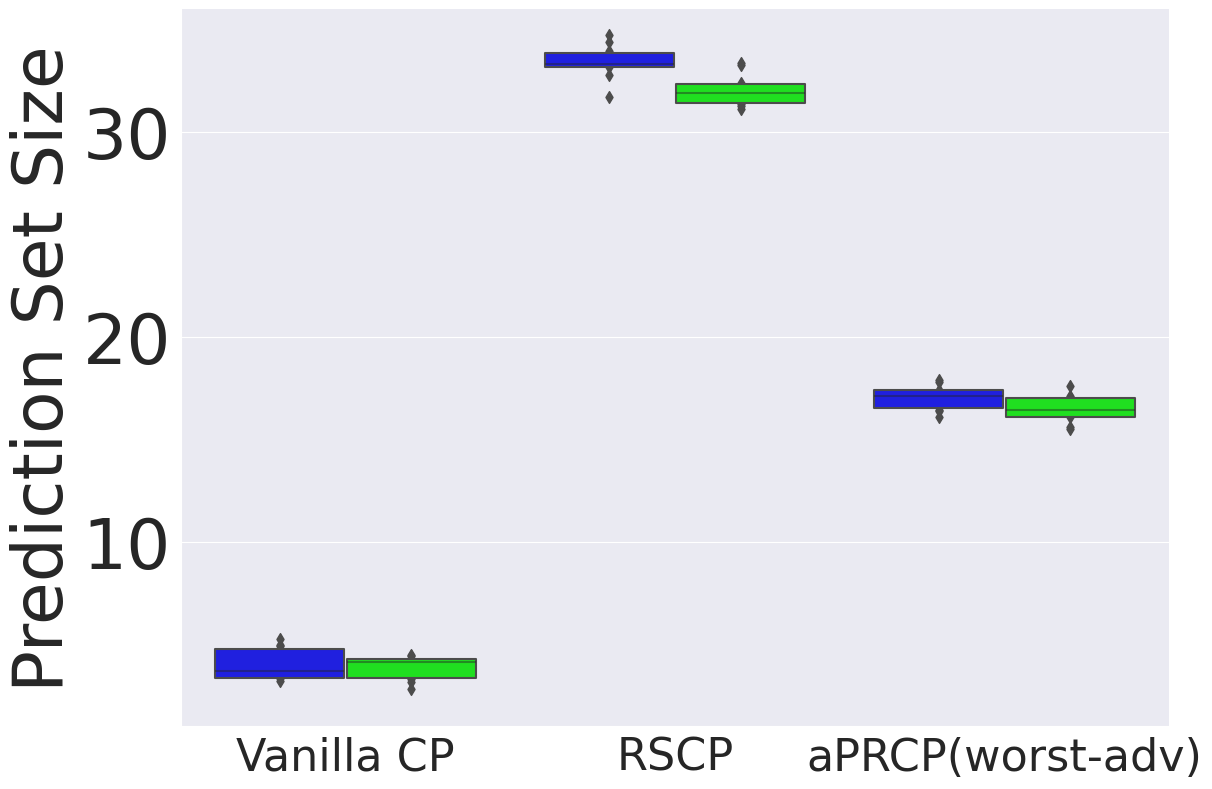}
        \end{minipage}    
    \end{minipage}
    \caption{Robust coverage (top) and prediction set size (bottom) constructed by Vanilla CP, RSCP, and aPRCP(worst-adv) using HPS and APS conformity scoring functions (target coverage is $90\%$) for the CIFAR10 and CIFAR100 data sets. Results are averaged over 50 different runs.}
    \label{Clean_data_results}
\end{figure*}